\documentclass{article} % For LaTeX2e
\usepackage{iclr2027_conference,times}

\usepackage{amsmath,amsfonts,bm}

\def\eqref#1{equation~\ref{#1}}
\def\1{\bm{1}}

\DeclareMathAlphabet{\mathsfit}{\encodingdefault}{\sfdefault}{m}{sl}
\SetMathAlphabet{\mathsfit}{bold}{\encodingdefault}{\sfdefault}{bx}{n}

\DeclareMathOperator*{\argmax}{arg\,max}
\DeclareMathOperator*{\argmin}{arg\,min}

\newcommand{\complex}{\mathcal{K}}
\newcommand{\vertexSet}{V}
\newcommand{\simplex}{\sigma}
\newcommand{\simplexT}{\tau}
\newcommand{\vertex}{v}
\newcommand{\range}{\mathbb{R}}
\newcommand{\isovalue}{a}
\newcommand{\isovalueT}{{a'}}
\newcommand{\datasetSize}{N}
\newcommand{\representation}{H}
\newcommand{\representationSample}{h}

\newcommand{\diagram}{\mathcal{D}}
\newcommand{\wasserstein}[1]{W_#1}
\newcommand{\projection}{\Delta}
\newcommand{\pointMetric}{d}

\newcommand{\wassersteinLoss}{\mathcal{L}}
\newcommand{\wassersteinEnergy}{\mathcal{E}}
\newcommand{\persistenceMap}{\mathcal{P}}
\newcommand{\filtrationMap}{\mathcal{F}}

\newcommand{\lt}{<}

\usepackage[dvipsnames]{xcolor}

\newcommand{\todo}[1]{\begingroup\color{red}#1\endgroup}

\newcommand{\myParagraph}[1]{\paragraph{#1}}
\renewcommand{\myParagraph}[1]{\textbf{#1}~}

\newcommand{\acronym}{TCR}
\newcommand{\acronymLongSing}{Topological Concept Representation}
\newcommand{\acronymLong}{\acronymLongSing{}s}

\newcommand{\subappref}[1]{\hyperref[#1]{Appendix~\ref*{#1}}}

\usepackage{hyperref}
\usepackage{multirow}
\usepackage{url}
\usepackage{booktabs}
\usepackage{graphicx}
\usepackage{afterpage}
\usepackage{siunitx}
\usepackage{float}
\usepackage{subcaption}
\usepackage{adjustbox}
\usepackage{stackengine}
\usepackage{wrapstuff}

\usepackage{amsthm}
\newtheorem{theorem}{Theorem}
\newtheorem{corollary}{Corollary}

\newtheorem{proposition}{Proposition}
\theoremstyle{remark}
\newtheorem{remark}{Remark}

\title{Operational Abstractions of Neural Network Concepts via Topological Representations}

\author{Mathieu Pont \& Christoph Garth 
\\
RPTU Kaiserslautern-Landau \\
Kaiserslautern, Germany \\
\texttt{\{mathieu.pont,garth\}@rptu.de} \\
}

\iclrfinalcopy % Uncomment for camera-ready version, but NOT for submission.
\begin{document}

\maketitle

\lhead{Under review as a conference paper at ICLR 2027}
\vspace{-0.45cm}

\begin{figure}[h]
\begin{center}
% \rule{\textwidth}{5.8cm}
\includegraphics[width=\linewidth]{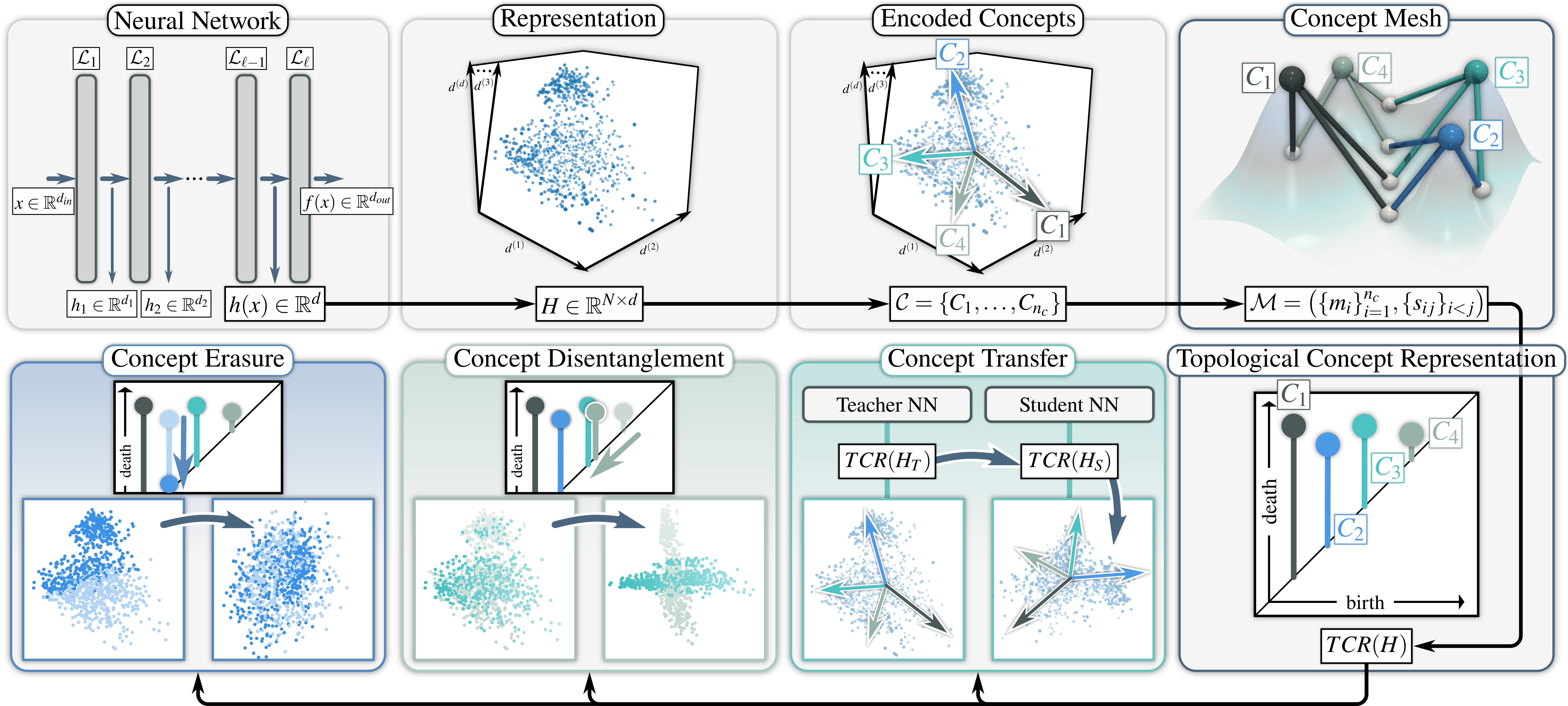}
\end{center}
\caption{\acronym{} provides a common post-hoc abstraction for concept organization in learned representations. It maps encoded concepts and their interactions to a concept mesh and summarizes its topological structure. Desired changes are specified on this abstraction and propagated back to the representation, enabling diverse interventions such as concept erasure, disentanglement and transfer.
}
\end{figure}

\begin{abstract}
Concept-based methods provide a semantic level for interpreting and manipulating learned representations, but existing editing approaches are typically specialized to particular interventions and do not provide a common and editable representation of concept organization. To achieve this, we introduce \emph{\acronymLong{}} (\acronym{}), a post-hoc operational abstraction that jointly characterizes the concepts encoded in a learned representation and their relationships. \acronym{} constructs an intermediate concept space from concept recoverability and interaction scores, and compactly encodes its organization through topology. Interventions are expressed through modifications of this abstraction that are propagated back to the underlying learned representation. This allows different concept-level operations to share the same optimization framework and separates the desired concept organization from the mechanism to achieve it. We establish stability and reparameterization-invariance properties of \acronym{} and its connections to existing concept-editing formulations. We use \acronym{} to disentangle concepts as a preprocessing step for existing erasure methods, improving worst-group accuracy by $21.89$ on average at comparable concept leakage. We further use \acronym{} to transfer concepts from teacher to student models, improving concept recoverability by up to $5.54$ while also improving or maintaining competitive test top-1 accuracy.
\end{abstract}

% TODO emphasize somewhere on "unwanted concepts, such as demographic or social attributes, leading to biases and potentially unfair predictions" and in LLMs as in Obliviator (and Vision models)

\section{Introduction}

%%%%%%%%%%%%%%%%%%%%%%%%%%%%%%%%%%%%%%%
%%%%%%%%%%%%%%%%%%%%%%%%%%%%%%%%%%%%%%%

As computational resources and learning methods continue to advance, increasingly large and complex models have become trainable \citep{llmScalingLaws}, giving rise to rich and complex learned representations \citep{ansuiniRepresentation}.
% Their complexity makes it difficult to understand and manipulate them at a level that corresponds to semantically meaningful properties we wish to reason about.
Their internal organization is often not directly interpretable, making it difficult to understand and manipulate them 
%at a 
% semantically meaningful level
%level that corresponds to 
in terms of 
semantically meaningful properties we wish to reason about.
A natural way to address this is to seek higher-level abstractions of these representations that would capture such properties in a compact way.
This notably raises the question of which \emph{operational abstractions} can be defined for such representations, that, in addition of capturing these properties, would offer a principled and unified interface to manipulate them.

The concept-based literature provides a natural semantic level for reasoning about learned representations by identifying human-interpretable properties encoded within them \citep{networkDissection, tcav, unifiedConceptExtraction}. 
Beyond interpretation, such concepts can also be manipulated, notably by erasing unwanted information from a representation \citep{inlp, rlace, leace}, preserving downstream task information during erasure \citep{jse, splince, obliviator} or transferring concept-related information across models \citep{conceptDistillation}. 
However, these methods are typically developed as independent frameworks tailored to a particular intervention and optimization algorithm. 
Consequently, changing the intended operation or introducing additional requirements often requires deriving a new formulation and optimization procedure rather than simply specifying a different desired outcome. 
Specifically, they are not formulated within a common framework in which concepts and their interactions are jointly characterized and different concept-level objectives can be expressed, composed and optimized. 
% Moreover, effects on concepts not explicitly included in the objective are generally not represented or controlled during the intervention, and are often assessed only afterward. 
% Moreover, the way concept information is entangled in the representation is typically treated as fixed: methods optimize an intervention within this existing organization, rather than reshaping it to alter the collateral effects of the intervention.
% Moreover, the effects of modifying one concept on others are generally not explicitly modeled during the intervention, and are often assessed only afterward.
% Moreover, the impact of modifying one concept on others is generally not explicitly considered during the intervention, and often assessed only afterward.
% Moreover, existing methods do not generally adjust the existing relationships between concepts to influence the outcome of an intervention.
% Moreover, existing methods generally do not treat the relationships between concepts as quantities that can be adjusted to influence the outcome of an intervention.
Moreover, existing methods generally do not treat concept relationships as quantities that can be adjusted to influence the outcome of an intervention.
% Moreover, the relationships between concepts are generally not treated as quantities that can be adjusted to influence the outcome of an intervention.
This fragmentation motivates a common post-hoc representation in which concept organization itself can be described, compared and manipulated in a unified way.

Topological Data Analysis (TDA) provides natural tools for constructing such abstractions \citep{edelsbrunner09, zomorodianBook}. The resulting objects compactly characterize the global organization of their input, while their multiscale nature accounts for relationships of varying strengths without committing to a single arbitrary threshold. They can also be compared in a principled manner and, beyond their descriptive role, can serve as \emph{operational proxies}: when coupled with appropriate optimization procedures, modifications expressed on the abstraction can be translated into controlled modifications of the underlying input \citep{carriere21}.

% This raises the question of whether such an operational abstraction can be defined for learned representations themselves.

% What topology adds is a way to summarize how connectivity and grouping emerge and persist across scales, without committing to one threshold.
% The object of interest is a structured system of concepts and relationships, and we want a compact representation of how these relationships organize globally and across scales.
% Topology also has a second advantage: it gives you discrete structural features that can be targeted by edit
% Topological summaries are attractive here because they encode multiscale connectivity in a compact form and expose explicit structural features that can be manipulated.

%%%%%%%%%%%%%%%%%%%%%%%%%%%%%%%%%%%%%%%
%%%%%%%%%%%%%%%%%%%%%%%%%%%%%%%%%%%%%%%

% This paper adresses these limitations by introducting the notion of \emph{\acronymLong{}} (\acronym{}), a unified and general way to represent concepts, their interactions at multiple scales and principled methods for their modifications.
% This notably allows
% Interestingly, we show that most concept erasure methods studied in this paper can be formulated under this framework, revealing unexpected relations between them.

We build on these properties to introduce \emph{\acronymLong{}} (\acronym{}), a post-hoc operational abstraction of the concept organization encoded in learned representations. 
\acronym{} first constructs an intermediate \emph{concept space} in which individual concepts are represented together with their relationships, characterized by scalar scores derived from the learned representation.
% This space is then compactly encoded through a topological abstraction, acting as an interface for the manipulation of this concept space and, hence, of the original representation. 
% An intervention can then be expressed by modifying the resulting topological abstraction to define a target \acronym{}, and optimizing a transformation of the underlying learned representation to realize this target. 
This space is then compactly encoded through a topological abstraction, which provides an interface for acting on the concept organization and, through it, on the underlying learned representation.
Concretely, an intervention is specified by modifying this abstraction to define a target \acronym{}, and optimizing a transformation of the 
%underlying 
learned representation to realize it.
In this way, different operations are expressed through different target topologies 
% rather than different editing algorithms, 
within a shared optimization framework,
thereby separating \emph{what} concept organization is desired from \emph{how} the representation is modified to achieve it. %, allowing for different objectives to be expressed within the same framework. % while encompassing a broad range of existing concept-erasure criteria.

We evaluate \acronym{} on complementary concept editing and transfer tasks. For concept erasure, we consider both direct erasure with \acronym{} and its use as a preprocessing step that disentangles unwanted concepts from task-relevant information before applying existing erasure methods; across datasets and baselines, this improves test WGA by $21.89$ on average at comparable levels of concept leakage. For concept transfer, we use the \acronym{} of a teacher representation as a target for a student and show that the transfer generalizes beyond the concepts explicitly included in the optimization, improving concept recoverability by up to $5.54$ while improving or maintaining competitive test top-1 accuracy.

\iffalse
% \note{We introduce Concept Topological Representations (CTR)

% CTR serves as an operational interface in which diverse representation edits can be expressed as target topologies and realized through learned feature transformations.
% CTR separates the desired organization of concepts from the particular transformation used to produce it.

% CTR turns concept editing from an algorithm-design problem into a target-specification problem.

% Concepts do not exist as an unstructured collection of independent variables. 
% They form relationships, groups, and hierarchies that may appear at different strengths. Topology provides a natural language for representing this organization without committing to a single arbitrary threshold.
% CTR captures how concepts progressively organize into larger structures across interaction scales

% CTR separates what concept organization is desired from how the neural representation is transformed to realize it. 

% Different applications become different target topologies rather than entirely different algorithms. 

% An edit is then specified as a modification of this structure and realized by optimizing a transformation of the neural representation.}

% Neural representations from different architectures live in incompatible coordinate spaces, but their concepts can be described in a shared concept-level abstraction. Therefore CTR may serve as a common interface between models, even when their feature dimensions and coordinates differ.
\fi

%%%%%%%%%%%%%%%%%%%%%%%%%%%%%%%%%%%%%%%
%%%%%%%%%%%%%%%%%%%%%%%%%%%%%%%%%%%%%%%

\myParagraph{Contributions.} Our main contributions are as follows:
%\vspace{-0.5em}
\begin{itemize}
\item We introduce \acronym{}, a post-hoc topological abstraction that jointly represents concepts and their relationships and provides a common interface for concept-level interventions.
\item We establish theoretical properties of \acronym{} such as stability, invariance to representation reparameterization and its connections to existing concept-editing formulations.
\item We evaluate \acronym{} across concept erasure, disentanglement and transfer, showing improved task preservation during erasure and improved transfer of concept recoverability.
\end{itemize}

%%%%%%%%%%%%%%%%%%%%%%%%%%%%%%%%%%%%%%%%%%%%%%%%%%%%%%%%%%%%%%%%%%%%%%%%%%%%%%%
%%%%%%%%%%%%%%%%%%%%%%%%%%%%%%%%%%%%%%%%%%%%%%%%%%%%%%%%%%%%%%%%%%%%%%%%%%%%%%%
%%%%%%%%%%%%%%%%%%%%%%%%%%%%%%%%%%%%%%%%%%%%%%%%%%%%%%%%%%%%%%%%%%%%%%%%%%%%%%%
% TODO cite visual concept connectome?

\myParagraph{Related work.}
We provide an extended discussion of concept representation and intervention methods as well as topological abstractions and persistence-based optimization in \autoref{sec_app_related_work}.

% \subsection{Related Work}

\section{Preliminaries}
\label{sec_preliminaries}

% This section introduces the background and notation required for our method.

%%%%%%%%%%%%%%%%%%%%%%%%%%%%%%%%%%%%%%%
%%%%%%%%%%%%%%%%%%%%%%%%%%%%%%%%%%%%%%%
\myParagraph{Concepts and probes.} 
Given a representation function \(h:\mathcal{X}\rightarrow\mathbb{R}^d\) and a collection of \(N\) samples \(\{x_1,\ldots,x_N\}\), we denote their representations by \(H=[h(x_1),\ldots,h(x_N)]^\top\in\mathbb{R}^{N\times d}\). 
A concept \(C_i\) is described 
%over these samples 
by a vector \(\mathbf{c}_i=(c_{i,1},\ldots,c_{i,N})^\top\), indicating its presence or strength for each sample. 
Such concept annotations may be provided explicitly or obtained through 
a
%an unsupervised 
concept discovery procedure. 
A common way to characterize how a concept is encoded in a representation is through a probe \(q_i:\mathbb{R}^d\rightarrow\mathbb{R}\) predicting \(c_{i,n}\) from \(h(x_n)\). 
For instance, a linear probe \mbox{\(q_i(h)=\mathbf{w}_i^\top h+b_i\)} associates the concept with a direction \(\mathbf{w}_i\) in the representation space. 
% The predictive ability of such a probe can then be used to quantify how strongly the concept is encoded in the representation.
% The ability of such a probe to predict \(\mathbf c_i\) can then be used to quantify how strongly the concept is represented.
The ability of such a probe to predict \(\mathbf c_i\) can then be used to quantify how strongly the concept is encoded in the representation.

%%%%%%%%%%%%%%%%%%%%%%%%%%%%%%%%%%%%%%%
%%%%%%%%%%%%%%%%%%%%%%%%%%%%%%%%%%%%%%%
\myParagraph{Persistence diagrams.} 
%We now briefly introduce the persistence diagram %, one of the most common topological descriptor
% A detailed introduction to computational topology can be found in the books of \citet{edelsbrunner09} and \citet{zomorodianBook}.
% and refer to the books of \citet{edelsbrunner09} and \citet{zomorodianBook} for a more detailed introduction.
%\citep{edelsbrunner09}.
Let a scalar field $f : \vertexSet \rightarrow \range$ defined on the vertex set $\vertexSet$ of a simplicial complex $\complex$.
% A topological representation aims to encode its features of interest in a compact manner and we focus in this work on the 0th topological features, intuitively corresponding to the \emph{pits} of \emph{peaks} of $f$.
% Each feature is characterized by an extremum-saddle pair of critical vertices of $f$ (a minimum and a saddle or a saddle and a maximum).
%The topological representations considered in this work encode the topological features of $f$
% More formally, the features involving minima of $f$ can be uniquely identified using the lower-star filtration.
The topological features of $f$ can be identified using the \emph{lower-star filtration} and encoded in the persistence diagram \citep{edelsbrunner09}.
% The persistence diagram \citep{edelsbrunner09} can encode the topological features of $f$ using the \emph{lower-star filtration}.
Let \mbox{$\mathbf{f} = \big(f(\vertex_1),\ldots,f(\vertex_{n_\vertex})\big) \in \range^{|\vertexSet|}$} be the vector of vertex values of $f$ and let the \emph{filtration map} \mbox{$\filtrationMap : \range^{|\vertexSet|} \rightarrow \range^{|\complex|}$} mapping a vector of vertex values to a 
\mbox{$\range^{|\complex|}$-dimensional} 
vector of simplex values where \mbox{$\filtrationMap_\simplex(\mathbf{f}) = \max_{\vertex \in \simplex} f(v)$} for each simplex \mbox{$\simplex \in \complex$}, ensuring that $\filtrationMap_\simplexT(\mathbf{f}) \leq \filtrationMap_\simplex(\mathbf{f})$ for 
%$\simplexT, \simplex \in \complex$ and 
any $\simplexT \subseteq \simplex$.
The lower-star filtration is the nested family of subcomplexes \mbox{$\complex_\isovalue = \{\simplex \in \complex ~ | ~ \filtrationMap_\simplex(\mathbf{f}) \leq \isovalue\}$} with $\complex_\isovalue \subseteq \complex_\isovalueT$ for any $\isovalue \leq \isovalueT$.
Assuming 
%that 
$f$ 
%is 
injective on the vertices of $\complex$, 
% being easily achieved in practice 
obtained easily
with symbolic perturbations \citep{edelsbrunner90}, 
the filtration changes only at the $n_\vertex = |\vertexSet|$ distinct vertex values,
% and can be 
represented by the finite sequence \mbox{$\emptyset \subset \complex_{f(\vertex_1)} \subset \ldots \subset \complex_{f(\vertex_{n_\vertex})} = \complex$} 
(%with $\vertex_1,\dots,\vertex_{n_\vertex}$ the ordered vertices 
such that \mbox{$f(\vertex_1) \lt \dots \lt f(\vertex_{n_\vertex})$)} where each step adds a vertex $\vertex$ as well as its lower star (i.e. each simplex $\simplex$ containing $\vertex$ such that $\filtrationMap_\simplex(\mathbf{f}) = f(\vertex)$).

Topological changes, such as the appearance and merging of connected components, occur at specific filtration values. When two components merge, the one created later \emph{dies} according to the \emph{Elder} \emph{rule} \citep{edelsbrunner09} and defines a \emph{persistence pair} \((\vertex,\vertex')\), where \(\vertex\) creates the component and \(\vertex'\) merges it;
%with an older one; 
\(f(\vertex)\) and \(f(\vertex')\) are respectively its \emph{birth} and \emph{death}, and \(f(\vertex')-f(\vertex)\) its \emph{persistence}, commonly interpreted as a measure of 
% feature 
importance. 
The component born at the global minimum never dies, 
%through a merge, 
and we conventionally assign it the global maximum as death value. The \emph{persistence diagram} \(\diagram(f)\) is the multiset of points \(\big(f(\vertex),f(\vertex')\big)\) associated with all persistence pairs.

%%%%%%%%%%%%%%%%%%%%%%%%%%%%%%%%%%%%%%%
%%%%%%%%%%%%%%%%%%%%%%%%%%%%%%%%%%%%%%%
\iffalse
\myParagraph{Wasserstein distances.} 
Two persistence diagrams can be compared using the Wasserstein distance.
It seeks a mapping between their points such that a point of one diagram can either be mapped to exactly one point in the other diagram, representing two features being matched together, or to its projection on the diagonal of the other diagram, intuitively representing the deletion or insertion of a feature.
In practice, when computing the distance between two diagrams $\diagram(f_i)$ and $\diagram(f_j)$, they are first pre-processed to augment each diagram with the diagonal projection of the off-diagonal points of the other diagram, i.e. $\diagram'(f_i) = \diagram(f_i) \cup \{\projection(p_j) ~ | ~ p_j \in \diagram(f_j)$ and similarly for $\diagram'(f_j)$ and where $\projection(p_i) = (\frac{x_i+y_i}{2},\frac{x_i+y_i}{2})$ is the diagonal projection of an off-diagonal point $p_i = (x_i, y_i)$.
This 
%procedure 
simulates the fact that persistence diagrams have an infinite number of diagonal points (dummy features with zero persistence), allows off-diagonal points to possibly be matched to their diagonal projection 
% (i.e. their closest point among all points on the diagonal).
% It also
and 
guarantees that the two diagrams have the same number of points (\mbox{$| \diagram'(f_i)| = |\diagram'(f_j)|$}).
\fi
\myParagraph{Wasserstein distances.}
Two persistence diagrams can be compared using the Wasserstein distance, which seeks a minimum-cost matching between their points. Each off-diagonal point can either be matched to exactly one point of the other diagram, associating two topological features, or to the diagonal, corresponding to the insertion or deletion of a feature. In practice, for diagrams \(\diagram(f_i)\) and \(\diagram(f_j)\), each is augmented with the diagonal projections of the off-diagonal points of the other, e.g. \mbox{\(\diagram'(f_i)=\diagram(f_i)\cup\{\projection(p_j)\mid p_j\in\diagram(f_j)\}\)}, where \mbox{\(\projection(x,y)=((x+y)/2,(x+y)/2)\)}. This accounts for the infinite multiplicity of the diagonal and ensures \(|\diagram'(f_i)|=|\diagram'(f_j)|\), so that the distance can be computed as a bijective matching between the augmented diagrams.

Two points $p_i = (x_i, y_i) \in \diagram'(f_i)$ and $p_j = (x_j, y_j) \in \diagram'(f_j)$ are mapped with a cost of % being their distance in the birth/death space:
%, called the \emph{ground metric}: 
\mbox{$\pointMetric_{q}(p_i, p_j)^q = |x_i - x_j|^q + |y_i - y_j|^q = \| p_i - p_j \|_q^q$},
% \begin{equation}
% \label{eq_groundMetric}
% \pointMetric_{q}(p_i, p_j) = \big(|x_i - x_j|^q + |y_i - y_j|^q\big)^{1/q} = \| p_i - p_j \|_q,
% \end{equation}
%\noindent
with \mbox{$q \geq 1$}
and where $\pointMetric_q(p_i, p_j)$ is set to zero if both $p_i$ and $p_j$ are diagonal points.
\iffalse
In the 
%special 
case where both $p_i$ and $p_j$ are diagonal 
points
%dummy features located on the diagonal (i.e. $x_i = y_i$ and $x_j = y_j$) 
then $\pointMetric_q(p_i, p_j)$ is set to zero (for them to not contribute to the distance).
\fi
The \mbox{$L^q$-Wasserstein} distance $\wasserstein{q}$ is then defined as:

\begin{eqnarray}
\label{eq_wasserstein}
\wasserstein{q}\big(\diagram(f_i), \diagram(f_j)\big)
=
\label{eq_wasserstein_mapping}
&
\hspace{-.15cm}
\underset{\phi \in \Phi}\min
 \hspace{-.15cm}
&
\Big(
\sum_{p_i \in \diagram'(f_i)}
\pointMetric_q\big(p_i,
\phi(p_i)\big)^q\Big)^{1/q},
\end{eqnarray}

\noindent
where $\Phi$ is the set of all possible assignments $\phi$ mapping each point $p_i \in \diagram'(f_i)$ to exactly one point $p_j \in \diagram'(f_j)$
% \revisionTVCG{Lines \ref{eq_wasserstein_destruction} and \ref{eq_wasserstein_creation}, \revisionTVCGminor{in which} $\projection(p_i) = (\frac{x_i+y_i}{2},\frac{x_i+y_i}{2})$ is the diagonal projection of an off-diagonal point $p_i = (x_i, y_i)$\revisionTVCGminor{,} indicate the deletion or the appearance of a feature, respectively.}
being either an off-diagonal point or the diagonal projection of $p_i$ (i.e. \mbox{$\phi(p_i) = \projection(p_i) = p_j \in \diagram'(f_j)$}) indicating its deletion. 
% Similarly, when $p_i$ is a diagonal point being matched to an off-diagonal point of $\diagram'(f_j)$ it indicates its appearance.
$\wasserstein{q}$ can be computed using traditional assignment problem solvers \citep{Munkres1957, Bertsekas81, cuturi2013Sinkhorn, Kerber2016}.
% \revisionTVCG{The Bottleneck distance \cite{edelsbrunner02} is recovered when $q \rightarrow \infty$, in which case only the maximum matching cost is retained.}
% TODO cite sinkhorn

%%%%%%%%%%%%%%%%%%%%%%%%%%%%%%%%%%%%%%%
%%%%%%%%%%%%%%%%%%%%%%%%%%%%%%%%%%%%%%%
\myParagraph{Persistence-based topological optimization.}
Persistence diagrams can be incorporated into differentiable optimization objectives \citep{carriere21}. Let \(\persistenceMap:\range^{|\complex|}\rightarrow\range^{|\complex|}\) denote the \emph{persistence map}, being a permutation of the simplex filtration values returned by \(\filtrationMap\) according to the persistence pairing, such that the birth and death values of each pair are placed consecutively (including zero-persistence pairs). % that are typically omitted from the persistence diagram). 
% For a fixed ordering of the simplices, the persistence pairs remain unchanged and \(\persistenceMap\) simply reorders the coordinates of the filtration vector. 
Gradients with respect to birth and death values can therefore be propagated back to the simplices inducing them and, through \(\filtrationMap\), to the vertex values defining the filtration. 
A persistence-based objective can then be written as \(\wassersteinLoss=\wassersteinEnergy\circ\persistenceMap\circ\filtrationMap\), where \(\wassersteinEnergy\) is a differentiable energy defined on the persistence representation, such as one derived from the Wasserstein distance of \autoref{eq_wasserstein_mapping}. Under mild regularity conditions, such objectives are differentiable almost everywhere, admit well-defined generalized subgradients at non-smooth points and can be optimized with stochastic subgradient methods with convergence guarantees to critical points \citep{carriere21}.

% TODO homegenize and make clearer the use of H and Z throughout the section (notably 4.1 and 4.2 first paragraph use H then Z is used)
\section{Topological Concept Representations}
\label{sec_method}

%%%%%%%%%%%%%%%%%%%%%%%%%%%%%%%%%%%%%%%%%%%%%%%%%%%%%%%%%%%%%%%%%%%%%%%%%%%%%%%
%%%%%%%%%%%%%%%%%%%%%%%%%%%%%%%%%%%%%%%%%%%%%%%%%%%%%%%%%%%%%%%%%%%%%%%%%%%%%%%
%%%%%%%%%%%%%%%%%%%%%%%%%%%%%%%%%%%%%%%%%%%%%%%%%%%%%%%%%%%%%%%%%%%%%%%%%%%%%%%

\subsection{Overview and Problem Formulation}
Given \(\datasetSize\) samples \(\{x_1,\ldots,x_\datasetSize\}\), let 
%\(\representation = \{\representationSample_i\}_{i=1}^\datasetSize\), with \(\representationSample_i \in \mathbb{R}^d\), 
\mbox{$\representation = [\representationSample(x_1),\ldots,\representationSample(x_N)]^\top \in \mathbb{R}^{\datasetSize\times d}$}
denote their representations, such as those extracted at a selected model layer, and let \(\mathcal{C} = \{C_1,\ldots,C_{n_C}\}\) be a collection of concepts defined on these samples.
Our objective is to construct an abstraction of $\representation$ that jointly characterizes which concepts are encoded in the representation and how these concepts interact with each other, and that can subsequently be used as an operational proxy to modify $\representation$.
To this end, we introduce \emph{\acronymLong{}} (\acronym{}).
We associate to each concept $C_i$ a scalar value $m_i$ characterizing how strongly $C_i$ is encoded in $\representation$, and to each pair of concepts $(C_i,C_j)$ an interaction value $s_{ij}$.
These values define a scalar field on a \emph{concept mesh} in which concepts form maxima and their interactions form saddle vertices.
The persistence diagram of this scalar field, the \acronym{} of $\representation$, summarizes both the presence of individual concepts and their organization with respect to the other concepts.
% TODO say that persistence diagrams permit to manipulate the concepts (more intuitive than below)
% Importantly, 
Beyond being descriptive, the resulting abstraction permits to express a desired intervention by defining a \emph{target \acronym{}}, e.g., by reducing or increasing a concept maximum or a saddle connecting two concepts.
Considering a parameterized transformation $T_\theta : \mathbb{R}^d \rightarrow \mathbb{R}^{d'}$, applied row-wise to $\representation$ as 
% \(T_\theta(\representation)=\{T_\theta(\representationSample_i)\}_{i=1}^{\datasetSize}\), 
\mbox{$T_\theta(\representation) := [T_\theta(\representationSample_1),\ldots,T_\theta(\representationSample_{\datasetSize})]^\top \in \mathbb{R}^{\datasetSize \times d'}$},
we optimize \(\theta\) such that the \acronym{} of \(T_\theta(\representation)\) matches this target.
Hence, different concept manipulations can be formulated as modifications of a common abstraction rather than through operation-specific objectives.
% TODO reference to overview figure at the pertinent locations in the text above

%%%%%%%%%%%%%%%%%%%%%%%%%%%%%%%%%%%%%%%%%%%%%%%%%%%%%%%%%%%%%%%%%%%%%%%%%%%%%%%
%%%%%%%%%%%%%%%%%%%%%%%%%%%%%%%%%%%%%%%%%%%%%%%%%%%%%%%%%%%%%%%%%%%%%%%%%%%%%%%
%%%%%%%%%%%%%%%%%%%%%%%%%%%%%%%%%%%%%%%%%%%%%%%%%%%%%%%%%%%%%%%%%%%%%%%%%%%%%%%

\subsection{Constructing \acronymLong{}}
\label{sec_ctr_construction}

Given a representation 
%matrix 
% $\representation$, 
$Z \in \mathbb{R}^{\datasetSize \times d'}$,
we first associate scalar values with individual concepts, then measure their pairwise interactions and combine 
% both quantities 
them
into a topological abstraction.
% Concepts can be provided through annotated labels or obtained through unsupervised concept discovery procedures (e.g., \citet{ace, craft}).
% \acronym{} only requires, for each $C_i$, a vector $\mathbf{c}_i$ describing the extent to which $C_i$ is encoded over each of the $\datasetSize$ samples.

%%%%%%%%%%%%%%%%%%%%%%%%%%%%%%%%%%%%%%%
%%%%%%%%%%%%%%%%%%%%%%%%%%%%%%%%%%%%%%%

\myParagraph{Concept scores.}
At the framework level, \acronym{} represents simply a concept score by a scalar %:
\iffalse
\begin{equation}
    m_i = S(\representation, C_i),
\end{equation}
\fi
% \mbox{$m_i = S(\representation, C_i)$}
\mbox{$m_i = S(Z, C_i)$}
whose value increases with how much the concept $C_i$ is encoded in the representation.
While different choices of $S$ can instantiate different notions of concept encoding, in this work we primarily quantify it through \emph{recoverability}.
Let $\mathbf{c}_i = (c_{i,1},\ldots,c_{i,\datasetSize})^\top$ denote the sample-level scores of $C_i$, and let
$Z = [z_1,\ldots,z_{\datasetSize}]^\top \in \mathbb{R}^{\datasetSize \times d'}$
denote the representation on which the concept score is evaluated.
% Given a probe $q_{\phi_i}$ and a prediction loss $\ell$, we define the concept score as the improvement of the fitted probe over the optimal constant predictor:
Given a probe $q_{\phi_i}$ and a prediction loss $\ell$, we define the concept score as:
\begin{equation}
    m_i = \mathcal{L}_i^0 - \mathcal{L}_i^\star,
    \qquad
    \mathcal{L}_i^0
    =
    \min_{a}
    \frac{1}{\datasetSize}
    \sum_{n=1}^{\datasetSize}
    \ell(c_{i,n}, a),
    \qquad
    \mathcal{L}_i^\star
    =
    \min_{\psi_i}
    \frac{1}{\datasetSize}
    \sum_{n=1}^{\datasetSize}
    \ell\!\left(c_{i,n}, q_{\psi_i}(z_n)\right),
    \label{eq_ctr_concept_score}
\end{equation}
where $\psi_i$ denotes the fitted probe parameters and $\mathcal{L}_i^0$ is the loss of the optimal constant predictor, which does not use $Z$ and thus provides a baseline for prediction without information from the representation.
%Here, $\phi_i^\star$ denotes the fitted probe parameters and $\mathcal{L}_i^0$ is the loss of the optimal constant predictor, providing a baseline for prediction without information from the representation.
Thus, $m_i$ is the improvement of the fitted probe over the optimal constant predictor where \(m_i=0\) indicates no improvement over the null predictor, while larger values indicate that \(C_i\) is more readily predicted from \(Z\).
We consider binary cross-entropy (BCE) with a logistic probe for binary concepts, and mean squared error (MSE) with a least-squares probe for binary or continuous concepts.
The score $m_i$ can optionally be normalized by $\mathcal{L}_i^0$, yielding a dimensionless score facilitating the joint treatment of concepts evaluated with different prediction losses, e.g., BCE for binary and MSE for continuous concepts.
For MSE, it corresponds to an \(R^2\)-type explained-variance score.
% TODO mention that you can do non-linear probes? Or remove the mention to the specific linear logistic and least-squares probe?
The MSE formulation additionally admits a convenient analytical form described in \autoref{sec_app_mse}.

%%%%%%%%%%%%%%%%%%%%%%%%%%%%%%%%%%%%%%%
%%%%%%%%%%%%%%%%%%%%%%%%%%%%%%%%%%%%%%%

\myParagraph{Concept interactions.}
Individual concept scores characterize 
%the extent to which each
how much a 
concept is encoded in a representation, but not how concepts relate to one another.
We therefore associate every pair $(C_i,C_j)$ with a normalized interaction score
\mbox{$\rho_{ij}\in[0,1]$}, where $\rho_{ij}=0$ indicates no interaction according to the chosen criterion and $\rho_{ij}=1$ indicates maximal interaction.
While some interaction criteria are naturally symmetric, others characterize a directional relationship
\mbox{$r_{i\rightarrow j}\in[0,1]$}.
Directional scores can be combined to associate each concept pair with a single interaction value:
\iffalse
\begin{equation}
\label{eq_ctr_interaction_sym}
\rho_{ij}
=
\Psi\!\left(
r_{i\rightarrow j},
r_{j\rightarrow i}
\right),
\end{equation}
\fi
$\rho_{ij} = \Psi\!\left(r_{i\rightarrow j}, r_{j\rightarrow i}\right)$,
where $\Psi:[0,1]^2\rightarrow[0,1]$ is a symmetric aggregation function, such as \mbox{$\Psi(a,b)=(a+b)/2$}.

We principally instantiate this construction through \emph{interference}.
For instance, a representation may strongly encode both an undesired concept and the task while relying on shared information to encode them, such that removing the former also reduces the score of the latter.
Inspired by \citet{naowarat26}, who characterize interference through the prediction error of one concept after erasing another, we measure directional interference through the relative decrease in the concept score of $C_j$ after removing $C_i$.
Let $\mathcal{E}_i(Z)$ denote the representation obtained by erasing concept $C_i$ from $Z$, and let $m_j(Z)$ denote the score of $C_j$ evaluated on $Z$.
We define the directional interference as:
\begin{equation}
\label{eq_ctr_interference}
r_{i\rightarrow j}
=
1-
\frac{
m_j\bigl(\mathcal{E}_i(Z)\bigr)
}{
m_j(Z)
}.
\end{equation}
Intuitively, $r_{ij}$ is high when removing $C_i$ substantially reduces the recoverability score of $C_j$, and low when $C_j$ remains largely unaffected by this intervention. 
% The symmetric score \(\rho_{ij}\) then summarizes this relationship across both directions.
We instantiate $\mathcal{E}_i$ with a differentiable implementation of LEACE~\citep{leace}. % and refit the probe for the remaining concept on the erased representation.
Consequently, \autoref{eq_ctr_interference} characterizes interactions through the effect of an explicit intervention, rather than through their correlation alone.
\iffalse
A different interaction score using MSE considers \emph{predictive overlap}, 
%(using concept predictive vectors), 
\fi
For MSE concept scores, interference admits an equivalent closed-form expression,
described in \autoref{sec_app_mse}.

%%%%%%%%%%%%%%%%%%%%%%%%%%%%%%%%%%%%%%%
%%%%%%%%%%%%%%%%%%%%%%%%%%%%%%%%%%%%%%%

\myParagraph{From concept scores and interactions to topology.}
To jointly represent concept scores and interactions, we first express both quantities on a common scalar scale.
While $m_i$ directly quantifies concept $C_i$, the interaction score $\rho_{ij}\in[0,1]$ is normalized and therefore only specifies the relative interaction between $C_i$ and $C_j$.
We seek an absolute interaction level that increases with $\rho_{ij}$, vanishes when $\rho_{ij}=0$, and cannot exceed the score of either concept: an interaction between two concepts should not be assigned a larger value than the extent to which either of them is itself represented.
Accordingly, we rescale $\rho_{ij}$ by the largest admissible value $\min(m_i,m_j)$ and define:
%\iffalse
\begin{equation}
\label{eq_ctr_saddle}
s_{ij}
=
\min(m_i,m_j)\rho_{ij}.
\end{equation}
%\fi
%$s_{ij} = \min(m_i,m_j)\rho_{ij}$.
Thus, $s_{ij}$ varies continuously from $0$ for no interaction to $\min(m_i,m_j)$ for maximal interaction. %, while remaining on the same scale as the individual concept scores.
% This places the unary concept information ${m_i}$ and pairwise interaction information ${s_{ij}}$ on a common scalar scale.
% This makes the concept score ${m_i}$ and pairwise interaction value ${s_{ij}}$ on a common scalar scale.
Intuitively, the difference $m_i - s_{ij}$ represents how much $C_i$ is \emph{independent} from $C_j$.

We use these quantities to define a scalar field on a graph $\mathcal{M}$ that we refer to as the \emph{concept mesh}.
The graph contains one concept vertex $v_i$ of value $m_i$ for each $C_i$, together with a connector vertex $v_{ij}$ of value $s_{ij}$ connected to $v_i$ and $v_j$.
Since $s_{ij}\leq\min(m_i,m_j)$, concept vertices form maxima of this scalar field, while connector vertices encode their pairwise interactions below the corresponding concept values.
We additionally introduce a common vertex with value below all other vertices to assign a finite reference level to the component that survives all pairwise merges.

\iffalse
% Since concept vertices correspond to maxima, 
We consider the upper-star counterpart of the filtration introduced in \autoref{sec_preliminaries}. %, equivalently obtained by applying the lower-star filtration to the negated scalar field.
Starting from the highest values, each concept maximum creates a connected component and connector vertices progressively merge these components.
Following the Elder rule, the component associated with the lower maximum dies when it joins an older component.
% For consistency with the persistence-diagram convention introduced in \autoref{sec_preliminaries},
We represent the resulting feature associated with $C_i$ % as:
as $p_i = (s_i,m_i)$,
\iffalse
\begin{equation}
\label{eq_ctr_pair}
p_i=(s_i,m_i),
\end{equation}
\fi 
where $s_i$ denotes the connector value at which the component born at $m_i$ merges with an older component.
% TODO better formulate next sentence
Its persistence \(m_i-s_i\) measures the hierarchical prominence of \(C_i\), i.e., the range of interaction levels over which it remains the representative maximum of its concept component.
Finally,
\mbox{$\acronym{}(Z)=\{p_i\}_{i=1}^{n_C}$}
constitutes the \acronymLongSing{} of $Z$.
% TODO state that TCR is a persistence diagram
Expressing this organization through \acronym{} enables the use of established methods, such as comparisons with Wasserstein distances and differentiable topological optimization, which we exploit in the following section.
\fi

We apply the upper-star counterpart of the filtration introduced in \autoref{sec_preliminaries}, processing vertices from highest to lowest values.
Concept maxima create connected components and connector vertices merge them; following the Elder rule, we associate $C_i$ with $p_i=(s_i,m_i)$, where $s_i$ is the connector value at which its component merges with an older one.
% Its persistence $m_i-s_i$ measures the hierarchical prominence of $C_i$, i.e., the range of interaction levels over which it remains the representative maximum of its component.
Finally, $\acronym{}(Z)=\{p_i\}_{i=1}^{n_C}$ is the persistence diagram defining the \acronymLongSing{} of $Z$.
This representation enables Wasserstein comparisons and differentiable topological optimization, which we exploit next.

\subsection{Editing \acronymLong{}}
\label{sec_ctr_editing}

Having constructed a topological description of the concept organization, we now turn it into an operational interface to specify interventions on a representation $\representation$.
% \acronym{} separates \emph{which} concept organization is desired from \emph{how} the representation should be modified to realize it.
Let $Z_\theta=T_\theta(\representation)$ denote the representation obtained by applying a differentiable transformation $T_\theta:\mathbb{R}^{d}\rightarrow\mathbb{R}^{d'}$ row-wise to $\representation$.
Starting from $\acronym{}(\representation)$, an intervention is specified by defining a target $\acronym{}^\star$ through changes to its concept maxima and interaction structure.
Given a pair $(s^\star_{\mathcal{A}\mathcal{B}}, m_i^\star)$ in $\acronym{}^\star$, examples include:
\begin{itemize}
    \item $m_i^\star \downarrow$ / $m_i^\star \uparrow$: weaken / strengthen concept $C_i$; while $m_i^\star = m_i$ preserves it;
    % \item $m_i^\star = m_i$: preserve concept $C_i$;
    \item $s_{\mathcal{A}\mathcal{B}}^\star \downarrow$ / $s_{\mathcal{A}\mathcal{B}}^\star \uparrow$: disentangle / entangle the sets of concepts $\mathcal{A}$ and $\mathcal{B}$ that merged at $s^\star_{\mathcal{A}\mathcal{B}}$;
    \item $\acronym{}^\star=\acronym{}(Z^\star)$: transfer the entire concept organization from a representation $Z^\star$.
\end{itemize}
These operations can be combined within the same target, e.g., suppressing an undesired concept while preserving the task and selected concepts. Importantly, the target specifies an \emph{editing criterion}, rather than an editing algorithm. Concept erasure, for instance, can be expressed by setting \mbox{$m_i^\star=0$} while preserving protected concepts. More generally, varying the concept score $S$ and target constraints allows \acronym{} to capture a broad class of existing erasure criteria at the objective level, independently of the optimization used to enforce them; we further examine this relation in \autoref{sec_ctr_properties}.

%%%%%%%%%%%%%%%%%%%%%%%%%%%%%%%%%%%%%%%
%%%%%%%%%%%%%%%%%%%%%%%%%%%%%%%%%%%%%%%

\myParagraph{Topological objective.}
Given $\acronym{}(Z_\theta)$ and $\acronym{}^\star$, we measure how closely the transformed representation $Z_\theta$ realizes the desired concept organization using the $L^q$-Wasserstein distance (\autoref{sec_preliminaries}):
\begin{equation}
\label{eq_ctr_loss}
\mathcal{L}_{\acronym{}}(\theta)
=
\wasserstein{q}\bigl(
\acronym{}(Z_\theta),
\acronym{}^\star
\bigr)^q.
%= \min_{\phi\in\Phi} \sum_{p\in \acronym{}(Z_\theta)'} \pointMetric_q\!\big(p,\phi(p)\big)^q.
\end{equation}
When concept correspondences are known (possibly only on a subset) this semantic information can be exploited by constraining the Wasserstein assignment to match the corresponding features. %, while the remaining features are optimally assigned by the standard solver.
% Semantic correspondences may be specified for only a subset of concepts and when none are available, the assignment reduces to the usual \mbox{Wasserstein matching}.
% The objective therefore compares the current and target concept organizations through a common and well-defined topological framework while preserving known semantic correspondences.

% At each optimization step, we compute $Z_\theta$, refit the concept probes, reconstruct its \acronym{}, solve the corresponding persistence matching, and differentiate \autoref{eq_ctr_loss} through this construction to update $\theta$.
% Differentiating through probe fitting ensures that recoverability is evaluated using the best probe for the current edited representation rather than a direction fixed before the intervention.
% The framework is agnostic to the parameterization of $T_\theta$; in our experiments, we consider post-hoc linear and nonlinear transformations while keeping the representation-producing network fixed.

%%%%%%%%%%%%%%%%%%%%%%%%%%%%%%%%%%%%%%%
%%%%%%%%%%%%%%%%%%%%%%%%%%%%%%%%%%%%%%%

\myParagraph{Topological optimization.}
To realize the target, we optimize $\theta$ through persistence-based topological optimization (\autoref{sec_preliminaries}).
For a given $Z_\theta$, we refit the concept probes, reconstruct the concept mesh and its \acronym{}.
The gradient of Eq.~5 is propagated through the persistence map \(\persistenceMap\) to the filtration values of the concept mesh, and through the filtration map \(\filtrationMap\) to the concept maxima and interaction saddles involved in the matched persistence features.
Since these values are differentiable functions of $Z_\theta$, through the differentiable concept scores and interactions,
%and differentiable probe fitting, 
the gradient further propagates through $Z_\theta=T_\theta(\representation)$ to update $\theta$.
% Differentiating through probe fitting ensures that concept scores reflect the optimal probes for the current transformed representation rather than directions fixed before the intervention.
% Thus, the target specifies \emph{which} concept organization is desired, while the optimization determines \emph{how} the representation is modified to realize it.
% Thus, \acronym{} provides an operational interface separating \emph{which} concept organization is desired from \emph{how} the representation should be modified to realize it.
The framework is agnostic to the parameterization of $T_\theta$, we consider post-hoc linear or nonlinear transformations, or directly the network producing $\representation$.
% We next detail two key use cases: concept disentanglement and transfer.
% TODO Different applications correspond to different target topologies rather than entirely different editing algorithms.

%%%%%%%%%%%%%%%%%%%%%%%%%%%%%%%%%%%%%%%
%%%%%%%%%%%%%%%%%%%%%%%%%%%%%%%%%%%%%%%

\myParagraph{Concept disentanglement.}
\iffalse
Considering two components $\mathcal{A}$ and $\mathcal{B}$ immediately before they merge in the filtration at a saddle $s_{\mathcal{A}\mathcal{B}}=\max_{C_i\in\mathcal{A},C_j\in\mathcal{B}}s_{ij}$.
By construction, it is the strongest interaction between the two groups, while interactions internal to either component are summarized within their existing structure.
Given the persistence pair $(s_{\mathcal{A}\mathcal{B}},m_{\mathcal{A}\mathcal{B}})$, we disentangle $\mathcal{A}$ and $\mathcal{B}$ by replacing this feature in the target \acronym{} by $(0,m_{\mathcal{A}\mathcal{B}}^\star)$.
The target maximum $m_{\mathcal{A}\mathcal{B}}^\star$ can be set to $m_{\mathcal{A}\mathcal{B}}$ to preserve it or to its current value to leave it unconstrained, so that only the saddle contributes to the objective.
Crucially, the saddle is selected dynamically, as its value decreases and the \acronym{} is recomputed, another cross-group interaction may become the strongest and take its place.
The optimization therefore successively acts on the saddle currently maintaining the strongest connection between the groups, effectively reducing it without having to specify all pairwise cross-group interaction targets.
\fi
Given two sets of concepts \(\mathcal A\) and \(\mathcal B\), we define their structural disentanglement by lowering the persistence saddles at which components containing concepts from respectively each set merge. Several such merges may occur at different locations of the hierarchy, involving different subsets of \(\mathcal A\) and \(\mathcal B\); their saddle coordinates are jointly targeted while selected concept maxima can be preserved. As the \acronym{} is recomputed during optimization, the selected saddles adapt to the current concept organization. This operation therefore edits the hierarchical connectivity between the two concept sets, rather than a prescribed individual pairwise interaction.

%%%%%%%%%%%%%%%%%%%%%%%%%%%%%%%%%%%%%%%
%%%%%%%%%%%%%%%%%%%%%%%%%%%%%%%%%%%%%%%

\myParagraph{Concept transfer.}
% A target \acronym{} does not have to be obtained by modifying the original one.
Given two representations $Z^A$ and $Z^B$, possibly produced by different models or having different dimensionalities, $\acronym{}(Z^A)$ can serve as the target for $Z^B$. % whenever the relevant concepts can be identified or matched between them.
Thus, concept transfer 
%is formulated with the same objective as local editing, but 
replaces a manually specified target by the concept organization from another representation.

%%%%%%%%%%%%%%%%%%%%%%%%%%%%%%%%%%%%%%%%%%%%%%%%%%%%%%%%%%%%%%%%%%%%%%%%%%%%%%%
%%%%%%%%%%%%%%%%%%%%%%%%%%%%%%%%%%%%%%%%%%%%%%%%%%%%%%%%%%%%%%%%%%%%%%%%%%%%%%%
%%%%%%%%%%%%%%%%%%%%%%%%%%%%%%%%%%%%%%%%%%%%%%%%%%%%%%%%%%%%%%%%%%%%%%%%%%%%%%%

\subsection{Properties of \acronymLong{}}
\label{sec_ctr_properties}

\myParagraph{\textit{(1)} Benefits of the persistence reduction.}
% The persistence reduction is intentionally lossy: among the $\binom{n_C}{2}$ pairwise interactions of the concept mesh, only the $n_C-1$ interactions that induce merges between previously disconnected components contribute to its $0$-dimensional persistence. 
% Rather than matching every interaction independently, \acronym{} therefore retains the interactions that determine the connectivity of the concept structure across scales, while being invariant to changes in currently redundant interactions. 
% This reduction also yields an adaptive objective for editing: as a topology-defining interaction is modified, another previously redundant interaction may become responsible for a merge and consequently enter the persistence pairing. Thus, \acronym{} does not aim to reconstruct the full interaction matrix, but provides a compact representation of its multiscale connectivity together with a common objective space in which structural interventions can be expressed and optimized.
Rather than reproducing every pairwise interaction of the concept mesh, \acronym{} summarizes how 
such local interactions combine globally to form groups of concepts and how they organize 
% these local relations jointly organize concepts into groups
% these local relations globally form groups of concepts and how they organize 
across interaction levels, beyond what individual interaction scores reveal in isolation.
% It does so by reducing the $\binom{n_C}{2}$ interactions to only the $n_C-1$ 
% saddles being responsible for this organization. 
% saddles needed to characterize the connectivity of the concept mesh across scales.
% ... correspond exactly to what is computed by persistence.
Among the $\binom{n_C}{2}$ interactions, the ones being responsible for the concept hierarchy correspond exactly to the $n_C-1$ saddles computed by persistence.
\acronym{} therefore provides a compact multiscale description of the concept organization by identifying the minimum sparse set of interactions through which this organization can be modified. %, giving an inductive bias toward organization rather than individual pairwise values.
Consequently, it is invariant to changes not affecting the concept hierarchy in contrast to concept meshes that can differ substantially entry-wise while still inducing the same concept hierarchy.
% Moreover, directly comparing concept meshes requires aligned concept identities or an explicit graph-matching criterion, whereas persistence diagrams provide permutation-invariant objects that can be compared without such correspondence.
\iffalse
Representing this organization as a persistence diagram further enables the use of an established operational framework including Wasserstein comparisons
% through assignment 
and optimization, along with theoretical properties (e.g., stability and convergence), 
%in contrast to concept meshes that would require less canonical additional choices.
%being, in contrast, less canonical if working directly with concept meshes.
% whereas working directly with concept meshes requires additional choices.
whereas using directly concept meshes requires additional and less canonical choices.
\fi
Representing this organization as a persistence diagram provides established tools for Wasserstein comparison and differentiable optimization, together with stability and convergence properties, whereas operating directly on concept meshes instead requires additional choices of correspondence, comparison and optimization criteria.

\myParagraph{\textit{(2)} Relation to existing erasure objectives.}
With appropriate choices of the concept scores, several existing erasure criteria, such as those of the baselines considered in our experiments, arise as special cases of \acronym{} at the objective or constraint level. Details and proofs are given in \autoref{sec_app_generalizing}.

\myParagraph{\textit{(3)} Affine reparameterization invariance.}
For the MSE--overlap instantiation, \acronym{} depends on the intrinsic concept organization represented rather than on a particular affine parameterization of the feature space.
Hence, \acronym{} is invariant to invertible affine feature changes, see \autoref{sec_app_invariance}.

\begin{proposition}[Affine reparameterization invariance]
\label{prop_ctr_invariance}
Let \mbox{$Z\in\mathbb{R}^{N\times d}$} be a representation of \mbox{$N$} samples and let $\widetilde Z = ZB^\top + \mathbf{1}a^\top$ with $B\in\mathbb{R}^{d\times d}$ invertible and $a\in\mathbb{R}^d$.
Assume that \acronym{} is constructed using the MSE concept scores of \autoref{eq_\acronym{}_predictive_gram} (\autoref{sec_app_mse}) and the predictive-overlap interactions of \autoref{eq_app_overlap_cosine}. Then $\acronym{}(\widetilde Z)=\acronym{}(Z)$.
The same result holds for normalized MSE concept scores.
\end{proposition}

\myParagraph{\textit{(4)} Stability.}
Finally, small perturbations of the representation induce proportionally small changes in its \acronym{} under mild local regularity conditions. Proofs and discussions are in \autoref{app_ctr_stability}.

\begin{theorem}[Representation-to-\acronym{} stability]
\label{thm_ctr_representation_stability}
Given \mbox{$\representation\in\mathbb{R}^{N\times d}$}. Assume that \mbox{$\Sigma(\representation)\succ0$} and that every concept represented
in the \acronym{} has strictly positive MSE recoverability,
\mbox{$m_i(\representation)>0$} for every \mbox{$i=1,\ldots,n_C$}. Then there exist
\mbox{$\varepsilon>0$} and \mbox{$C>0$} such that, for every centered representation
\mbox{$\representation'$} satisfying \mbox{$\|\representation-\representation'\|_F<\varepsilon$}, $d_B\left(\acronym{}(\representation), \acronym{}(\representation')\right) \le C\,\|\representation-\representation'\|_F$ with $d_B = \wasserstein{\infty}$.
\end{theorem}

%%%%%%%%%%%%%%%%%%%%%%%%%%%%%%%%%%%%%%%%%%%%%%%%%%%%%%%%%%%%%%%%%%%%%%%%%%%%%%%

% Benefits of Topology
% - Reduce the number of vertices taken into account from the full concept mesh to 2*|C|
% - More than pairwise interactions

% Implementation details
% - Warm start probes
% - Differentiable logistic regression and LEACE (with variants)
% - Gradient stop on some saddle parts

% TODO Formulate RQ and Obs. like AlphaEdit

\section{Experiments \& Results}

We evaluate \acronym{} on two applications: \emph{concept disentanglement before erasure}, studying whether modifying concept interactions can facilitate the removal of undesired information while preserving task utility, and \emph{concept transfer}, investigating whether the concept organization of a teacher model can be transferred to a student. We organize our experiments around the following questions:

\begin{itemize}
    \item[(Q1)] \textit{How does \acronym{} disentanglement preprocessing improves existing erasure methods at similar linear concept leakage and can \acronym{} directly realize effective concept erasure?}
    \item[(Q2)] \textit{How do different experimental settings and data regimes, such as the level of erasure, spurious correlation and label imbalance, affect the task utility preservation \acronym{}?}
    \item[(Q3)] \textit{Can \acronym{} transfer concept information across models while maintaining task accuracy compared with standard knowledge-distillation methods?}
    \item[(Q4)] \textit{How does \acronym{} transfer generalize beyond the concepts explicitly optimized and how is this transfer affected by concepts discovered at increasingly local and fine spatial scales?}
\end{itemize}

%%%%%%%%%%%%%%%%%%%%%%%%%%%%%%%%%%%%%%%%%%%%%%%%%%%%%%%%%%%%%%%%%%%%%%%%%%%%%%%
%%%%%%%%%%%%%%%%%%%%%%%%%%%%%%%%%%%%%%%%%%%%%%%%%%%%%%%%%%%%%%%%%%%%%%%%%%%%%%%
%%%%%%%%%%%%%%%%%%%%%%%%%%%%%%%%%%%%%%%%%%%%%%%%%%%%%%%%%%%%%%%%%%%%%%%%%%%%%%%

\subsection{Concept Disentanglement Before Erasure}
\label{sec_exp_dis}

\myParagraph{Experimental setup.}
We consider as baselines INLP \citep{inlp}, LEACE \citep{leace}, RLACE \citep{rlace}, SPLINCE \citep{splince}, LEOPARD \citep{leopard}, KRaM \citep{kram} and Obliviator \citep{obliviator}. %, covering linear and non-linear methods as well as task-agnostic and task-preserving formulations.
We evaluate on Waterbirds \citep{waterbirds}, Bias in Bios \citep{bias_in_bios}, DIAL-Mention \citep{dial_blodgett,dial_elazar} and Jigsaw \citep{jigsaw}.
We use frozen ImageNet-pretrained ResNet-50 \citep{resnet} average-pooling features for Waterbirds and BERT-base \citep{bert} final-layer \texttt{[CLS]} representations for the three text datasets.
Each dataset defines a classification task together with an unwanted concept to erase.
All methods operate post-hoc on these fixed pretrained representations reduced to 128 dimensions using PCA fitted on the training set before intervention.
More details are provided in \subappref{sec_app_exp_setup_disent}.

% full disentanglement may be impossible if task and concept labels are correlated, it may be impossible to make statistically correlated labels independent while preserving both labels perfectly.

% CTR generalizes good disentanglement even if interference only uses LEACE.

% In practice CTR seems to make easier the tuning of the utility-erasure trade-off because of the many small gradient steps that can be adjusted instead of the few big steps as in INLP

%%%%%%%%%%%%%%%%%%%%%%%%%%%%%%%%%%%%%%%
%%%%%%%%%%%%%%%%%%%%%%%%%%%%%%%%%%%%%%%

% Generated with: /home/mathieu/my-venv/bin/python3 /media/mathieu/Data/Etude/Postdoc/ourPapers/conceptTopologicalRepresentation/figures/table_wga/make_table.py --siunitx --no-math --scriptsize --underline-second-best --select-by-validation-wga --aspect-ratio 1.3 --text-scale 2.6 --style-scale 2.9 --label-spacing 1.08 1.18 --label-point-force 1.35 1.70 --label-pull 0.028 0.040 --label-iterations 2500
% Requires \usepackage{siunitx}
% Optional resizing requires \usepackage{graphicx}
\newcommand{\tableWGATabular}{%
\begin{tabular}{l *{4}{S[table-format=2.2{ {\scriptsize \textbf{\textpm 00.00}}}, detect-weight=true, table-align-text-after=true, mode=text]}}
\toprule
\textbf{Method}              & {\textbf{Waterbirds}}                                                                           & {\textbf{Bias in Bios SN}}                                                                      & {\textbf{Dial Mention}}                                                                         & {\textbf{Jigsaw}}                                                                               \\
\midrule
\textit{Raw (WGA / Leakage)} & {\textit{56.54 / 80.22}}                                                                        & {\textit{78.77 / 98.48}}                                                                        & {\textit{39.20 / 54.57}}                                                                        & {\textit{53.29 / 41.57}}                                                                        \\
\midrule
INLP                         & 27.88 {\scriptsize \textpm 0.00}                                                                & 59.47 {\scriptsize \textpm 0.00}                                                                & 52.05 {\scriptsize \textpm 0.00}                                                                & 55.64 {\scriptsize \textpm 0.00}                                                                \\
LEACE                        & 13.71 {\scriptsize \textpm 0.00}                                                                & 59.47 {\scriptsize \textpm 0.00}                                                                & 50.45 {\scriptsize \textpm 0.00}                                                                & 52.59 {\scriptsize \textpm 0.00}                                                                \\
RLACE                        & 14.55 {\scriptsize \textpm 0.95}                                                                & 59.53 {\scriptsize \textpm 0.92}                                                                & 50.80 {\scriptsize \textpm 0.29}                                                                & 53.27 {\scriptsize \textpm 0.37}                                                                \\
SPLINCE                      & 13.71 {\scriptsize \textpm 0.00}                                                                & 58.98 {\scriptsize \textpm 0.00}                                                                & 50.45 {\scriptsize \textpm 0.00}                                                                & 52.59 {\scriptsize \textpm 0.00}                                                                \\
LEOPARD                      & 13.71 {\scriptsize \textpm 0.00}                                                                & 59.47 {\scriptsize \textpm 0.00}                                                                & 49.95 {\scriptsize \textpm 0.19}                                                                & 52.59 {\scriptsize \textpm 0.00}                                                                \\
KRaM                         & 9.08 {\scriptsize \textpm 4.01}                                                                 & 44.02 {\scriptsize \textpm 3.17}                                                                & 52.25 {\scriptsize \textpm 1.40}                                                                & 53.12 {\scriptsize \textpm 1.41}                                                                \\
Obliviator                   & 18.68 {\scriptsize \textpm 3.46}                                                                & 58.18 {\scriptsize \textpm 1.29}                                                                & 50.56 {\scriptsize \textpm 1.48}                                                                & 52.56 {\scriptsize \textpm 1.07}                                                                \\
\midrule
\acronym{}-DIS + INLP        & 66.01 {\scriptsize \textpm 8.80}                                                                & 81.12 {\scriptsize \textpm 3.85}                                                                & 54.66 {\scriptsize \textpm 1.15}                                                                & 56.71 {\scriptsize \textpm 1.12}                                                                \\
\acronym{}-DIS + LEACE       & 74.45 {\scriptsize \textpm 1.31}                                                                & 81.61 {\scriptsize \textpm 1.26}                                                                & 58.04 {\scriptsize \textpm 0.74}{\llap{\underline{\phantom{58.04 {\scriptsize \textpm 0.74}}}}} & 60.58 {\scriptsize \textpm 0.92}                                                                \\
\acronym{}-DIS + RLACE       & \bfseries 74.66 {\scriptsize \textbf{\textpm 0.96}}                                             & 81.59 {\scriptsize \textpm 1.36}                                                                & 57.78 {\scriptsize \textpm 0.73}                                                                & \bfseries 61.05 {\scriptsize \textbf{\textpm 0.72}}                                             \\
\acronym{}-DIS + SPLINCE     & 74.17 {\scriptsize \textpm 1.24}                                                                & 81.53 {\scriptsize \textpm 0.97}                                                                & 58.03 {\scriptsize \textpm 0.70}                                                                & 60.56 {\scriptsize \textpm 0.99}                                                                \\
\acronym{}-DIS + LEOPARD     & 74.47 {\scriptsize \textpm 0.89}{\llap{\underline{\phantom{74.47 {\scriptsize \textpm 0.89}}}}} & 81.47 {\scriptsize \textpm 1.43}                                                                & 57.83 {\scriptsize \textpm 0.83}                                                                & 60.61 {\scriptsize \textpm 0.87}{\llap{\underline{\phantom{60.61 {\scriptsize \textpm 0.87}}}}} \\
\acronym{}-DIS + KRaM        & 55.64 {\scriptsize \textpm 11.32}                                                               & 57.48 {\scriptsize \textpm 5.85}                                                                & 54.23 {\scriptsize \textpm 1.56}                                                                & 56.92 {\scriptsize \textpm 1.27}                                                                \\
\acronym{}-DIS + Obliviator  & 69.30 {\scriptsize \textpm 7.99}                                                                & 83.50 {\scriptsize \textpm 1.17}{\llap{\underline{\phantom{83.50 {\scriptsize \textpm 1.17}}}}} & \bfseries 58.20 {\scriptsize \textbf{\textpm 1.73}}                                             & 60.01 {\scriptsize \textpm 1.21}                                                                \\
\acronym{}-ERA               & 71.17 {\scriptsize \textpm 1.42}                                                                & \bfseries 86.52 {\scriptsize \textbf{\textpm 1.01}}                                             & 55.70 {\scriptsize \textpm 0.89}                                                                & 57.37 {\scriptsize \textpm 1.37}                                                                \\
\bottomrule
\end{tabular}%
}
\NewDocumentCommand{\tableWGA}{m m m o}{%
\begin{table}[#3]
\caption{#1}
\label{#2}
\centering
\IfNoValueTF{#4}{%
\tableWGATabular
}{%
\resizebox{#4\textwidth}{!}{\tableWGATabular}%
}%
\end{table}%
}

\tableWGA{Test task worst-group accuracy with an adjusted balanced linear concept leakage $< 10\%$.}{table_wga}{b!}

\myParagraph{(Q1)}
To compare task preservation after concept removal, we report the test worst-group accuracy (WGA), averaged over 10 seeds, of methods with an adjusted balanced linear concept leakage below $10\%$ in \autoref{table_wga} (metrics described formally in \subappref{sec_app_exp_setup_disent}).
Disentangling the task and concept with \acronym{} before erasure consistently improves task utility: \acronym{}-DIS increases WGA for every erasure method on all four datasets.
The improvement is particularly pronounced on Waterbirds and Bias in Bios, where the best baseline WGA increases from $27.88$ to $74.66$ and from $59.53$ to $83.50$, respectively.
This also shows that the interference saddle generalizes beyond the erasure operator used to define it, although constructed from LEACE, it facilitates erasure with all considered methods.
Moreover, directly optimizing \acronym{} for concept erasure (\acronym{}-ERA) outperforms every standalone erasure baseline across the four datasets and achieves the highest WGA on Bias in Bios.
Across all datasets, \acronym{}-DIS improves the WGA of each baseline by 21.89 on average, with large gains on Waterbirds (+53.91) and Bias in Bios (+21.31), while \acronym{}-ERA improves WGA by 23.43.

\iffalse
Note that concept leakage is measured through a linear probe trained with a convex loss, rather than as statistical independence from the concept. 
This follows the linear-guardedness perspective of prior concept-erasure work, for instance, LEACE guarantees that, after erasure, no linear predictor can improve over the optimal constant predictor under convex losses \citep{leace}. 
However, direct linear-guardedness need not imply that a downstream model cannot infer the concept indirectly through other preserved information \citep{loglinear_guardedness}.
This distinction
% is particularly relevant 
matters
when task and concept labels are correlated, since preserving task-predictive information may necessarily retain some predictability of the concept. 
We therefore additionally evaluate stronger leakage probes, including oracle-conditioned linear and non-linear probes, in \subappref{sec_app_exp_disent_leak}.
\fi
Note that, following prior work on linear-guardedness, e.g., LEACE \citep{leace}, concept leakage is measured through a linear probe.
We therefore additionally evaluate stronger leakage probes, including oracle-conditioned linear and non-linear probes, in \subappref{sec_app_exp_disent_leak}.
Additional analysis of global representational similarity before and after erasure is provided in \subappref{sec_app_exp_disent_cka}.

\myParagraph{Takeaway.}
\acronym{} disentanglement can effectively serve as preprocessing for a diverse set of erasure methods and systematically facilitates utility-preserving concept erasure. Moreover, \acronym{} can also perform effective task-preserving concept erasure directly within the same framework.

%%%%%%%%%%%%%%%%%%%%%%%%%%%%%%%%%%%%%%%
%%%%%%%%%%%%%%%%%%%%%%%%%%%%%%%%%%%%%%%

\iffalse
\begin{figure}[t]
    \centering
    \begin{subfigure}{0.32\textwidth}
        \centering
        \includegraphics[width=\linewidth]{figures/erasure_utility_tradeoff/erasure_utility_tradeoff_ranges_waterbirds.png}
        \caption{Erasure-utility trade-off.}
        \label{fig_erasure_utility}
    \end{subfigure}
    \hfill
    \begin{subfigure}{0.32\textwidth}
        \centering
        \includegraphics[width=\linewidth]{figures/spurious_and_imbalance/spurious_and_imbalance_dial_mention_li.png}
        \caption{WGA given label imbalance.}
        \label{fig_li}
    \end{subfigure}
    \hfill
    \begin{subfigure}{0.32\textwidth}
        \centering
        \includegraphics[width=\linewidth]{figures/spurious_and_imbalance/spurious_and_imbalance_jigsaw_sc.png}
        \caption{WGA given spurious correlation.}
        \label{fig_sc}
    \end{subfigure}
    \caption{Overall caption.}
    \label{fig_tradeoff}
\end{figure}
\fi

\begin{figure}[t]
    \centering
    \adjustbox{valign=m}{%
        \stackinset{l}{16pt}{t}{0pt}{\scriptsize (a)}{%
            \includegraphics[width=0.297\textwidth]{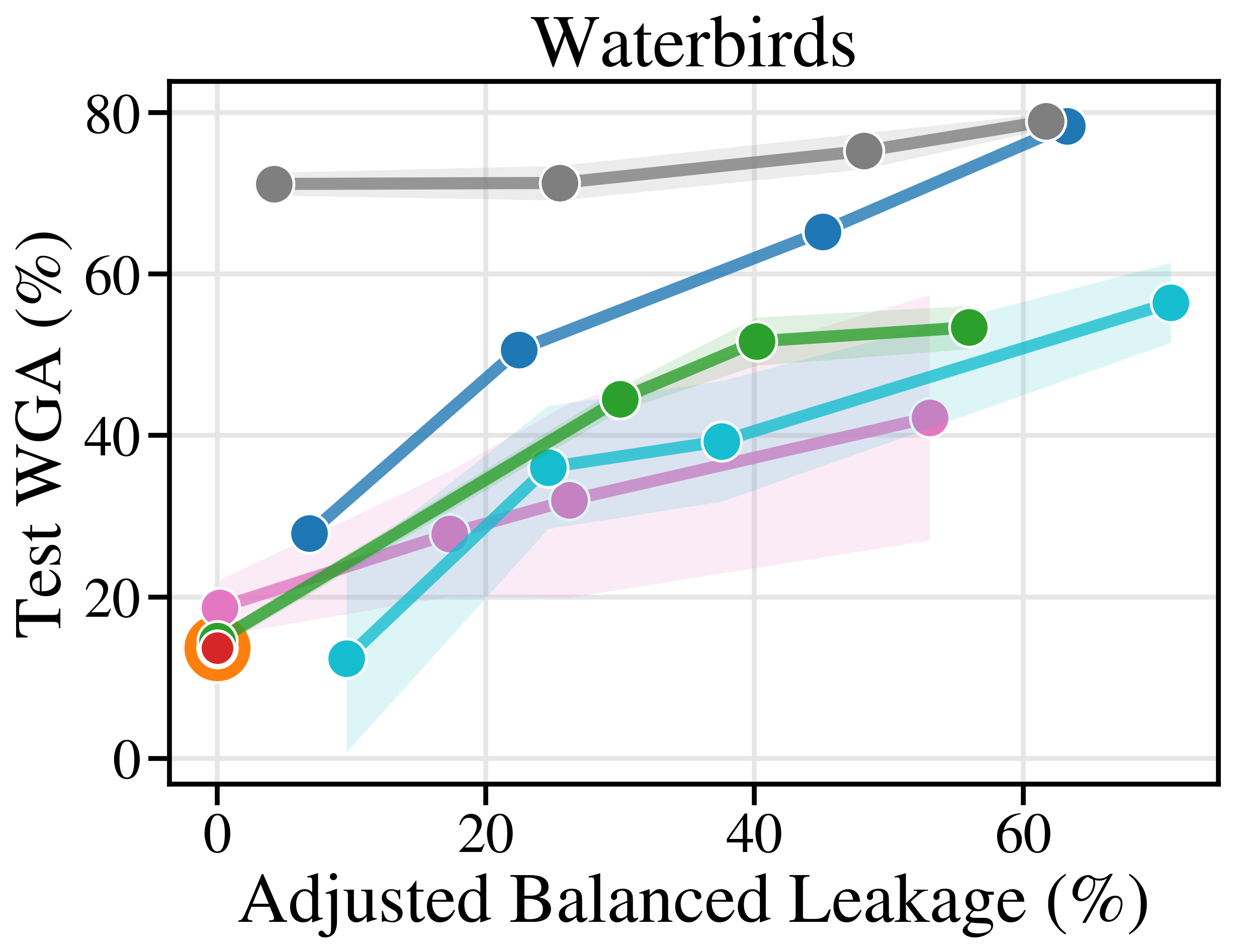}%
        }%
    }\hfill
    \adjustbox{valign=m}{%
        \stackinset{l}{16pt}{t}{0pt}{\scriptsize (b)}{%
            \includegraphics[width=0.297\textwidth]{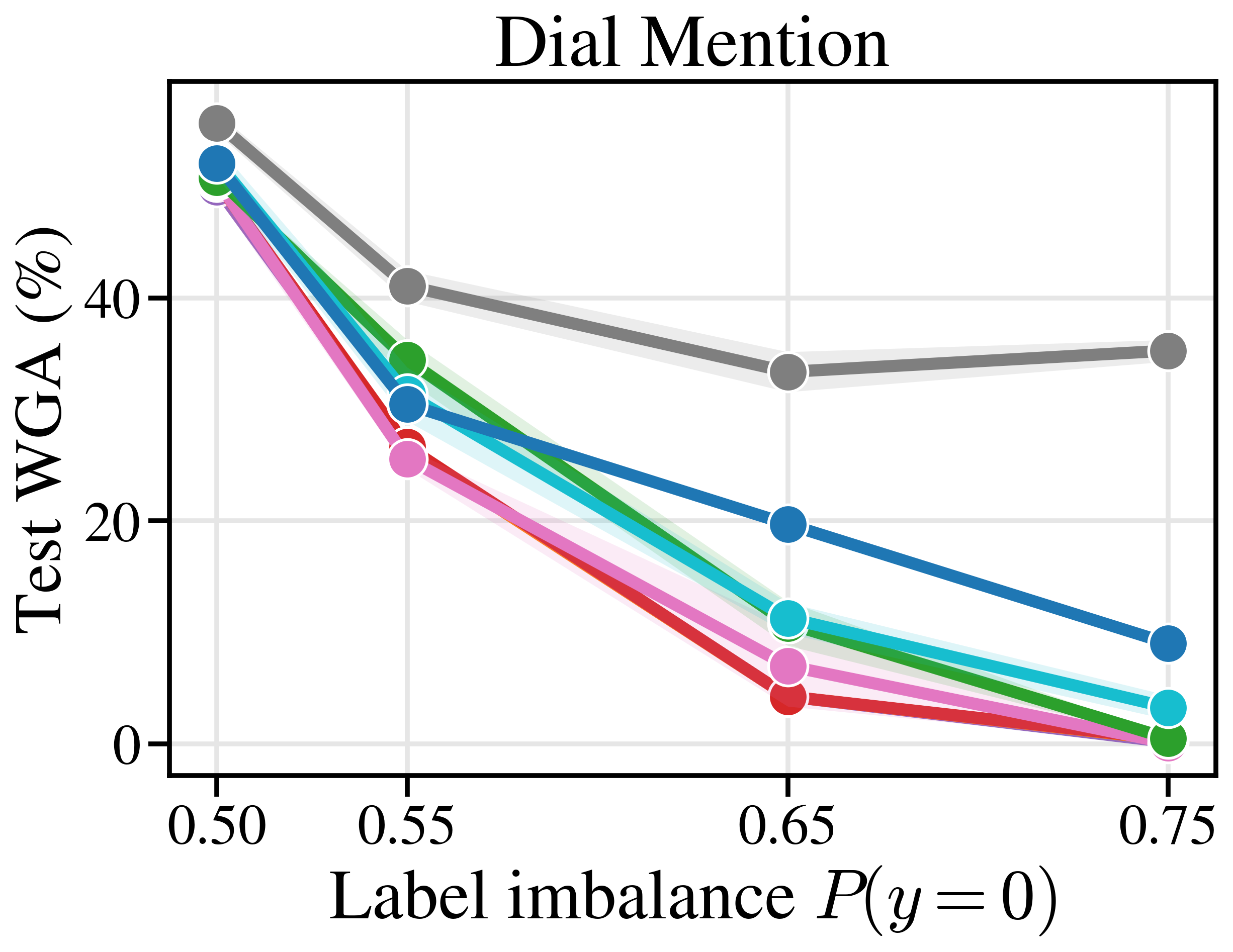}%
        }%
    }\hfill
    \adjustbox{valign=m}{%
        \stackinset{l}{16pt}{t}{0pt}{\scriptsize (c)}{%
            \includegraphics[width=0.297\textwidth]{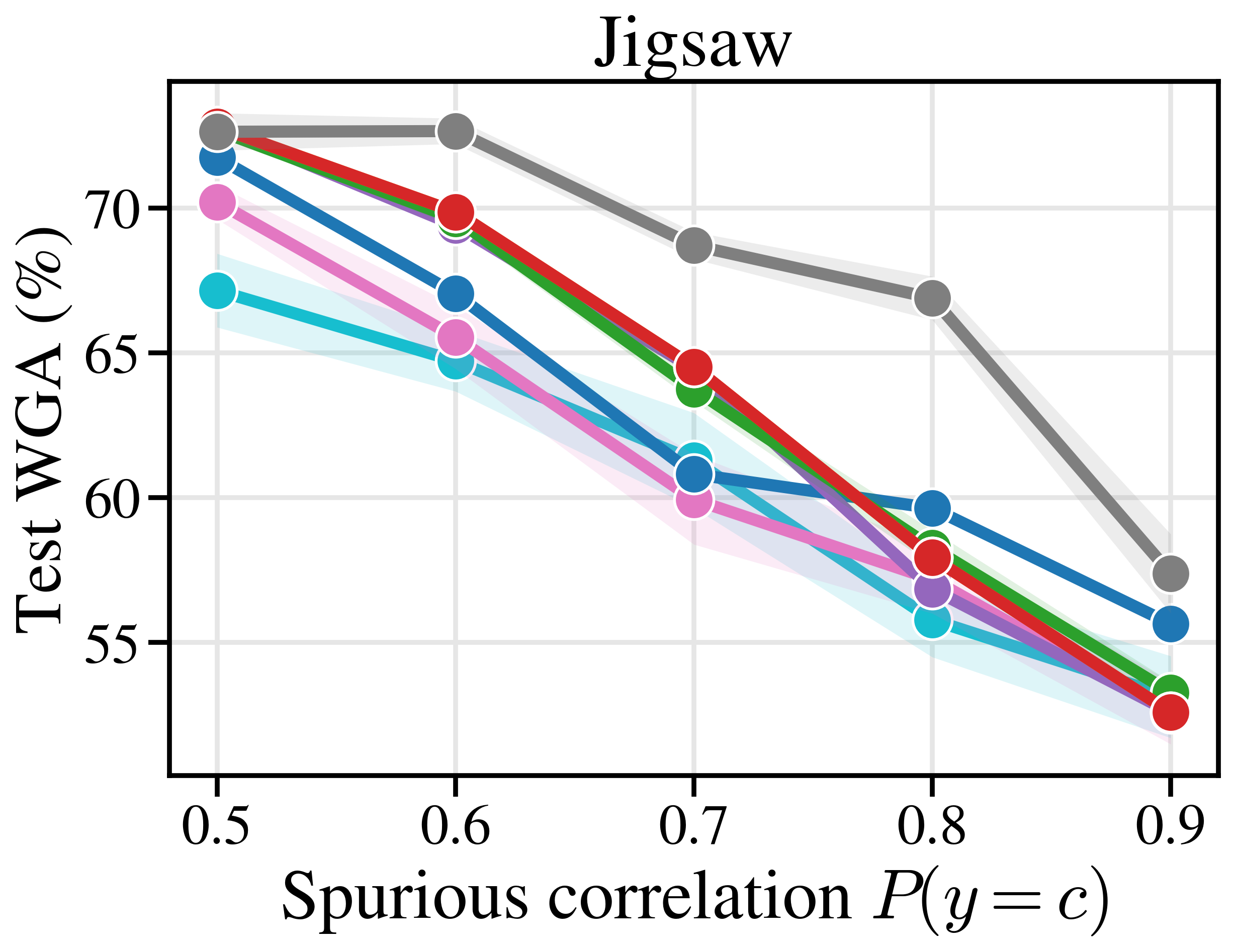}%
        }%
    }
    \raisebox{0.45em}{%
        \adjustbox{valign=m}{%
            \includegraphics[width=0.099\textwidth]{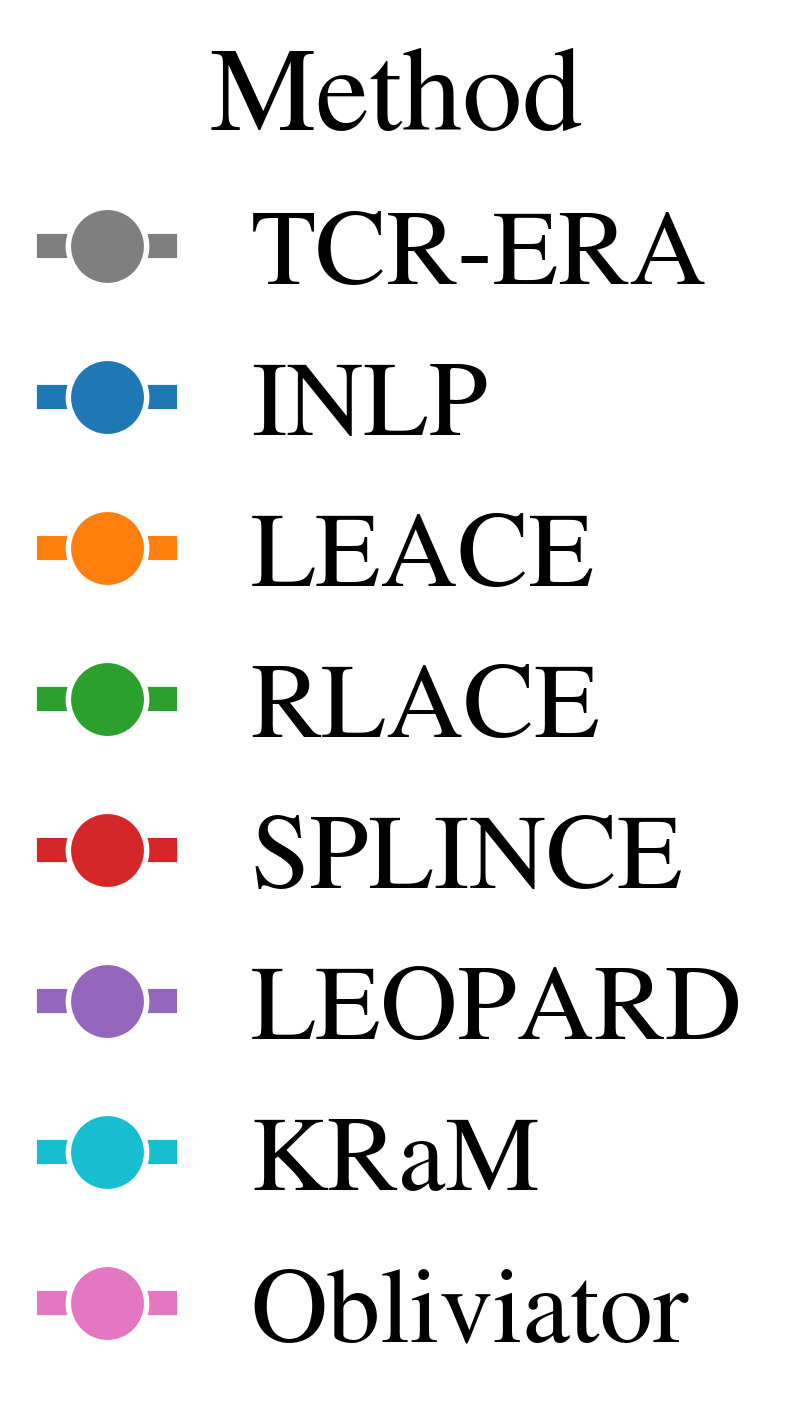}%
        }%
    }
    \caption{Erasure-utility trade-off showing the test WGA given the adjusted balanced linear leakage (a) and impact on test WGA of varying label imbalance (b) and spurious correlation (c).}
    \label{fig_tradeoff}
\end{figure}

\myParagraph{(Q2)}
To assess utility preservation under different erasure strengths and data regimes, we vary the erasure constraint, label imbalance and spurious correlation.
We first examine the erasure-utility trade-off across different erasure strengths.
On Waterbirds (\autoref{fig_tradeoff}a), \acronym{} retains $71.17$ WGA at $4.26\%$ adjusted balanced linear leakage, while INLP reaches $27.88$ WGA at $6.88\%$.
As the erasure constraint is relaxed, several baselines recover substantial task utility and the gap to \acronym{} narrows, whereas \acronym{} retains high WGA already under strong erasure.
We then vary properties of the data distribution while maintaining strong concept erasure.
The effect of increasing label imbalance is illustrated on Dial Mention (\autoref{fig_tradeoff}b), where the margin of \acronym{} over the best baseline at each setting grows from $3.46$ WGA points at $P(y=0)=0.50$ to $26.27$ at $P(y=0)=0.75$.
For increasing spurious correlation, on Jigsaw (\autoref{fig_tradeoff}c), \acronym{} is comparable to the best baseline at $P(y=c)=0.5$ and achieves the highest WGA at every stronger correlation, with a gain of up to $7.25$ points at $P(y=c)=0.8$.
Further erasure-utility and data-regime analyses are reported in \subappref{sec_app_exp_disent_tradeoff}.

\myParagraph{Takeaway.}
\acronym{} maintains a favorable erasure-utility trade-off across erasure strengths and consistently strong utility preservation relative to the baselines across data regimes, with its advantage particularly pronounced under strong erasure, label imbalance and increased spurious correlation.

%%%%%%%%%%%%%%%%%%%%%%%%%%%%%%%%%%%%%%%%%%%%%%%%%%%%%%%%%%%%%%%%%%%%%%%%%%%%%%%
%%%%%%%%%%%%%%%%%%%%%%%%%%%%%%%%%%%%%%%%%%%%%%%%%%%%%%%%%%%%%%%%%%%%%%%%%%%%%%%
%%%%%%%%%%%%%%%%%%%%%%%%%%%%%%%%%%%%%%%%%%%%%%%%%%%%%%%%%%%%%%%%%%%%%%%%%%%%%%%

\subsection{Concept Transfer}
\label{sec_exp_transfer}

%\begingroup
\myParagraph{Experimental setup.}
Following the distillation experimental protocol of \citet{contrastiveKD}, we evaluate concept transfer on CIFAR-100 \citep{cifar100}, using a ResNet32$\times$4 teacher and ResNet8$\times$4 and ShuffleNetV2 students. 
We consider as baselines a student trained without distillation (Student), KD \citep{kd}, FitNet \citep{fitnet}, SPKD \citep{spkd} and TopKD \citep{topkd}. 
For \acronym{}, concepts are discovered from the teacher penultimate representation, from the full images or from 2$\times$2, 3$\times$3 or 4$\times$4 crops \citep{craft}, and their organization is transferred to the student during distillation. Further details are given in \subappref{sec_app_exp_setup_disti}.
%\parfillskip=0pt\par
%\endgroup

%%%%%%%%%%%%%%%%%%%%%%%%%%%%%%%%%%%%%%%
%%%%%%%%%%%%%%%%%%%%%%%%%%%%%%%%%%%%%%%

\begingroup
\myParagraph{(Q3)}
To evaluate how \acronym{} transfers concept information while preserving task performance, we
report in \autoref{table_distillation} test top-1 accuracy and concept recoverability (CR $k_{20}$), being the mean held-out $R^2$ of linear probes predicting 20 teacher-derived concept coefficients
from the student representation, that can be seen as their preservation (detailed in \subappref{sec_app_exp_setup_disti}).
For ResNet8$\times$4, both \acronym{} configurations outperform all baselines on
both metrics: optimizing 10 concepts ($k_{10}$) achieves the highest
\parfillskip=0pt\par
\endgroup

% Generated with: /home/mathieu/my-venv/bin/python3 /media/mathieu/Data/Etude/Postdoc/ourPapers/conceptTopologicalRepresentation/figures/distillation/table.py tex
% Requires \usepackage{siunitx}
% Optional resizing requires \usepackage{graphicx}
\newcommand{\tableDistillationTabular}{%
\begin{tabular}{l *{4}{S[table-format=2.2{ {\scriptsize \textbf{\textpm 00.00}}}, detect-weight=true, table-align-text-after=true, mode=text]}}
\toprule
\textbf{Method} & \multicolumn{2}{c}{\textbf{ResNet8$\times$4}} & \multicolumn{2}{c}{\textbf{ShuffleNetV2}} \\
\cmidrule(lr){2-3}\cmidrule(lr){4-5}
                    & {\textbf{Top-1}}                                                                                & {\textbf{CR $k_{20}$}}                                                                          & {\textbf{Top-1}}                                                                                & {\textbf{CR $k_{20}$}}                                                                          \\
\midrule
Student             & 72.35 {\scriptsize \textpm 0.24}                                                                & 41.50 {\scriptsize \textpm 0.18}                                                                & 72.58 {\scriptsize \textpm 0.23}                                                                & 52.65 {\scriptsize \textpm 0.23}                                                                \\
KD                  & 73.35 {\scriptsize \textpm 0.24}                                                                & 54.07 {\scriptsize \textpm 0.17}                                                                & 74.62 {\scriptsize \textpm 0.19}                                                                & 66.64 {\scriptsize \textpm 0.19}                                                                \\
FitNet              & 73.19 {\scriptsize \textpm 0.23}                                                                & 41.05 {\scriptsize \textpm 0.25}                                                                & 74.25 {\scriptsize \textpm 0.15}                                                                & 53.52 {\scriptsize \textpm 0.45}                                                                \\
SPKD                & 73.00 {\scriptsize \textpm 0.24}                                                                & 54.51 {\scriptsize \textpm 0.19}                                                                & \bfseries 76.04 {\scriptsize \textbf{\textpm 0.15}}                                             & \bfseries 67.62 {\scriptsize \textbf{\textpm 0.23}}                                             \\
TopKD               & 73.20 {\scriptsize \textpm 0.28}                                                                & 54.18 {\scriptsize \textpm 0.25}                                                                & 74.98 {\scriptsize \textpm 0.28}                                                                & 66.58 {\scriptsize \textpm 0.14}                                                                \\
\midrule
\acronym{} $k_{10}$ & \bfseries 73.68 {\scriptsize \textbf{\textpm 0.47}}                                             & 57.18 {\scriptsize \textpm 0.06}{\llap{\underline{\phantom{57.18 {\scriptsize \textpm 0.06}}}}} & 74.97 {\scriptsize \textpm 0.70}                                                                & 66.83 {\scriptsize \textpm 0.37}{\llap{\underline{\phantom{66.83 {\scriptsize \textpm 0.37}}}}} \\
\acronym{} $k_{20}$ & 73.55 {\scriptsize \textpm 0.28}{\llap{\underline{\phantom{73.55 {\scriptsize \textpm 0.28}}}}} & \bfseries 60.05 {\scriptsize \textbf{\textpm 0.35}}                                             & 75.02 {\scriptsize \textpm 0.54}{\llap{\underline{\phantom{75.02 {\scriptsize \textpm 0.54}}}}} & 66.67 {\scriptsize \textpm 0.28}                                                                \\
\bottomrule
\end{tabular}%
}
\NewDocumentCommand{\tableDistillation}{m m m o}{%
\begin{table}[#3]
\caption{#1}
\label{#2}
\centering
\IfNoValueTF{#4}{%
\tableDistillationTabular
}{%
\resizebox{#4\textwidth}{!}{\tableDistillationTabular}%
}%
\end{table}%
}

%\tableDistillation{Caption.}{table_disti}{t!}

\begin{wrapstuff}[
    r,
    width=0.3\textwidth,
    type=figure,
    leftsep=0.02\textwidth,
    voffset=0.8\baselineskip
]
    \centering

    \includegraphics[width=\linewidth]
        {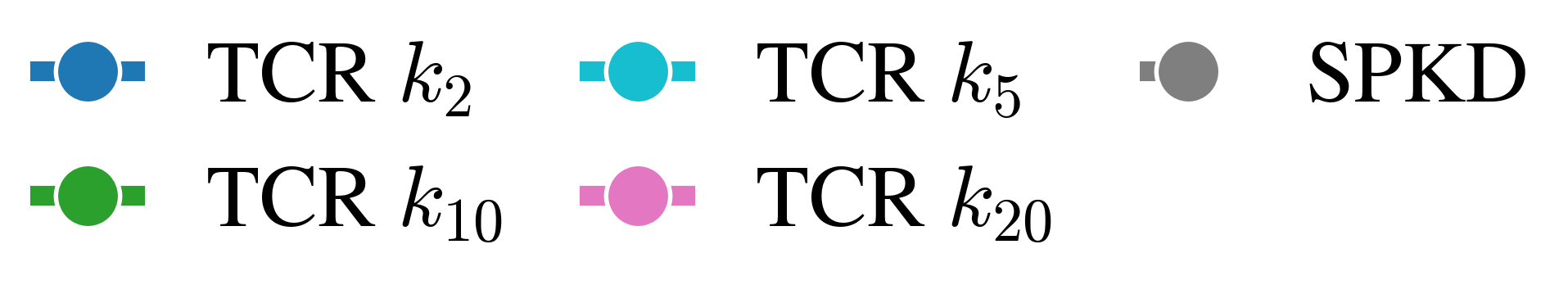}

    \stackinset{l}{10pt}{t}{0pt}{\scriptsize (a)}{%
        \includegraphics[width=\linewidth]
            {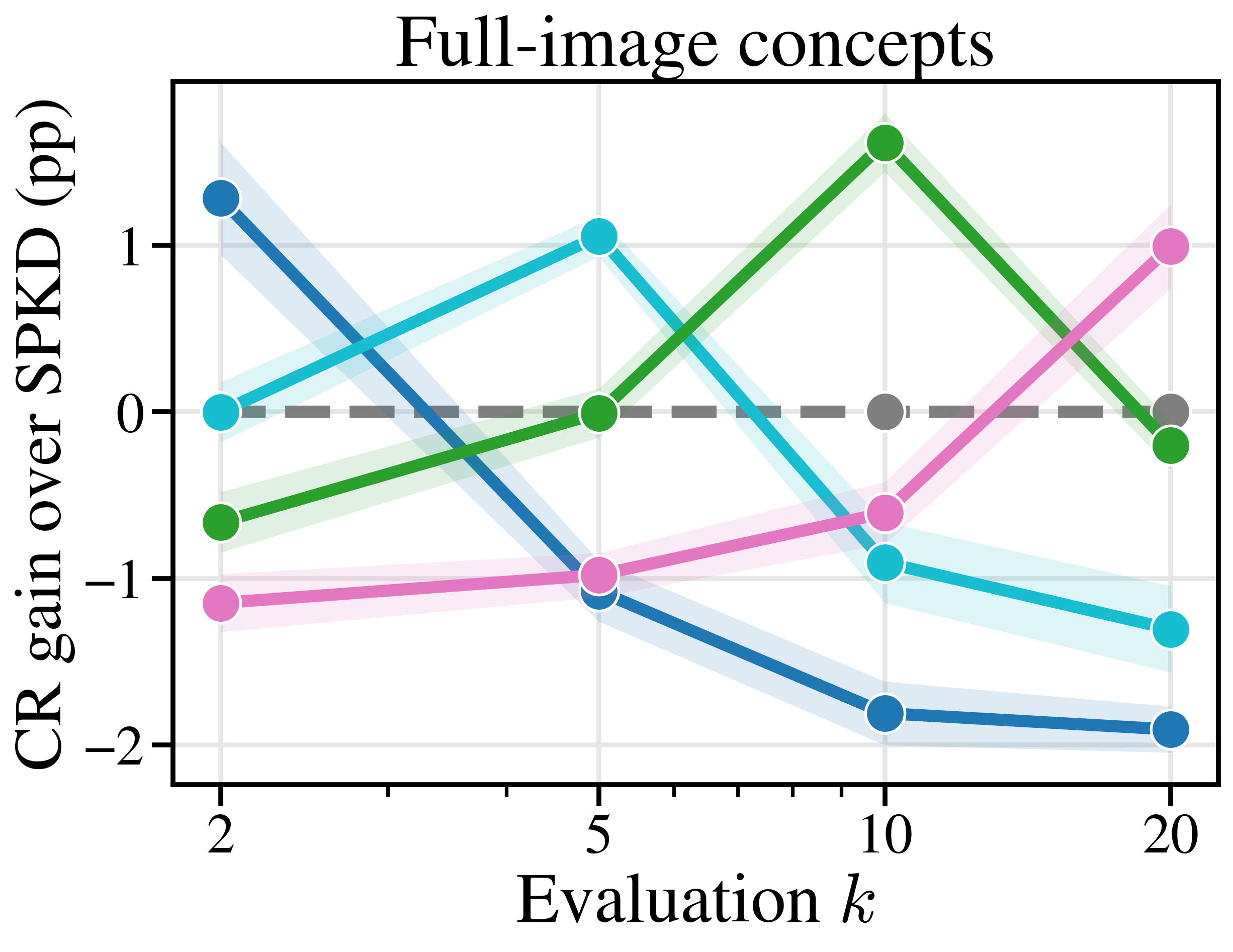}%
    }

    \stackinset{l}{10pt}{t}{0pt}{\scriptsize (b)}{%
        \includegraphics[width=\linewidth]
            {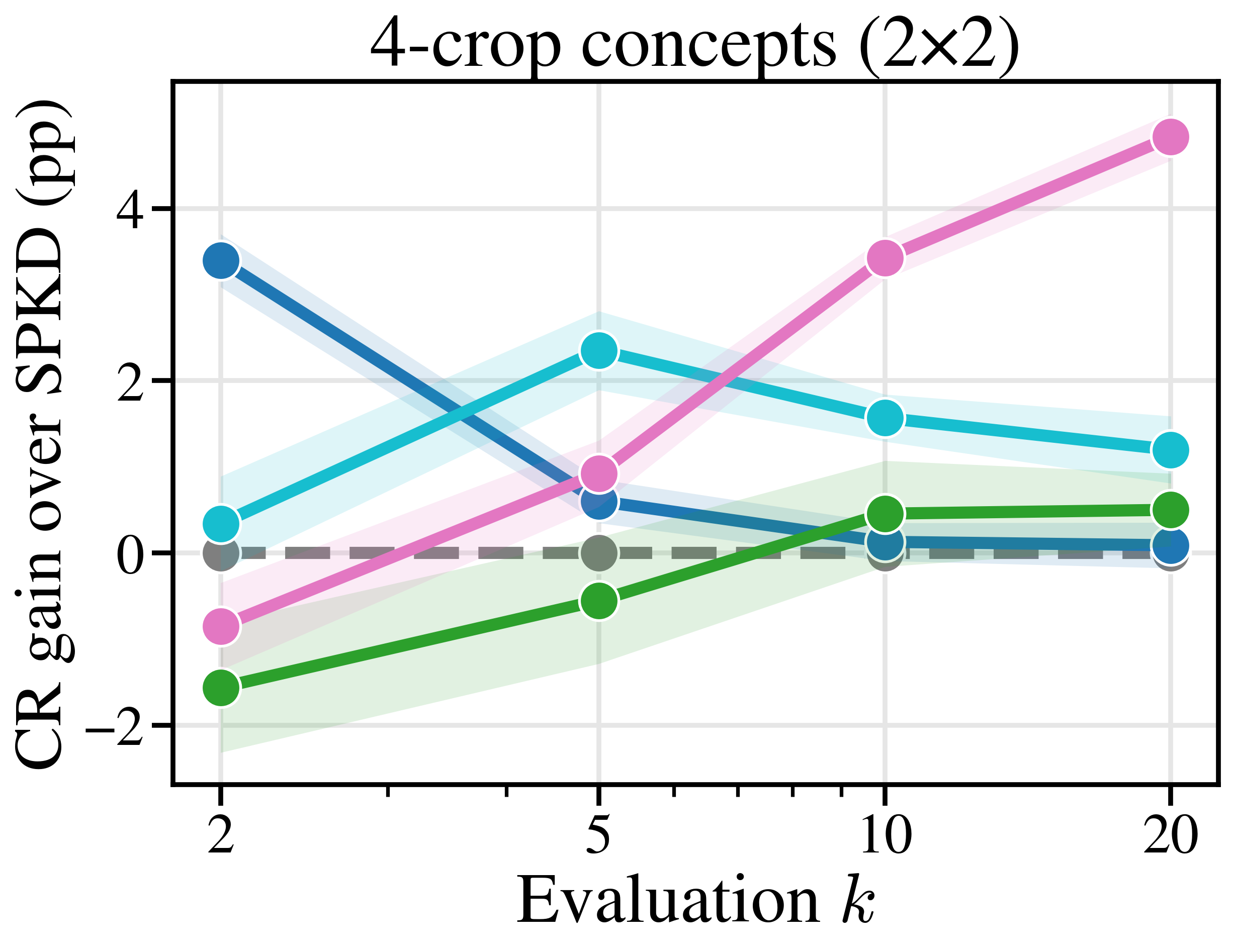}%
    }

    \stackinset{l}{10pt}{t}{0pt}{\scriptsize (c)}{%
        \includegraphics[width=\linewidth]
            {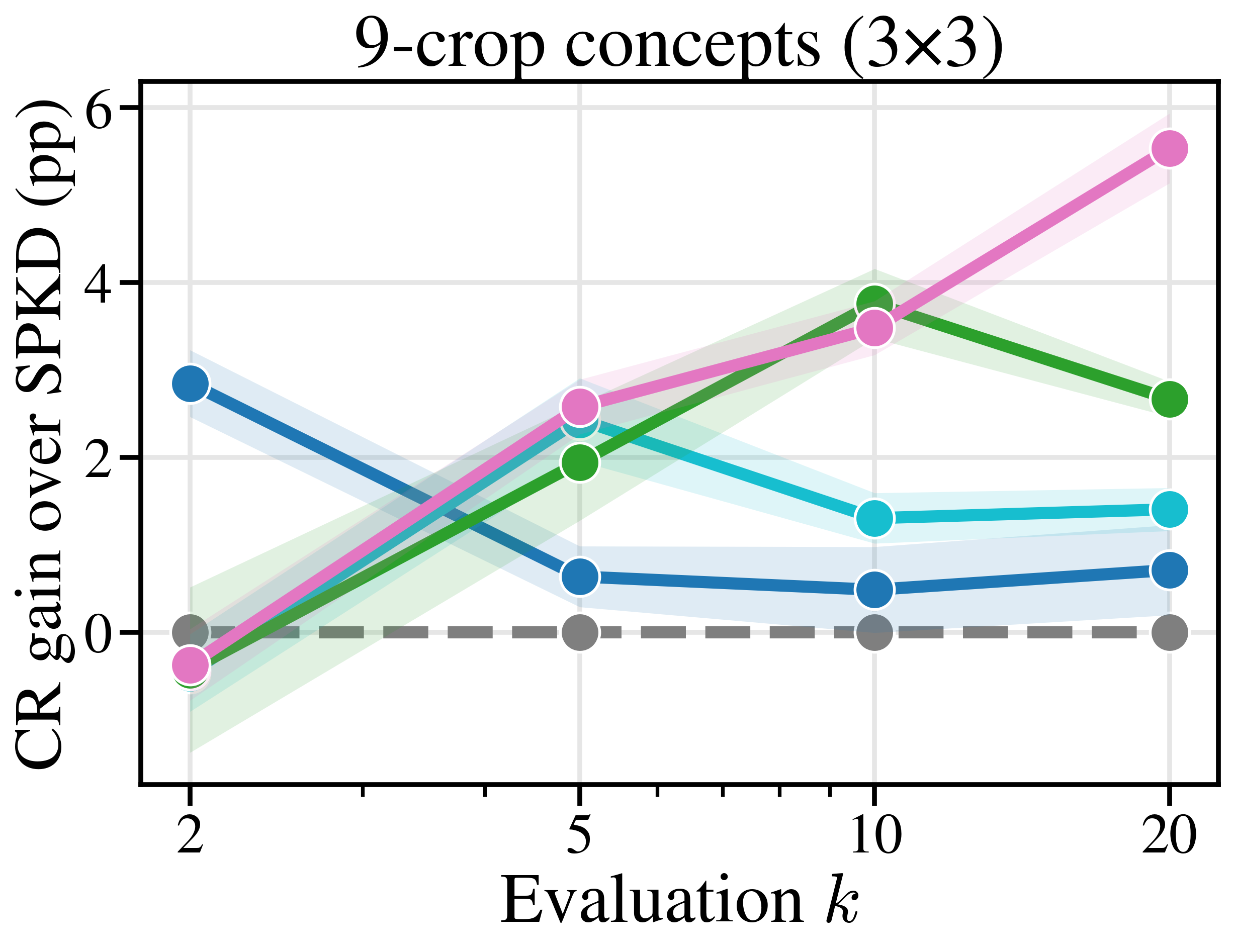}%
    }

    \stackinset{l}{10pt}{t}{0pt}{\scriptsize (d)}{%
        \includegraphics[width=\linewidth]
            {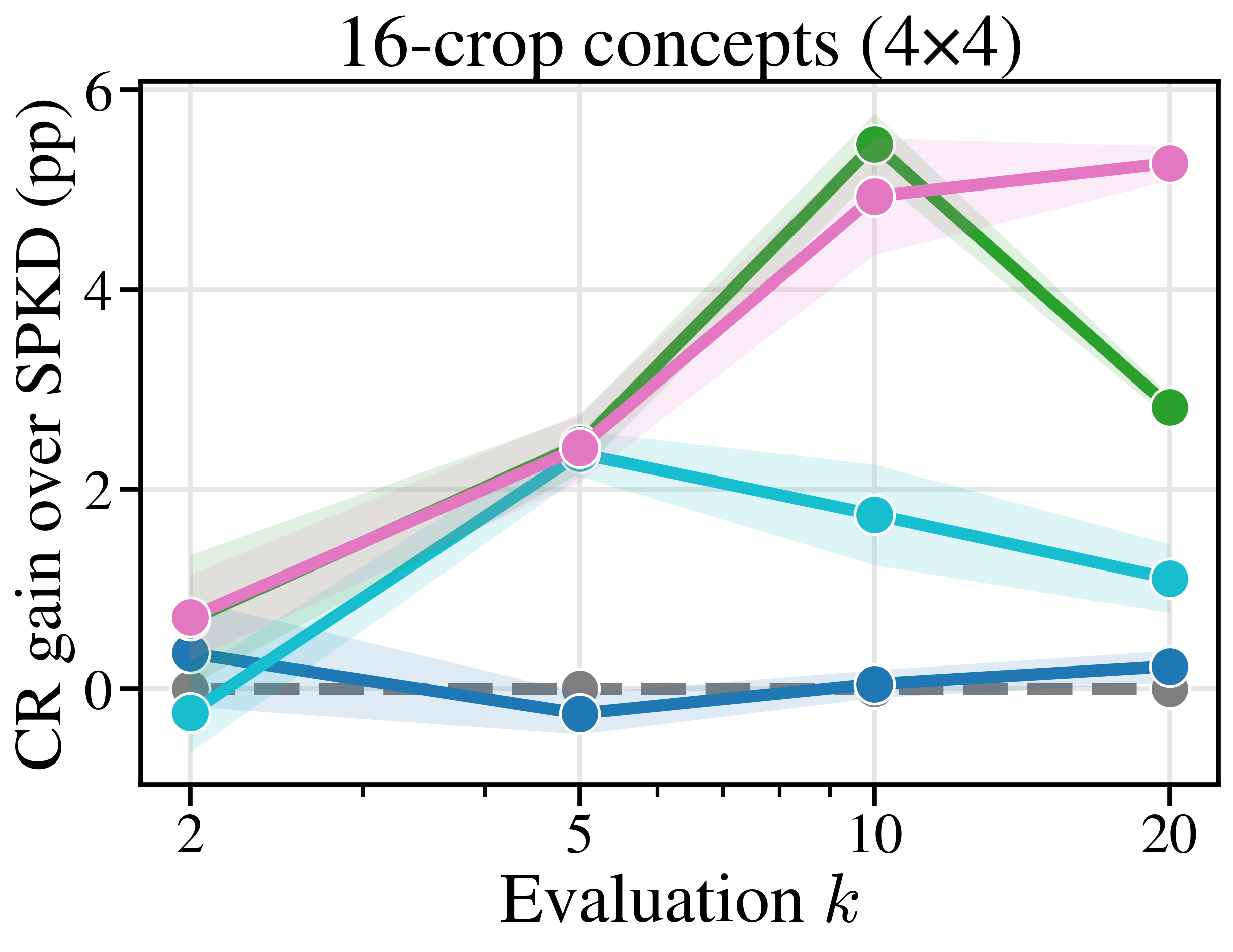}%
    }

    \caption{Impact of concept spatial scale on recoverability.}
    \label{fig_distillation}
\end{wrapstuff}

% Non-floating table occupying the available left-hand width
% Start the wrapped paragraph before entering a minipage; otherwise the
% minipage resets wrapstuff while the figure is still pending.
\noindent\mbox{}\par
\vspace{-\baselineskip}
\noindent
\begin{minipage}{0.67\textwidth}
    \centering
    \captionof{table}{
%Test task accuracy and concept recoverability (CR) for concept transfer using concepts discovered from a $3\times3$ crop grid.
Distillation test top-1 accuracy and concept recoverability.
}
    \label{table_distillation}

    \resizebox{\linewidth}{!}{%
        \tableDistillationTabular
    }
\end{minipage}

\noindent
top-1 accuracy ($73.68$) while $k_{20}$ achieves the highest CR ($60.05$). %, indicating a trade-off between top-1 accuracy and CR.
Compared with the strongest 
baseline for each metric, this corresponds to gains of $0.33$
for top-1 
%points 
over KD and $5.54$ for CR 
%points 
over SPKD, respectively.
Averaged over all baselines, gains are respectively $0.60$ and $9.55$ points.
For ShuffleNetV2, SPKD achieves the strongest results on both metrics
% ($76.04$ top-1 and $67.62$ CR), 
while \acronym{} remains close in CR. % but obtains lower task accuracy. %, reaching $66.73$ CR and $74.82$ top-1 with $k_{10}$.

\myParagraph{Takeaway.}
\acronym{} can transfer teacher-derived concept information to a student while preserving task
performance, simultaneously improving concept recoverability and task accuracy (or maintaining it).
% over standard distillation baselines. 
% for ResNet8$\times$4. On ShuffleNetV2, \acronym{} retains comparable concept recoverability but lower task accuracy.

%%%%%%%%%%%%%%%%%%%%%%%%%%%%%%%%%%%%%%%
%%%%%%%%%%%%%%%%%%%%%%%%%%%%%%%%%%%%%%%

\iffalse
\begin{figure}[t]
    \centering

    \includegraphics[width=0.297\textwidth]{figures/distillation/concept_recoverability_legend_two_rows.png}%
    \\

    \stackinset{l}{10pt}{t}{0pt}{\scriptsize (a)}{%
        \includegraphics[width=0.297\textwidth]{figures/distillation/concept_recoverability_resnet8x4_base.png}%
    }%
    \\

    \stackinset{l}{10pt}{t}{0pt}{\scriptsize (b)}{%
        \includegraphics[width=0.297\textwidth]{figures/distillation/concept_recoverability_resnet8x4_cc_4.png}%
    }%
    \\

    \stackinset{l}{10pt}{t}{0pt}{\scriptsize (c)}{%
        \includegraphics[width=0.297\textwidth]{figures/distillation/concept_recoverability_resnet8x4_cc_9.png}%
    }%
    \\

    \stackinset{l}{10pt}{t}{0pt}{\scriptsize (d)}{%
        \includegraphics[width=0.297\textwidth]{figures/distillation/concept_recoverability_resnet8x4_cc_16.png}%
    }%

    \caption{.}
    \label{fig_distillation}
\end{figure}
\fi

\myParagraph{(Q4)}
To assess whether \acronym{} transfers concept information beyond the concept bank used during
optimization, we train \acronym{} with $k_{\mathrm{train}}\in\{2,5,10,20\}$ concepts and
evaluate CR using all $k_{\mathrm{eval}}\in\{2,5,10,20\}$ concept banks in
\autoref{fig_distillation}.
We report the CR gain over SPKD, the best baseline regarding CR, with $0$ indicating equal recoverability.
Since concept banks for different values of $k$ are independently discovered, settings with
$k_{\mathrm{train}}\neq k_{\mathrm{eval}}$ measure transfer to concepts that were not explicitly
optimized.
For full-image concepts (\autoref{fig_distillation}a), \acronym{} improves over SPKD for every
matched $k_{\mathrm{train}}=k_{\mathrm{eval}}$ setting, but not when evaluated on different
%concept 
bank sizes. %, for which the mean gain is $-0.88$ CR points.
This behavior changes as concepts become increasingly local: the mean gain when
$k_{\mathrm{train}}\neq k_{\mathrm{eval}}$ becomes $+0.48$, $+1.16$ and $+1.39$ 
%points 
for the
2$\times$2, 3$\times$3 and 4$\times$4 crop grids, respectively.
Likewise, the mean gain across all training and evaluation banks increases from $-0.35$ for
full-image concepts to $+1.05$, $+1.78$ and $+1.88$ 
%points
at 
%increasingly 
finer spatial scales.
At the two finest scales, larger training banks yield particularly broad transfer; for example,
training with $k_{20}$ improves CR on the independently discovered $k_{10}$ bank by $3.48$
%points 
for the 3$\times$3 grid and $4.94$ 
%points 
for the 4$\times$4 grid.
These results indicate a transition from primarily bank-specific transfer for global concepts
to broader cross-bank transfer as concepts become more spatially localized.
One possible interpretation is that global concepts may induce more bank-specific structure, whereas local concepts expose more reusable structure that are shared across independently discovered banks.
% One possible interpretation is that full-image concepts induce more bank-specific structure, whereas localized concepts expose structure shared across independently discovered concept banks.
% However, this effect may also partly reflect the increasing difficulty of recovering local concepts with standard distillation: the mean SPKD CR decreases from $78.05$ for full-image concepts to $65.68$, $60.75$ and $57.63$ at increasingly fine spatial scales, leaving more room for explicit concept-structure transfer.
Part of this effect may also arise because localized concepts are less recoverable for the baselines, leaving more room for 
improvement. 
%explicit concept-structure transfer. 
%explicit concept transfer. % of teacher-derived concept information.
% Overall, these results indicate a transition from primarily bank-specific transfer for global concepts to broader transfer across concept banks as concepts become more spatially localized.

\myParagraph{Takeaway.}
For global concepts, \acronym{} primarily improves recoverability of the concept bank used
during optimization. As concepts become increasingly local, its advantage over SPKD grows and
extends to independently discovered concept banks of different sizes. %, particularly when \acronym{} is trained with richer concept vocabularies.

\wrapstuffclear

% Limitations:
% - Dependence on the chosen concept set;
% - Dependence on the probe and erasure operator;
% - Operational interaction versus causal dependence;
% - Pairwise measurements versus genuinely higher-order interactions;
% - Scalability with the number of concepts;
% - Whether topology is unique or only one useful abstraction;
% - Nonlinear information that may remain after linear editing;
% - Distribution-shift sensitivity.

% - PCA dim 128
% - No multi-class

\section{Conclusion}

% Future work:
% - Concept interactions on more than two concepts (give a better reduction from the concept mesh to the persistence diagram)
% - Entanglement
% - CTR on many layers instead of only the penultimate one
% - LLM entanglement of performance and safety guards
% - Visualization tool
% - Also optimize backbone
% - How much users can use that easily, cite work of a guy at TORI workshop
% - How persistence compares to other concept importance measures

We introduced \acronym{}, a topological abstraction that jointly represents concepts and their relationships and provides a common interface for concept-level interventions. Across concept erasure, disentanglement and transfer, our results demonstrate its ability to manipulate concept organization while preserving task-relevant information. % (improving respectively WGA by $21.89$ and CR by up to $5.54$).
For disentanglement, it improves WGA by $21.89$ on average at comparable concept leakage and, for transfer, it improves concept recoverability by up to $5.54$. % while improving or maintaining competitive test top-1 accuracy.

% The following are the requested sections and do not count toward the page limit.

\subsection*{AI use statement}

In this work, we used generative AI tools for 
assisting in the writing of proofs, % 5
designing or providing feedback on research methodology or experiments, % 7
implementing methods, % 8
supporting qualitative and thematic data analysis and % 10
interpreting results. % 11
We have not used generative AI tools for
generating synthetic data sets, % 1
helping developing theoretical models or conceptual frameworks, % 2
formulating mathematical claims, % 3
providing critical ingredients for proving mathematical claims, % 4
proposing or refining hypotheses, % 6
assisting with translation and % 9
cleaning and reformatting dataset.
% and [the rest of the required disclosure tasks] are not applicable to this work.
Additionally, we used generative AI tools for 
% Formulate questions for surveys or interviews, 
% create or modify scientific figures or images, 
suggesting experimental parameters, 
creating or editing software code,
% creation of artifacts, 
% draft parts of a research paper, 
% transcribe recordings of research material, 
summarize or analyse existing literature,
% discover research topics or identify gaps, 
brainstorming,
sourcing/searching for information, 
editing a research paper to improve readability and
identifying relevant literature.
% format references, 
% suggest a structure for a research paper, 
% propose a title or keywords for a research paper.
We have reviewed all AI-assisted work. 
% [Elaborate. For example, “we checked LLM-generated research ideas for potential plagiarism through a manual literature survey”, “LLM-generated code was verified and tested for correctness by 2 authors”, etc.]. 
We take responsibility for the final content of this work, including text, claims or artifacts produced with the aid of generative AI.

\subsection*{Ethics Statement}
Our experiments use existing benchmark datasets, including human-generated text with demographic, dialectal and identity-related annotations or proxies. We do not collect new human-subject data. Such annotations and proxies may be imperfect and should not be interpreted as definitive characterizations of individuals or groups. While concept erasure can help reduce unwanted information in learned representations, removing a selected concept does not by itself guarantee fairness or complete removal of the corresponding information; we therefore evaluate leakage using complementary diagnostics. More generally, representation interventions can be applied to arbitrary concepts, and their use in high-stakes applications requires careful validation of both the information removed and the utility preserved.

\subsection*{Reproducibility Statement}
We provide the code used to produce the results of this paper as supplementary material with the anonymous submission. Detailed experimental protocols, including datasets and splits, model architectures, hyperparameter selection and evaluation metrics, are provided in the appendix. Upon de-anonymization, we will make the code publicly available on GitHub.

\bibliography{iclr2027_conference}

\begin{thebibliography}{56}
\providecommand{\natexlab}[1]{#1}
\providecommand{\url}[1]{\texttt{#1}}
\expandafter\ifx\csname urlstyle\endcsname\relax
  \providecommand{\doi}[1]{doi: #1}\else
  \providecommand{\doi}{doi: \begingroup \urlstyle{rm}\Url}\fi

\bibitem[Akbari et~al.(2025)Akbari, Afshari, and Boddeti]{obliviator}
Ramin Akbari, Milad Afshari, and Vishnu Boddeti.
\newblock Obliviator reveals the cost of nonlinear guardedness in concept
  erasure.
\newblock In D.~Belgrave, C.~Zhang, H.~Lin, R.~Pascanu, P.~Koniusz,
  M.~Ghassemi, and N.~Chen (eds.), \emph{Advances in Neural Information
  Processing Systems}, volume 38, Main Conference, pp.\  144860--144892. Curran
  Associates, Inc., 2025.
\newblock \doi{10.52202/085713-4847}.
\newblock URL
  \url{https://proceedings.neurips.cc/paper_files/paper/2025/file/d55f39791f04745b2e0c8abebf3dd5d7-Paper-Conference.pdf}.

\bibitem[Ansuini et~al.(2019)Ansuini, Laio, Macke, and
  Zoccolan]{ansuiniRepresentation}
Alessio Ansuini, Alessandro Laio, Jakob~H. Macke, and Davide Zoccolan.
\newblock \emph{Intrinsic dimension of data representations in deep neural
  networks}.
\newblock Curran Associates Inc., Red Hook, NY, USA, 2019.

\bibitem[Bau et~al.(2017)Bau, Zhou, Khosla, Oliva, and
  Torralba]{networkDissection}
David Bau, Bolei Zhou, Aditya Khosla, Aude Oliva, and Antonio Torralba.
\newblock Network dissection: Quantifying interpretability of deep visual
  representations.
\newblock In \emph{2017 IEEE Conference on Computer Vision and Pattern
  Recognition (CVPR)}, pp.\  3319--3327, 2017.
\newblock \doi{10.1109/CVPR.2017.354}.

\bibitem[Belrose et~al.(2023)Belrose, Schneider-Joseph, Ravfogel, Cotterell,
  Raff, and Biderman]{leace}
Nora Belrose, David Schneider-Joseph, Shauli Ravfogel, Ryan Cotterell, Edward
  Raff, and Stella Biderman.
\newblock Leace: Perfect linear concept erasure in closed form.
\newblock In A.~Oh, T.~Naumann, A.~Globerson, K.~Saenko, M.~Hardt, and
  S.~Levine (eds.), \emph{Advances in Neural Information Processing Systems},
  volume~36, pp.\  66044--66063. Curran Associates, Inc., 2023.
\newblock \doi{10.52202/075280-2884}.
\newblock URL
  \url{https://proceedings.neurips.cc/paper_files/paper/2023/file/d066d21c619d0a78c5b557fa3291a8f4-Paper-Conference.pdf}.

\bibitem[Bertsekas(1981)]{Bertsekas81}
Dimitri~P. Bertsekas.
\newblock A new algorithm for the assignment problem.
\newblock \emph{Mathematical Programming}, 21\penalty0 (1):\penalty0 152--171,
  1981.
\newblock \doi{10.1007/BF01584237}.

\bibitem[Blodgett et~al.(2016)Blodgett, Green, and O'Connor]{dial_blodgett}
Su~Lin Blodgett, Lisa Green, and Brendan~T. O'Connor.
\newblock Demographic dialectal variation in social media: A case study of
  african-american english.
\newblock In Jian Su, Xavier Carreras, and Kevin Duh (eds.), \emph{EMNLP}, pp.\
   1119--1130. The Association for Computational Linguistics, 2016.
\newblock ISBN 978-1-945626-25-8.
\newblock URL
  \url{http://dblp.uni-trier.de/db/conf/emnlp/emnlp2016.html#BlodgettGO16}.

\bibitem[Borkan et~al.(2019)Borkan, Dixon, Sorensen, Thain, and
  Vasserman]{jigsaw}
Daniel Borkan, Lucas Dixon, Jeffrey Sorensen, Nithum Thain, and Lucy Vasserman.
\newblock Nuanced metrics for measuring unintended bias with real data for text
  classification.
\newblock In Sihem Amer-Yahia, Mohammad Mahdian, Ashish Goel, Geert-Jan Houben,
  Kristina Lerman, Julian~J. McAuley, Ricardo Baeza-Yates, and Leila Zia
  (eds.), \emph{WWW (Companion Volume)}, pp.\  491--500. ACM, 2019.
\newblock ISBN 978-1-4503-6675-5.
\newblock URL
  \url{http://dblp.uni-trier.de/db/conf/www/www2019c.html#BorkanDSTV19}.

\bibitem[Bubenik \& Elchesen(2022)Bubenik and Elchesen]{bubenik22}
Peter Bubenik and Alex Elchesen.
\newblock Universality of persistence diagrams and the bottleneck and
  wasserstein distances.
\newblock \emph{Comput. Geom. Theory Appl.}, 105\penalty0 (C), August 2022.
\newblock ISSN 0925-7721.
\newblock \doi{10.1016/j.comgeo.2022.101882}.
\newblock URL \url{https://doi.org/10.1016/j.comgeo.2022.101882}.

\bibitem[Carriere et~al.(2021)Carriere, Chazal, Glisse, Ike, Kannan, and
  Umeda]{carriere21}
Mathieu Carriere, Frederic Chazal, Marc Glisse, Yuichi Ike, Hariprasad Kannan,
  and Yuhei Umeda.
\newblock Optimizing persistent homology based functions.
\newblock In \emph{Proceedings of the 38th International Conference on Machine
  Learning}, volume 139, pp.\  1294--1303. PMLR, 18--24 Jul 2021.
\newblock URL \url{https://proceedings.mlr.press/v139/carriere21a.html}.

\bibitem[Carrière et~al.(2026)Carrière, Ike, Lacombe, and
  Nishikawa]{topoOptSurvey26}
Mathieu Carrière, Yuichi Ike, Théo Lacombe, and Naoki Nishikawa.
\newblock Persistence-based topological optimization: a survey, 03 2026.

\bibitem[Che et~al.(2024)Che, Galaz-Garc{\'i}a, Guijarro, and
  Membrillo~Solis]{che24}
Mauricio Che, Fernando Galaz-Garc{\'i}a, Luis Guijarro, and Ingrid~Amaranta
  Membrillo~Solis.
\newblock Metric geometry of spaces of persistence diagrams.
\newblock \emph{Journal of Applied and Computational Topology}, 8\penalty0
  (8):\penalty0 2197--2246, Dec 2024.
\newblock ISSN 2367-1734.
\newblock \doi{10.1007/s41468-024-00189-2}.
\newblock URL \url{https://doi.org/10.1007/s41468-024-00189-2}.

\bibitem[Chen et~al.(2019)Chen, Ni, Bai, and Wang]{chen19}
Chao Chen, Xiuyan Ni, Qinxun Bai, and Yusu Wang.
\newblock A topological regularizer for classifiers via persistent homology.
\newblock In \emph{Proceedings of the 22nd International Conference on
  Artificial Intelligence and Statistics}, volume~89, pp.\  2573--2582. PMLR,
  16--18 Apr 2019.
\newblock URL \url{https://proceedings.mlr.press/v89/chen19g.html}.

\bibitem[Chowdhury et~al.(2023)Chowdhury, Monath, Dubey, Ahmed, and
  Chaturvedi]{kram}
Somnath Basu~Roy Chowdhury, Nicholas Monath, Avinava Dubey, Amr Ahmed, and
  Snigdha Chaturvedi.
\newblock Robust concept erasure via kernelized rate-distortion maximization.
\newblock In \emph{Proceedings of the 37th International Conference on Neural
  Information Processing Systems}, NIPS '23, Red Hook, NY, USA, 2023. Curran
  Associates Inc.

\bibitem[Clémot et~al.(2026)Clémot, Digne, and Tierny]{topoAE++26}
Mattéo Clémot, Julie Digne, and Julien Tierny.
\newblock Topological autoencoders++: Fast and accurate cycle-aware
  dimensionality reduction.
\newblock \emph{IEEE Transactions on Visualization and Computer Graphics},
  32\penalty0 (2):\penalty0 1622--1639, 2026.
\newblock \doi{10.1109/TVCG.2025.3644671}.

\bibitem[Cuturi(2013)]{cuturi2013Sinkhorn}
Marco Cuturi.
\newblock Sinkhorn distances: lightspeed computation of optimal transport.
\newblock In \emph{Proceedings of the 27th International Conference on Neural
  Information Processing Systems - Volume 2}, NIPS'13, pp.\  2292–2300, Red
  Hook, NY, USA, 2013. Curran Associates Inc.

\bibitem[De-Arteaga et~al.(2019)De-Arteaga, Romanov, Wallach, Chayes, Borgs,
  Chouldechova, Geyik, Kenthapadi, and Kalai]{bias_in_bios}
Maria De-Arteaga, Alexey Romanov, Hanna~M. Wallach, Jennifer~T. Chayes,
  Christian Borgs, Alexandra Chouldechova, Sahin~Cem Geyik, Krishnaram
  Kenthapadi, and Adam~Tauman Kalai.
\newblock Bias in bios: A case study of semantic representation bias in a
  high-stakes setting.
\newblock In danah boyd and Jamie~H. Morgenstern (eds.), \emph{FAT}, pp.\
  120--128. ACM, 2019.
\newblock URL
  \url{http://dblp.uni-trier.de/db/conf/fat/fat2019.html#De-ArteagaRWCBC19}.

\bibitem[Devlin et~al.(2019)Devlin, Chang, Lee, and Toutanova]{bert}
Jacob Devlin, Ming-Wei Chang, Kenton Lee, and Kristina Toutanova.
\newblock {BERT}: Pre-training of deep bidirectional transformers for language
  understanding.
\newblock In Jill Burstein, Christy Doran, and Thamar Solorio (eds.),
  \emph{Proceedings of the 2019 Conference of the North {A}merican Chapter of
  the Association for Computational Linguistics: Human Language Technologies,
  Volume 1 (Long and Short Papers)}, pp.\  4171--4186, Minneapolis, Minnesota,
  June 2019. Association for Computational Linguistics.
\newblock \doi{10.18653/v1/N19-1423}.
\newblock URL \url{https://aclanthology.org/N19-1423/}.

\bibitem[Edelsbrunner \& Harer(2009)Edelsbrunner and Harer]{edelsbrunner09}
H.~Edelsbrunner and J.~Harer.
\newblock \emph{Computational Topology: An Introduction}.
\newblock American Mathematical Society, 2009.

\bibitem[Edelsbrunner \& Mucke(1990)Edelsbrunner and Mucke]{edelsbrunner90}
Herbert Edelsbrunner and Ernst~P Mucke.
\newblock Simulation of simplicity: a technique to cope with degenerate cases
  in geometric algorithms.
\newblock \emph{ACM Transactions on Graphics}, 9\penalty0 (1):\penalty0
  66--104, 1990.
\newblock \doi{10.1145/77635.77639}.

\bibitem[Edelsbrunner et~al.(2002)Edelsbrunner, Letscher, and
  Zomorodian]{edelsbrunner02}
Herbert Edelsbrunner, David Letscher, and Afra Zomorodian.
\newblock {Topological Persistence and Simplification}.
\newblock 28\penalty0 (4):\penalty0 511--533, 2002.
\newblock \doi{10.1007/S00454-002-2885-2}.

\bibitem[Elazar \& Goldberg(2018)Elazar and Goldberg]{dial_elazar}
Yanai Elazar and Yoav Goldberg.
\newblock Adversarial removal of demographic attributes from text data.
\newblock In Ellen Riloff, David Chiang, Julia Hockenmaier, and Jun'ichi Tsujii
  (eds.), \emph{EMNLP}, pp.\  11--21. Association for Computational
  Linguistics, 2018.
\newblock ISBN 978-1-948087-84-1.
\newblock URL
  \url{http://dblp.uni-trier.de/db/conf/emnlp/emnlp2018.html#ElazarG18}.

\bibitem[Fel et~al.(2023{\natexlab{a}})Fel, Boutin, Moayeri, Cad{\`e}ne,
  Bethune, And{\'e}ol, Chalvidal, and Serre]{unifiedConceptExtraction}
Thomas Fel, Victor Boutin, Mazda Moayeri, R{\'e}mi Cad{\`e}ne, Louis Bethune,
  L{\'e}o And{\'e}ol, Mathieu Chalvidal, and Thomas Serre.
\newblock A holistic approach to unifying automatic concept extraction and
  concept importance estimation.
\newblock In \emph{Proceedings of the 37th International Conference on Neural
  Information Processing Systems}, NIPS '23, Red Hook, NY, USA,
  2023{\natexlab{a}}. Curran Associates Inc.

\bibitem[Fel et~al.(2023{\natexlab{b}})Fel, Picard, B\'ethune, Boissin,
  Vigouroux, Colin, Cad\`ene, and Serre]{craft}
Thomas Fel, Agustin Picard, Louis B\'ethune, Thibaut Boissin, David Vigouroux,
  Julien Colin, R\'emi Cad\`ene, and Thomas Serre.
\newblock Craft: Concept recursive activation factorization for explainability.
\newblock In \emph{Proceedings of the IEEE/CVF Conference on Computer Vision
  and Pattern Recognition (CVPR)}, pp.\  2711--2721, June 2023{\natexlab{b}}.

\bibitem[Ghorbani et~al.(2019)Ghorbani, Wexler, Zou, and Kim]{ace}
Amirata Ghorbani, James Wexler, James Zou, and Been Kim.
\newblock \emph{Towards automatic concept-based explanations}.
\newblock Curran Associates Inc., Red Hook, NY, USA, 2019.

\bibitem[Gupta et~al.(2023)Gupta, Saini, and Narayanan]{conceptDistillation}
Avani Gupta, Saurabh Saini, and P~J Narayanan.
\newblock Concept distillation: Leveraging human-centered explanations for
  model improvement.
\newblock In A.~Oh, T.~Naumann, A.~Globerson, K.~Saenko, M.~Hardt, and
  S.~Levine (eds.), \emph{Advances in Neural Information Processing Systems},
  volume~36, pp.\  63724--63737. Curran Associates, Inc., 2023.
\newblock \doi{10.52202/075280-2783}.
\newblock URL
  \url{https://proceedings.neurips.cc/paper_files/paper/2023/file/c9450295fd667740a39a68148fc17f6e-Paper-Conference.pdf}.

\bibitem[He et~al.(2016)He, Zhang, Ren, and Sun]{resnet}
Kaiming He, Xiangyu Zhang, Shaoqing Ren, and Jian Sun.
\newblock {Deep Residual Learning for Image Recognition}.
\newblock In \emph{Proceedings of 2016 IEEE Conference on Computer Vision and
  Pattern Recognition}, CVPR '16, pp.\  770--778. IEEE, June 2016.
\newblock \doi{10.1109/CVPR.2016.90}.
\newblock URL \url{http://ieeexplore.ieee.org/document/7780459}.

\bibitem[Hensel et~al.(2021)Hensel, Moor, and Rieck]{hensel2021survey}
Felix Hensel, Michael Moor, and Bastian Rieck.
\newblock A survey of topological machine learning methods.
\newblock \emph{Frontiers in Artificial Intelligence}, Volume 4 - 2021, 2021.
\newblock ISSN 2624-8212.
\newblock \doi{10.3389/frai.2021.681108}.
\newblock URL
  \url{https://www.frontiersin.org/journals/artificial-intelligence/articles/10.3389/frai.2021.681108}.

\bibitem[Hinton et~al.(2015)Hinton, Vinyals, and Dean]{kd}
Geoffrey~E. Hinton, Oriol Vinyals, and Jeffrey Dean.
\newblock Distilling the knowledge in a neural network.
\newblock \emph{CoRR}, abs/1503.02531, 2015.
\newblock URL
  \url{http://dblp.uni-trier.de/db/journals/corr/corr1503.html#HintonVD15}.

\bibitem[Hofer et~al.(2019)Hofer, Kwitt, Niethammer, and Dixit]{hofer19a}
Christoph Hofer, Roland Kwitt, Marc Niethammer, and Mandar Dixit.
\newblock Connectivity-optimized representation learning via persistent
  homology.
\newblock In \emph{Proceedings of the 36th International Conference on Machine
  Learning}, volume~97, pp.\  2751--2760. PMLR, 09--15 Jun 2019.
\newblock URL \url{https://proceedings.mlr.press/v97/hofer19a.html}.

\bibitem[Hofer et~al.(2020)Hofer, Graf, Niethammer, and Kwitt]{hofer20a}
Christoph Hofer, Florian Graf, Marc Niethammer, and Roland Kwitt.
\newblock Topologically densified distributions.
\newblock In \emph{Proceedings of the 37th International Conference on Machine
  Learning}, volume 119, pp.\  4304--4313. PMLR, 13--18 Jul 2020.
\newblock URL \url{https://proceedings.mlr.press/v119/hofer20a.html}.

\bibitem[Holstege et~al.(2024)Holstege, Wouters, Giersbergen, and Diks]{jse}
Floris Holstege, Bram Wouters, Noud~Van Giersbergen, and Cees Diks.
\newblock Removing spurious concepts from neural network representations via
  joint subspace estimation.
\newblock In Ruslan Salakhutdinov, Zico Kolter, Katherine Heller, Adrian
  Weller, Nuria Oliver, Jonathan Scarlett, and Felix Berkenkamp (eds.),
  \emph{Proceedings of the 41st International Conference on Machine Learning},
  volume 235 of \emph{Proceedings of Machine Learning Research}, pp.\
  18568--18610. PMLR, 21--27 Jul 2024.
\newblock URL \url{https://proceedings.mlr.press/v235/holstege24a.html}.

\bibitem[Holstege et~al.(2025)Holstege, Ravfogel, and Wouters]{splince}
Floris Holstege, Shauli Ravfogel, and Bram Wouters.
\newblock Preserving task-relevant information under linear concept removal.
\newblock In \emph{The Thirty-ninth Annual Conference on Neural Information
  Processing Systems}, 2025.
\newblock URL \url{https://openreview.net/forum?id=t79rMQGk4S}.

\bibitem[Huben et~al.(2024)Huben, Cunningham, Smith, Ewart, and
  Sharkey]{huben2024sae}
Robert Huben, Hoagy Cunningham, Logan~Riggs Smith, Aidan Ewart, and Lee
  Sharkey.
\newblock Sparse autoencoders find highly interpretable features in language
  models.
\newblock In \emph{The Twelfth International Conference on Learning
  Representations}, 2024.
\newblock URL \url{https://openreview.net/forum?id=F76bwRSLeK}.

\bibitem[Kaplan et~al.(2020)Kaplan, McCandlish, Henighan, Brown, Chess, Child,
  Gray, Radford, Wu, and Amodei]{llmScalingLaws}
Jared Kaplan, Sam McCandlish, Tom Henighan, Tom~B. Brown, Benjamin Chess, Rewon
  Child, Scott Gray, Alec Radford, Jeffrey Wu, and Dario Amodei.
\newblock Scaling laws for neural language models.
\newblock \emph{CoRR}, abs/2001.08361, 2020.
\newblock URL \url{https://arxiv.org/pdf/2001.08361.pdf}.

\bibitem[Kerber et~al.(2017)Kerber, Morozov, and Nigmetov]{Kerber2016}
Michael Kerber, Dmitriy Morozov, and Arnur Nigmetov.
\newblock Geometry helps to compare persistence diagrams.
\newblock \emph{ACM J. of Experimental Algorithmics}, 22:\penalty0 1--20, 2017.
\newblock \doi{10.1145/3064175}.

\bibitem[Kim et~al.(2018)Kim, Wattenberg, Gilmer, Cai, Wexler, Viegas, and
  sayres]{tcav}
Been Kim, Martin Wattenberg, Justin Gilmer, Carrie Cai, James Wexler, Fernanda
  Viegas, and Rory sayres.
\newblock Interpretability beyond feature attribution: Quantitative testing
  with concept activation vectors ({TCAV}).
\newblock In Jennifer Dy and Andreas Krause (eds.), \emph{Proceedings of the
  35th International Conference on Machine Learning}, volume~80 of
  \emph{Proceedings of Machine Learning Research}, pp.\  2668--2677. PMLR,
  10--15 Jul 2018.
\newblock URL \url{https://proceedings.mlr.press/v80/kim18d.html}.

\bibitem[Kim et~al.(2024)Kim, You, Lee, Kim, and Jung]{topkd}
Jungeun Kim, Junwon You, Dongjin Lee, Ha~Young Kim, and Jae-Hun Jung.
\newblock Do topological characteristics help in knowledge distillation?
\newblock In Ruslan Salakhutdinov, Zico Kolter, Katherine Heller, Adrian
  Weller, Nuria Oliver, Jonathan Scarlett, and Felix Berkenkamp (eds.),
  \emph{Proceedings of the 41st International Conference on Machine Learning},
  volume 235 of \emph{Proceedings of Machine Learning Research}, pp.\
  24674--24693. PMLR, 21--27 Jul 2024.
\newblock URL \url{https://proceedings.mlr.press/v235/kim24aj.html}.

\bibitem[Kissi et~al.(2025)Kissi, Pont, Levine, and Tierny]{kissiSolver25}
Mohamed Kissi, Mathieu Pont, Joshua~A. Levine, and Julien Tierny.
\newblock A practical solver for scalar data topological simplification.
\newblock \emph{IEEE Transactions on Visualization and Computer Graphics},
  31\penalty0 (1):\penalty0 97--107, 2025.
\newblock \doi{10.1109/TVCG.2024.3456345}.

\bibitem[Kornblith et~al.(2019)Kornblith, Norouzi, Lee, and Hinton]{cka}
Simon Kornblith, Mohammad Norouzi, Honglak Lee, and Geoffrey Hinton.
\newblock Similarity of neural network representations revisited.
\newblock In Kamalika Chaudhuri and Ruslan Salakhutdinov (eds.),
  \emph{Proceedings of the 36th International Conference on Machine Learning},
  volume~97 of \emph{Proceedings of Machine Learning Research}, pp.\
  3519--3529. PMLR, 09--15 Jun 2019.
\newblock URL \url{https://proceedings.mlr.press/v97/kornblith19a.html}.

\bibitem[Krizhevsky \& Hinton(2009)Krizhevsky and Hinton]{cifar100}
Alex Krizhevsky and Geoffrey Hinton.
\newblock Learning multiple layers of features from tiny images.
\newblock Technical Report~0, University of Toronto, Toronto, Ontario, 2009.
\newblock URL
  \url{https://www.cs.toronto.edu/~kriz/learning-features-2009-TR.pdf}.

\bibitem[Moor et~al.(2020)Moor, Horn, Rieck, and Borgwardt]{MoorHRB20}
Michael Moor, Max Horn, Bastian Rieck, and Karsten~M. Borgwardt.
\newblock Topological autoencoders.
\newblock In \emph{ICML}, 2020.

\bibitem[Munkres(1957)]{Munkres1957}
James Munkres.
\newblock Algorithms for the assignment and transportation problems.
\newblock \emph{J. of SIAM}, 5\penalty0 (1):\penalty0 32--38, 1957.
\newblock \doi{https://doi.org/10.1137/0105003}.

\bibitem[Naowarat et~al.(2026)Naowarat, Tang, and Goldwater]{naowarat26}
Burin Naowarat, Hao Tang, and Sharon Goldwater.
\newblock A framework for analyzing concept representations in neural models.
\newblock In Claire Bonial and Yevgeni Berzak (eds.), \emph{Proceedings of the
  30th Conference on Computational Natural Language Learning}, pp.\  574--587,
  San Diego, California, USA, July 2026. Association for Computational
  Linguistics.
\newblock ISBN 979-8-89176-410-1.
\newblock \doi{10.18653/v1/2026.conll-main.34}.
\newblock URL \url{https://aclanthology.org/2026.conll-main.34/}.

\bibitem[Nigmetov \& Morozov(2024)Nigmetov and Morozov]{nigmetov24}
Arnur Nigmetov and Dmitriy Morozov.
\newblock Topological optimization with big steps.
\newblock \emph{Discrete \& Computational Geometry}, 1 2024.
\newblock ISSN 0179-5376.
\newblock \doi{10.1007/s00454-023-00613-x}.
\newblock URL \url{https://www.osti.gov/biblio/2337698}.

\bibitem[Pont \& Garth(2026)Pont and Garth]{regionAware}
Mathieu Pont and Christoph Garth.
\newblock Region-aware wasserstein distances of persistence diagrams and merge
  trees.
\newblock \emph{IEEE Transactions on Visualization and Computer Graphics},
  32\penalty0 (9):\penalty0 7756--7773, 2026.
\newblock \doi{10.1109/TVCG.2026.3703539}.

\bibitem[Ravfogel et~al.(2020)Ravfogel, Elazar, Gonen, Twiton, and
  Goldberg]{inlp}
Shauli Ravfogel, Yanai Elazar, Hila Gonen, Michael Twiton, and Yoav Goldberg.
\newblock Null it out: Guarding protected attributes by iterative nullspace
  projection.
\newblock In Dan Jurafsky, Joyce Chai, Natalie Schluter, and Joel Tetreault
  (eds.), \emph{Proceedings of the 58th Annual Meeting of the Association for
  Computational Linguistics}, pp.\  7237--7256, Online, July 2020. Association
  for Computational Linguistics.
\newblock \doi{10.18653/v1/2020.acl-main.647}.
\newblock URL \url{https://aclanthology.org/2020.acl-main.647/}.

\bibitem[Ravfogel et~al.(2022)Ravfogel, Twiton, Goldberg, and Cotterell]{rlace}
Shauli Ravfogel, Michael Twiton, Yoav Goldberg, and Ryan~D Cotterell.
\newblock Linear adversarial concept erasure.
\newblock In Kamalika Chaudhuri, Stefanie Jegelka, Le~Song, Csaba Szepesvari,
  Gang Niu, and Sivan Sabato (eds.), \emph{Proceedings of the 39th
  International Conference on Machine Learning}, volume 162 of
  \emph{Proceedings of Machine Learning Research}, pp.\  18400--18421. PMLR,
  17--23 Jul 2022.
\newblock URL \url{https://proceedings.mlr.press/v162/ravfogel22a.html}.

\bibitem[Ravfogel et~al.(2023)Ravfogel, Goldberg, and
  Cotterell]{loglinear_guardedness}
Shauli Ravfogel, Yoav Goldberg, and Ryan Cotterell.
\newblock Log-linear guardedness and its implications.
\newblock In Anna Rogers, Jordan Boyd-Graber, and Naoaki Okazaki (eds.),
  \emph{Proceedings of the 61st Annual Meeting of the Association for
  Computational Linguistics (Volume 1: Long Papers)}, pp.\  9413--9431,
  Toronto, Canada, July 2023. Association for Computational Linguistics.
\newblock \doi{10.18653/v1/2023.acl-long.523}.
\newblock URL \url{https://aclanthology.org/2023.acl-long.523/}.

\bibitem[Romero et~al.(2015)Romero, Ballas, Kahou, Chassang, Gatta, and
  Bengio]{fitnet}
Adriana Romero, Nicolas Ballas, Samira~Ebrahimi Kahou, Antoine Chassang, Carlo
  Gatta, and Yoshua Bengio.
\newblock Fitnets: Hints for thin deep nets.
\newblock In Yoshua Bengio and Yann LeCun (eds.), \emph{3rd International
  Conference on Learning Representations, {ICLR} 2015, San Diego, CA, USA, May
  7-9, 2015, Conference Track Proceedings}, 2015.
\newblock URL \url{http://arxiv.org/abs/1412.6550}.

\bibitem[Sagawa et~al.(2020)Sagawa, Koh, Hashimoto, and Liang]{waterbirds}
Shiori Sagawa, Pang~Wei Koh, Tatsunori~B Hashimoto, and Percy Liang.
\newblock Distributionally robust neural networks for group shifts: On the
  importance of regularization for worst-case generalization.
\newblock In \emph{International Conference on Learning Representations
  (ICLR)}, 2020.

\bibitem[Saillenfest \& Lemberger(2025)Saillenfest and Lemberger]{leopard}
Antoine Saillenfest and Pirmin Lemberger.
\newblock Nonlinear concept erasure: {A} density matching approach.
\newblock In In{\^{e}}s Lynce, Nello Murano, Mauro Vallati, Serena Villata,
  Federico Chesani, Michela Milano, Andrea Omicini, and Mehdi Dastani (eds.),
  \emph{{ECAI} 2025 - 28th European Conference on Artificial Intelligence,
  25-30 October 2025, Bologna, Italy - Including 14th Conference on Prestigious
  Applications of Intelligent Systems {(PAIS} 2025)}, volume 413 of
  \emph{Frontiers in Artificial Intelligence and Applications}, pp.\
  2138--2145. {IOS} Press, 2025.
\newblock \doi{10.3233/FAIA251053}.
\newblock URL \url{https://doi.org/10.3233/FAIA251053}.

\bibitem[Soler et~al.(2018)Soler, Plainchault, Conche, and
  Tierny]{soler_ldav18}
Maxime Soler, Melanie Plainchault, Bruno Conche, and Julien Tierny.
\newblock { Lifted Wasserstein Matcher for Fast and Robust Topology Tracking }.
\newblock In \emph{2018 IEEE 8th Symposium on Large Data Analysis and
  Visualization (LDAV)}, pp.\  23--33, October 2018.
\newblock \doi{10.1109/LDAV.2018.8739196}.
\newblock URL
  \url{https://doi.ieeecomputersociety.org/10.1109/LDAV.2018.8739196}.

\bibitem[Tian et~al.(2020)Tian, Krishnan, and Isola]{contrastiveKD}
Yonglong Tian, Dilip Krishnan, and Phillip Isola.
\newblock Contrastive representation distillation.
\newblock In \emph{ICLR}. OpenReview.net, 2020.
\newblock URL
  \url{http://dblp.uni-trier.de/db/conf/iclr/iclr2020.html#TianKI20}.

\bibitem[Tung \& Mori(2019)Tung and Mori]{spkd}
Fred Tung and Greg Mori.
\newblock { Similarity-Preserving Knowledge Distillation }.
\newblock In \emph{2019 IEEE/CVF International Conference on Computer Vision
  (ICCV)}, pp.\  1365--1374, Los Alamitos, CA, USA, November 2019. IEEE
  Computer Society.
\newblock \doi{10.1109/ICCV.2019.00145}.
\newblock URL
  \url{https://doi.ieeecomputersociety.org/10.1109/ICCV.2019.00145}.

\bibitem[Vielhaben et~al.(2023)Vielhaben, Bluecher, and
  Strodthoff]{vielhaben2023multidimensional}
Johanna Vielhaben, Stefan Bluecher, and Nils Strodthoff.
\newblock Multi-dimensional concept discovery ({MCD}): A unifying framework
  with completeness guarantees.
\newblock \emph{Transactions on Machine Learning Research}, 2023.
\newblock ISSN 2835-8856.
\newblock URL \url{https://openreview.net/forum?id=KxBQPz7HKh}.

\bibitem[Zomorodian(2010)]{zomorodianBook}
Afra~J. Zomorodian.
\newblock Topology for computing.
\newblock In Mikhail~J. Atallah and Marina Blanton (eds.), \emph{Algorithms and
  Theory of Computation Handbook (Second Edition)}, chapter~3, pp.\  82--112.
  CRC Press, Boca Raton, 2010.
\newblock \doi{https://doi.org/10.1017/CBO9780511546945}.

\end{thebibliography}
\bibliographystyle{iclr2027_conference}

\newpage

\appendix

\section*{Appendices}

These appendices provide additional details including extended related work, method details, mathematical proofs, detailed experimental setup and additional experiments and results to further illustrate and validate the effectiveness of \acronym{}. The contents are organized as follows:

\begin{itemize}
   \item \autoref{sec_app_related_work}: Extended Related Work
   \item Additional Method Details
   \begin{itemize}
       \item \autoref{sec_app_mse}: MSE Formulation
       \item \autoref{sec_app_impl}: \acronym{} Implementation Details and Computational Cost
   \end{itemize}
   \item Mathematical Proofs
   \begin{itemize}
     \item \autoref{sec_app_concept_signatures} and \ref{sec_app_generalizing}: Existing Concept Erasure Criteria as Special Cases of \acronym{}
     \item \autoref{sec_app_invariance}: Affine Reparameterization Invariance of \acronym{}
     \item \autoref{app_ctr_stability}: Stability of \acronym{}
   \end{itemize}
   \item \autoref{sec_app_exp_setup}: Detailed Experimental Setup
   \begin{itemize}
     \item \subappref{sec_app_exp_setup_disent}: Setup for Concept Disentanglement Before Erasure
     \item \subappref{sec_app_exp_setup_disti}: Setup for Concept Transfer
   \end{itemize}
   \item \autoref{sec_app_exp}: Additional Experiments \& Results
   \begin{itemize}
      \item \subappref{sec_app_exp_disent}: Concept Disentanglement Before Erasure
      \begin{itemize}
          \item \subappref{sec_app_exp_disent_leak}: Beyond-linear Concept Leakage
          \item \subappref{sec_app_exp_disent_tradeoff}: Utility Preservation Across Erasure Levels and Data Regimes
          \item \subappref{sec_app_exp_disent_cka}: Representation Similarity After Intervention
      \end{itemize}
      % \item \subappref{sec_app_exp_transfer}: Concept Transfer
   \end{itemize}
   \item \subappref{sec_limitations}: Limitations
 \end{itemize}

\section{Extended Related Work}
\label{sec_app_related_work}

\myParagraph{Concept discovery.}
Concept-based interpretability aims to characterize learned representations through high-level semantic properties that are more meaningful to humans than individual activation dimensions \citep{tcav, ace}. 
% In a common setting, 
A concept is often represented as a direction in the activation space of a selected layer, for instance obtained from a linear classifier trained to distinguish examples associated with the concept \citep{tcav} or from multidimensional concept subspaces \citep{vielhaben2023multidimensional}. 
Concepts can also be discovered in an unsupervised manner, through clustering \citep{ace}, matrix factorizations such as NMF \citep{craft} or sparse autoencoders \citep{huben2024sae}. 
These approaches provide complementary mechanisms for extracting semantically meaningful properties encoded by neural representations.

\myParagraph{Concept-based interventions.}
Beyond their use for interpretation, concepts can also serve as targets for post-hoc interventions on learned representations.
A prominent line of work considers concept erasure, which aims to remove information associated with an undesired concept from a fixed representation through linear \citep{inlp, rlace, leace} or non-linear formulations \citep{kram, leopard}.
Other methods additionally seek to preserve information relevant to a downstream task \citep{jse, splince, obliviator}.
While effective for their respective objectives, these approaches are generally tailored to specific interventions, with new edits often requiring new objectives or optimization procedures.
More fundamentally, they do not provide a common representation in which concept organization can be described and manipulated independently of the editing algorithm.

%%%%%%%%%%%%%%%%%%%%%%%%%%%%%%%%%%%%%%%
%%%%%%%%%%%%%%%%%%%%%%%%%%%%%%%%%%%%%%%

\iffalse
\myParagraph{Topological abstractions.}
Topological Data Analysis (TDA) provides concise abstractions that capture the global organization of complex data while abstracting away local variations \citep{edelsbrunner09, hensel2021survey}. 
In particular, persistent homology tracks the evolution of topological features across a filtration, characterizing their appearance, disappearance, and interactions over multiple scales \citep{edelsbrunner02}. 
These features can be summarized through persistence diagrams, while richer abstractions such as merge trees additionally encode how connected components progressively merge. 
Such multiscale abstractions have been extensively used to characterize and compare complex data, including in conjunction with machine learning \citep{hensel2021survey}. 
Their ability to jointly represent individual features and their global organization makes them a natural candidate for describing structured relationships between concepts.
\fi

\myParagraph{Topological abstractions.}
Topological Data Analysis (TDA) provides concise abstractions of the global organization of the features present within complex data \citep{edelsbrunner09, hensel2021survey}.
In particular, such methods track how topological features appear and disappear across a filtration of the input, thereby characterizing structure over multiple scales \citep{edelsbrunner02}.
These features can be summarized compactly through persistence diagrams, while richer abstractions such as merge trees additionally encode how connected components progressively merge.
Their ability to jointly represent individual features and their global organization makes such abstractions a natural candidate for describing structured relationships between concepts.

%%%%%%%%%%%%%%%%%%%%%%%%%%%%%%%%%%%%%%%
%%%%%%%%%%%%%%%%%%%%%%%%%%%%%%%%%%%%%%%

\iffalse
\myParagraph{Persistence-based topological optimization.}
Beyond their descriptive role, topological abstractions can also be incorporated into differentiable optimization objectives.
For instance by enforcing topological constraints on the output of specific layers of neural networks.
Popular applications include regularization, to improve generalization on unseen data, by simplifying the topology of decision boundaries \citep{chen19, hofer20a}, or dimensionality reduction, where neural networks are optimized while preserving the topological properties of the input data in the representation being learned \citep{hofer19a, MoorHRB20, topoAE++26}.
\citet{carriere21}, based on the convergence properties of stochastic subgradient descent \citep{davis2020}, proposed a generic topological optimization framework supporting a wide range of losses and applications.
This framework has been later extended in various ways, e.g, to handle its sparse gradient problem or to make it more efficient on challenging datasets \citep{nigmetov24, kissiSolver25, topoOptSurvey26}.
However, \acronym{} does not simply add a regularization loss to neural networks optimization but provide an editable proxy of their concepts that users can manipulate.
\fi

\myParagraph{Persistence-based topological optimization.}
Beyond their descriptive role, topological abstractions can be incorporated into differentiable objectives to constrain learned representations. Such losses have been used, for instance, to regularize decision boundaries \citep{chen19, hofer20a} or preserve input topology during dimensionality reduction \citep{hofer19a, MoorHRB20, topoAE++26}. \citet{carriere21} introduced a generic framework for optimizing persistence-based objectives, later extended to improve gradient density and scalability \citep{nigmetov24, kissiSolver25, topoOptSurvey26}. \acronym{} differs from these approaches in that it is not used to regularize the representation itself, but to encode its concept organization and to provide a common way to specify desired changes to the concepts of the underlying representation.

\section{Analytical MSE Formulation}
\label{sec_app_mse}

This section details the closed-form MSE concept scores and predictive-overlap
interactions used as an efficient analytical construction of \acronym{}.

\myParagraph{Concept scores.}
Assume that the representation and concept scores are centered, and let
\mbox{$\Sigma_Z=\operatorname{Cov}(Z)$} and
\mbox{$\mathbf{g}_i=\operatorname{Cov}(Z,\mathbf{c}_i)$}.
For a linear least-squares probe with coefficient vector \mbox{$\boldsymbol{\beta}_i$},
the population MSE is:
\begin{equation}
\mathbb{E}\!\left[
\left(c_i-\boldsymbol{\beta}_i^\top Z\right)^2
\right]
=
\operatorname{Var}(\mathbf{c}_i)
-2\boldsymbol{\beta}_i^\top\mathbf{g}_i
+\boldsymbol{\beta}_i^\top\Sigma_Z\boldsymbol{\beta}_i.
\label{eq_app_mse_objective}
\end{equation}
Its minimum-norm solution is
\mbox{$\boldsymbol{\beta}_i^\star=\Sigma_Z^+\mathbf{g}_i$}, where
\mbox{$(\cdot)^+$} denotes the Moore--Penrose pseudoinverse. Substitution into
\autoref{eq_app_mse_objective} shows that the improvement over the optimal
constant predictor is:
\begin{equation}
    m_i
    =
    \mathbf{g}_i^\top \Sigma_Z^+ \mathbf{g}_i
    =
    A_{ii},
    \qquad
    A_{ij}
    :=
    \mathbf{g}_i^\top \Sigma_Z^+ \mathbf{g}_j.
    \label{eq_\acronym{}_predictive_gram}
\end{equation}
When the score is normalized by the constant-predictor loss, it is instead
\mbox{$A_{ii}/\operatorname{Var}(\mathbf{c}_i)$}, i.e., the explained-variance
fraction. Thus, the diagonal of the matrix $A$ contains all MSE
concept scores.

\myParagraph{Predictive overlap.}
As an efficient analytical alternative, we additionally consider \emph{predictive overlap} as concept interactions.
The off-diagonal entries of the same matrix quantify how strongly two concepts
rely on aligned predictive directions. We define their normalized predictive
overlap as:
\begin{equation}
\label{eq_\acronym{}_overlap}
\rho_{ij}^{\mathrm{overlap}}
=
\frac{A_{ij}^2}{A_{ii}A_{jj}}.
\end{equation}
To interpret this quantity, define the whitened predictive direction
\mbox{$\mathbf{v}_i=(\Sigma_Z^{1/2})^+\mathbf{g}_i$}. Then
\mbox{$A_{ij}=\mathbf{v}_i^\top\mathbf{v}_j$} and
\mbox{$A_{ii}=\|\mathbf{v}_i\|_2^2$}, so \autoref{eq_\acronym{}_overlap} becomes:
\begin{equation}
\rho_{ij}^{\mathrm{overlap}}
=
\frac{(\mathbf{v}_i^\top\mathbf{v}_j)^2}
{\|\mathbf{v}_i\|_2^2\|\mathbf{v}_j\|_2^2}.
\label{eq_app_overlap_cosine}
\end{equation}
It is therefore the squared cosine between the predictive directions of the two
concepts and belongs to $[0,1]$ whenever $A_{ii}$ and $A_{jj}$ are nonzero. A
value of zero indicates orthogonal predictive directions, whereas a value of one
indicates collinear directions, irrespective of their sign. Consequently, a
single matrix $A$ yields both the concept scores through its diagonal and every
pairwise predictive overlap through its normalized off-diagonal entries.

\section{Implementation Details and Computational Cost}
\label{sec_app_impl}

\myParagraph{Differentiable construction.}
When optimizing a representation, we reconstruct \acronym{} from the current transformed
features at each optimization step. For BCE recoverability, the linear concept probes are
therefore refitted differentiably, such that gradients propagate through the probe optimization.
For interference interactions, removing a source concept is performed with differentiable
LEACE and the probe of the target concept is subsequently refitted on the erased
representation. Hence, interference measures the recoverability remaining after erasure,
rather than the degradation of a fixed pre-erasure predictor. When differentiating a saddle
$s_{ij}=\min(m_i,m_j)\rho_{ij}$, we stop gradients through the
$\min(m_i,m_j)$ factor while retaining its value in the forward pass. This prevents
interaction optimization from trivially decreasing a saddle by decreasing one of its incident
concept scores instead of modifying their interaction.

\myParagraph{Analytical least-squares construction.}
The MSE formulation of \autoref{sec_app_mse} allows all concept scores to be computed
jointly. We first accumulate the sufficient statistics required to compute the representation
covariance and concept cross-covariances, namely sums of features, feature outer products
and feature--concept products. Given $n_C$ concepts, we collect the resulting
cross-covariances in $G=[\mathbf g_1,\ldots,\mathbf g_{n_C}]$ and perform a single
eigendecomposition of the $d\times d$ representation covariance $\Sigma_Z$.
The same decomposition yields all
least-squares predictive directions and concept scores simultaneously. When predictive
overlap is used, all pairwise interactions are subsequently obtained from their
$n_C\times n_C$ predictive Gram matrix. The covariance decomposition is performed in
double precision while numerically null eigendirections are discarded.

\myParagraph{Computational cost.}
Computing the sufficient statistics requires
$O(Nd^2+Ndn_C)$ operations, followed by one $O(d^3)$ covariance decomposition.
Computing only the concept scores then scales linearly with $n_C$ up to matrix products,
whereas forming all predictive-overlap interactions additionally costs $O(dn_C^2)$.
Thus, although a complete concept mesh contains $O(n_C^2)$ pairwise interactions,
the analytical construction does not require $O(n_C^2)$ matrix decompositions.
By contrast, interference explicitly evaluates both directional effects for each concept pair
and therefore requires $O(n_C^2)$ erased-recoverability evaluations. In our
disentanglement experiments this construction involves only the task and concept nodes,
while the larger concept sets used for transfer use the analytical least-squares construction.
We further evaluate computation of \acronym{} in \subappref{sec_app_exp_disent}.

\myParagraph{Online construction for representation transfer.}
When \acronym{} is optimized jointly with a network, constructing it from the complete training
set after every minibatch would be unnecessarily expensive. We instead maintain a bounded
feature bank and update the statistics required by the analytical formulation as new
features are produced. The resulting \acronym{} therefore uses a larger sample of the current
representation than the instantaneous minibatch while keeping its memory footprint bounded.
The precise bank size and \acronym{} instantiation used for concept transfer are reported in the
corresponding experimental setup.

\myParagraph{Choice of recoverability.}
BCE and MSE are different instantiations of concept recoverability rather than numerical
approximations of one another. BCE is natural for the binary supervised concepts considered
in our post-hoc erasure experiments and supports the intervention-based interference measure.
MSE instead measures explained variance under linear least squares, naturally accommodates
continuous concept coefficients and admits the shared analytical construction above.
Likewise, interference and predictive overlap provide complementary interaction measures:
the former measures the effect of explicitly erasing one concept on the recoverability of
another, whereas the latter measures the alignment of their whitened predictive directions.
We therefore use BCE with interference for concept disentanglement and normalized MSE
for concept transfer.

\section{Concept Signatures}
\label{sec_app_concept_signatures}

\iffalse
The persistence coordinates of a \acronym{} feature describe the score of a concept and the interaction level at which it merges with an older component, but they do not explicitly encode additional information about the identity or usage of that concept. In some settings, such information is available and can be useful when comparing two Topological Concept Representations. For instance, two features may correspond to concepts having similar sample-level annotations, or one may wish to distinguish concepts that have similar persistence coordinates but are encoded through different directions of the representation. We therefore optionally associate each \acronym{} feature with a concept signature and extend the Wasserstein ground metric to jointly compare persistence coordinates and signatures.
\fi

The persistence coordinates of a \acronym{} feature describe the score of a concept and the interaction level at which it merges with an older component, but they do not encode additional concept-specific information. We therefore optionally associate each feature with a concept signature and extend the Wasserstein ground metric to jointly compare persistence coordinates and signatures. This formalism is distinct from the hard semantic matching used in the main method (\autoref{sec_ctr_editing}) and is introduced primarily to express the objective-level specializations studied in \autoref{sec_app_generalizing}.

Related constructions have previously augmented persistence-diagram comparisons with information beyond birth and death. This has been done by incorporating spatial information associated with persistence pairs \citep{soler_ldav18}, properties of their associated regions \citep{regionAware} or, more generally, by allowing persistence-diagram points to live in general metric spaces \citep{bubenik22,che24}. Concept signatures follow a similar principle, but use auxiliary information describing the represented concepts rather than spatial or regional information.

%%%%%%%%%%%%%%%%%%%%%%%%%%%%%%%%%%%%%%%%%%%%%%%%%%%%%%%%%%%%%%%%%%%%%%%%%%%%%%%
%%%%%%%%%%%%%%%%%%%%%%%%%%%%%%%%%%%%%%%%%%%%%%%%%%%%%%%%%%%%%%%%%%%%%%%%%%%%%%%
%%%%%%%%%%%%%%%%%%%%%%%%%%%%%%%%%%%%%%%%%%%%%%%%%%%%%%%%%%%%%%%%%%%%%%%%%%%%%%%

\subsection{Signature-Augmented Ground Metric}

Recall that a \acronym{} feature associated with concept \mbox{$C_i$} is represented by the persistence point \mbox{$p_i=(s_i,m_i)\in\mathbb{R}^2$}.
We additionally associate it with a signature \mbox{$\boldsymbol{\sigma}_i\in\mathcal{S}$}, where \mbox{$(\mathcal{S},d_{\sigma})$} is a metric space describing the information carried by the signatures.
Given two signature-augmented persistence features \mbox{$(p_i,\boldsymbol{\sigma}_i)$} and \mbox{$(p_j,\boldsymbol{\sigma}_j)$}, we define their signature-augmented ground metric as:
\begin{equation}
\label{eq_signature_ground_metric}
\widetilde{\pointMetric}_{q,\lambda_{\sigma}}
\bigl(
(p_i,\boldsymbol{\sigma}_i),
(p_j,\boldsymbol{\sigma}_j)
\bigr)
=
\left(
|s_i-s_j|^q
+
|m_i-m_j|^q
+
\lambda_{\sigma}^q d_{\sigma}(\boldsymbol{\sigma}_i,\boldsymbol{\sigma}_j)^q
\right)^{1/q},
\end{equation}
where \mbox{$q\geq 1$} and \mbox{$\lambda_{\sigma}\geq 0$} controls the relative importance of the signature information.
Setting \mbox{$\lambda_{\sigma}=0$} recovers the standard ground metric used in \autoref{sec_preliminaries}.
When signatures are vectors in a common Euclidean space, a direct choice is \mbox{$\mathcal{S}=\mathbb{R}^r$} with \mbox{$d_{\sigma}(\boldsymbol{\sigma}_i,\boldsymbol{\sigma}_j)=\|\boldsymbol{\sigma}_i-\boldsymbol{\sigma}_j\|_q$}.
In this case, \autoref{eq_signature_ground_metric} is simply the \mbox{$\ell^q$} distance between the augmented vectors \mbox{$(s_i,m_i,\lambda_{\sigma}\boldsymbol{\sigma}_i)$} and \mbox{$(s_j,m_j,\lambda_{\sigma}\boldsymbol{\sigma}_j)$}.

\begin{proposition}
\label{prop_signature_ground_metric}
Let \mbox{$q\geq 1$}, \mbox{$\lambda_{\sigma}>0$} and let \mbox{$d_{\sigma}$} be a metric on \mbox{$\mathcal{S}$}.
Then \mbox{$\widetilde{\pointMetric}_{q,\lambda_{\sigma}}$} is a metric on \mbox{$\mathbb{R}^2\times\mathcal{S}$}.
\end{proposition}

\begin{proof}
Non-negativity, symmetry and identity of indiscernibles follow directly from the corresponding properties of the absolute value and \mbox{$d_{\sigma}$}.
For the triangle inequality, let \mbox{$u,v,w\in\mathbb{R}^2\times\mathcal{S}$}.
The two persistence-coordinate differences and the signature distance each satisfy their corresponding triangle inequalities.
Applying these inequalities component-wise, followed by Minkowski's inequality to the resulting three non-negative terms, gives \mbox{$\widetilde{\pointMetric}_{q,\lambda_{\sigma}}(u,w)\leq \widetilde{\pointMetric}_{q,\lambda_{\sigma}}(u,v)+\widetilde{\pointMetric}_{q,\lambda_{\sigma}}(v,w)$}.
Hence \mbox{$\widetilde{\pointMetric}_{q,\lambda_{\sigma}}$} is a metric.
\end{proof}

The proposition allows signatures to be more general than vectors, although vector-valued signatures are sufficient for the constructions considered here.
The corresponding signature-aware Wasserstein distance is obtained by replacing the ground metric \mbox{$\pointMetric_q$} by \mbox{$\widetilde{\pointMetric}_{q,\lambda_{\sigma}}$} in \autoref{eq_wasserstein_mapping}.
Thus, signatures modify the cost of matching two off-diagonal features but leave the assignment formulation itself unchanged.

\myParagraph{Matching to the diagonal.}
A signature is intended to characterize a persistence feature, not to prevent that feature from being removed or inserted through a diagonal match.
We therefore let the diagonal projection of a signature-augmented point inherit its signature.
For \mbox{$p_i=(s_i,m_i)$}, let \mbox{$\projection(p_i)=((s_i+m_i)/2,(s_i+m_i)/2)$}; its signature-augmented projection is \mbox{$(\projection(p_i),\boldsymbol{\sigma}_i)$}.
Consequently, matching a feature to its diagonal projection has no signature contribution and retains the usual persistence deletion cost:
\begin{equation}
\label{eq_signature_diagonal_cost}
\widetilde{\pointMetric}_{q,\lambda_{\sigma}}
\bigl(
(p_i,\boldsymbol{\sigma}_i),
(\projection(p_i),\boldsymbol{\sigma}_i)
\bigr)
=
\pointMetric_q\bigl(p_i,\projection(p_i)\bigr).
\end{equation}
As in the standard augmented-diagram formulation, the cost between two diagonal dummy points is set to zero.
Hence signatures only affect correspondences between actual persistence features, while feature insertion and deletion retain their standard topological interpretation.

\section{Existing Concept Erasure Criteria as Special Cases of \acronym{}}
\label{sec_app_generalizing}

The editing formulation of \acronym{} separates the criterion used to characterize a concept from the procedure used to modify the representation. In the main manuscript, this criterion is represented by the concept score \mbox{$m_i=S(Z,C_i)$} and an intervention is specified by a target value \mbox{$m_i^\star$}, possibly together with targets on the interaction structure. Concept signatures from \autoref{sec_app_concept_signatures} further allow the target to constrain information associated with a concept beyond its persistence coordinates. Together, these ingredients allow several existing concept-erasure criteria to be recovered as particular \acronym{} editing objectives. Our goal is not to reproduce their optimization algorithms, which may rely on closed-form projections, adversarial optimization, sequential subspace construction or eigendecompositions, but rather the objectives or constraints that they seek to enforce.

We group these connections according to the quantity used to characterize concept information: covariance and predictive recoverability for linear erasure, distributional and kernel dependence for nonlinear erasure and rate-distortion criteria. Most of these connections arise in a feature-wise specialization of \acronym{}, in which interaction coordinates are fixed and only the concept maxima and, when needed, their signatures are edited. This is sufficient to recover the corresponding erasure criteria, while the saddle structure of \acronym{} provides additional information that is not required by these special cases.

\subsection{Objective-Level Specialization of \acronym{}}

\begin{proposition}[Feature-wise specialization]
\label{prop_ctr_feature_specialization}
Consider \mbox{$n$} semantically matched signature-augmented persistence features of the current and target \acronym{} respectively given by \mbox{$(p_i(Z),\boldsymbol{\sigma}_i(Z))$} and \mbox{$(p_i^\star,\boldsymbol{\sigma}_i^\star)$}, with \mbox{$p_i(Z)=(s_i,m_i(Z))$} and \mbox{$p_i^\star=(s_i,m_i^\star)$}. Assume that the interaction coordinate \mbox{$s_i$} is identical in both features and does not contribute to the intervention. Under the semantic matching introduced in \autoref{sec_ctr_editing} and the signature-augmented ground metric of \autoref{eq_signature_ground_metric}, the \acronym{} objective reduces to:
\begin{equation}
\label{eq_ctr_specialization_reduction}
\mathcal{L}_{\acronym{}}(Z)
=
\sum_{i=1}^{n}
\left(
\left|m_i(Z)-m_i^\star\right|^q
+
\lambda_{\sigma}^q
 d_{\sigma}\!\left(\boldsymbol{\sigma}_i(Z),\boldsymbol{\sigma}_i^\star\right)^q
\right).
\end{equation}
\end{proposition}

\begin{proof}
Since \mbox{$(p_i(Z),\boldsymbol{\sigma}_i(Z))$} is constrained to match \mbox{$(p_i^\star,\boldsymbol{\sigma}_i^\star)$}, its contribution to the \mbox{$L^q$-Wasserstein} objective is the \mbox{$q$}-th power of their signature-augmented ground distance. The interaction coordinates being equal, the term \mbox{$|s_i-s_i|^q$} vanishes and the contribution reduces to \mbox{$|m_i(Z)-m_i^\star|^q+\lambda_{\sigma}^q d_{\sigma}(\boldsymbol{\sigma}_i(Z),\boldsymbol{\sigma}_i^\star)^q$}. Summing over the matched features gives \autoref{eq_ctr_specialization_reduction}.
\end{proof}

Setting \mbox{$\lambda_{\sigma}=0$} recovers the extrema-only specialization, for which \autoref{eq_ctr_specialization_reduction} reduces to \mbox{$\sum_i|m_i(Z)-m_i^\star|^q$}. Consequently, once a meaningful score and, when needed, a meaningful concept signature have been chosen, their target values directly specify the corresponding \acronym{} editing objective. In the following, we instantiate this construction with quantities having independent statistical interpretations, rather than defining features from arbitrary pre-existing losses. The resulting connections therefore arise from expressing covariance, predictive information, distributional discrepancy, statistical dependence and rate-distortion within the same editing interface.

The specializations considered below are summarized compactly as follows.
Each row specifies the \acronym{} scores and target values associated with an
existing erasure criterion:

\[
\resizebox{\linewidth}{!}{%
$\displaystyle
\renewcommand{\arraystretch}{1.5}
\left\{
\begin{array}{@{}l@{\qquad}l@{\qquad}l@{\qquad}c@{\;}l@{}}

m_c(P)=\|P\mathbf g_c\|_2,\;
\boldsymbol{\sigma}_c(P)=P\mathbf g_c
&
m_c^\star=0,
&
\boldsymbol{\sigma}_c^\star=\mathbf 0
&
\equiv
&
\text{LEACE}
\\

m_a(P)=\|P\mathbf g_a\|_2,\;
\boldsymbol{\sigma}_a(P)=P\mathbf g_a, \; a \in \{c, y\}
&
m_c^\star=0,
&
\boldsymbol{\sigma}_c^\star=\mathbf 0,
&
\equiv
&
\text{SPLINCE}
\\[-1mm]

&
m_y^\star=\|\mathbf g_y\|_2,
&
\boldsymbol{\sigma}_y^\star=\mathbf g_y
&
&
\\

m_c(P)=\mathcal L_c^0-
\min_{\psi}\mathcal L_{\mathrm{BCE}}\!\left(c,q_\psi(PH)\right)
&
m_c^\star=0
&
&
\equiv
&
\text{RLACE}
\\[-1mm]

&
&
&
\sim
&
\text{INLP}
\\

m_c(P)=
\left(
\sum_{a<b}
\widehat{\operatorname{MMD}}^{\,2}
\!\left(
\mathbb P_{z\mid c=a},
\mathbb P_{z\mid c=b}
\right)
\right)^{1/2}
&
m_c^\star=0
&
&
\equiv
&
\text{LEOPARD}
\\

m_S(Z)=\operatorname{HSIC}(Z,S),\;
m_A(Z)=\operatorname{HSIC}(Z,A)
&
m_S^\star=0,
&
&
\approx
&
\text{Obliviator}
\\[-1mm]

&
\multicolumn{2}{@{}l@{}}{
m_A^\star>m_A(H),\; A\in\{X,X_i,Y\}
}
&
&
\\

m_c(Z)=U-R(Z\mid K),\;
m_R(Z)=\lambda R(Z)
&
m_c^\star=0,
&
m_R^\star=\lambda R(H)
&
\equiv
&
\text{KRaM}

\end{array}
\right.
$%
}
\]

Here, $\equiv$ denotes an exact correspondence at the stated criterion or
objective level under the specified restrictions, $\approx$ denotes a correspondence
in the optimized quantities and direction of intervention without equality of the
complete objective, and $\sim$ denotes a structural or procedural correspondence.
These relations do not imply equivalence of the complete optimization algorithms
or of the transformations they select. The following subsections detail each case.

\subsection{Linear and Predictive Erasure}
\label{sec_app_generalizing_linear}

\myParagraph{Covariance-based erasure: LEACE and SPLINCE.}
LEACE \citep{leace} characterizes perfect linear erasure through the absence of cross-covariance between the edited representation and the concept, while selecting among valid erasers by minimizing representation distortion. SPLINCE \citep{splince} additionally requires the cross-covariance with a target task to remain unchanged. These constraints can be recovered from covariance-based concept features and signatures. Let \mbox{$h\in\mathbb{R}^{d}$} be a centered random representation, let \mbox{$c$} and \mbox{$y$} denote centered scalar concept and task variables and define \mbox{$\mathbf g_c=\operatorname{Cov}(h,c)$} and \mbox{$\mathbf g_y=\operatorname{Cov}(h,y)$}. For a linear edit \mbox{$z=Ph$}, their cross-covariances are:
\begin{equation}
\label{eq_ctr_leace_covariance_transform}
\operatorname{Cov}(z,c)=P\mathbf g_c,
\qquad
\operatorname{Cov}(z,y)=P\mathbf g_y.
\end{equation}
For this covariance-based specialization, we consider the \acronym{} concept score \(m_a(P)=\|P\mathbf g_a\|_2\) for both the concept \(a=c\) and the task \(a=y\). We associate each score with the covariance signature \(\boldsymbol{\sigma}_a(P)=P\mathbf g_a\) and use the Euclidean signature distance.

\begin{proposition}[LEACE and SPLINCE as covariance-based \acronym{} specializations]
\label{prop_ctr_leace_splince}
Consider the feature-wise specialization of \autoref{prop_ctr_feature_specialization} with the covariance scores and signatures defined above and \mbox{$\lambda_{\sigma}>0$}. For concept erasure, let the target concept feature satisfy \mbox{$m_c^\star=0$} and \mbox{$\boldsymbol{\sigma}_c^\star=\mathbf{0}$}. Then its \acronym{} contribution is zero if and only if the LEACE linear-guardedness constraint \mbox{$P\mathbf g_c=0$} holds. If a task feature is additionally assigned the targets \mbox{$m_y^\star=\|\mathbf g_y\|_2$} and \mbox{$\boldsymbol{\sigma}_y^\star=\mathbf g_y$}, the joint \acronym{} contribution is zero if and only if:
\begin{equation}
\label{eq_ctr_splince_constraints}
P\mathbf g_c=0,
\qquad
P\mathbf g_y=\mathbf g_y,
\end{equation}
which are the concept-erasure and task-covariance-preservation constraints of SPLINCE. The same result holds for vector-valued concepts or tasks by replacing the covariance vectors with the corresponding cross-covariance matrices and the Euclidean norm with the Frobenius norm.
\end{proposition}

\begin{proof}
For the concept feature, \mbox{$m_c(P)=0$} and \mbox{$\boldsymbol{\sigma}_c(P)=\mathbf{0}$} are both equivalent to \mbox{$P\mathbf g_c=\mathbf{0}$}. Since every term in \autoref{eq_ctr_specialization_reduction} is non-negative, its concept contribution is therefore zero if and only if the LEACE linear-erasure constraint holds.

For the task feature, the signature contribution vanishes if and only if \mbox{$P\mathbf g_y=\mathbf g_y$}. This equality also implies \mbox{$m_y(P)=\|P\mathbf g_y\|_2=\|\mathbf g_y\|_2=m_y^\star$}, so the persistence-coordinate contribution vanishes as well. Hence, the complete \acronym{} contribution is zero if and only if both conditions in Eq.~\eqref{eq_ctr_splince_constraints} hold. The matrix-valued case follows identically using the Frobenius norm.
\end{proof}

Preserving only the scalar 
%MSE recoverability 
covariance norm
of the task is in general weaker than the SPLINCE requirement \mbox{$P\mathbf g_y=\mathbf g_y$}.
% since two different covariance vectors may have the same norm.
The covariance signature makes this distinction explicit: preserving \mbox{$\boldsymbol{\sigma}_y$} imposes the covariance-vector equality itself, while its associated covariance score is then preserved automatically. Thus, concept signatures allow the complete SPLINCE constraint to be expressed without replacing vector preservation by a weaker scalar 
%recoverability condition. 
constraint.
LEACE and SPLINCE additionally select a particular transformation among those satisfying their respective constraints through a minimum-distortion criterion. The propositions above establish the erasure and preservation constraints represented by the \acronym{} editing objective; reproducing the selected transformation also requires imposing the same distortion criterion and transformation family.

%%%%%%%%%%%%%%%%%%%%%%%%%%%%%%%%%%%%%%%%%%%%%%%%%%%%%%%%%%%%%%%%%%%%%%%%%%%%%%%
%%%%%%%%%%%%%%%%%%%%%%%%%%%%%%%%%%%%%%%%%%%%%%%%%%%%%%%%%%%%%%%%%%%%%%%%%%%%%%%
%%%%%%%%%%%%%%%%%%%%%%%%%%%%%%%%%%%%%%%%%%%%%%%%%%%%%%%%%%%%%%%%%%%%%%%%%%%%%%%

\myParagraph{Predictive erasure: RLACE and INLP.}
RLACE \citep{rlace} formulates linear concept erasure as an adversarial game in which a linear predictor seeks to recover the concept while a rank-constrained linear transformation seeks to prevent this prediction. This criterion is directly related to the BCE recoverability score used by \acronym{}. Let \mbox{$c\in\{0,1\}$} be a binary concept and, for a linear edit \mbox{$z=Ph$}, let:
\begin{equation}
\label{eq_ctr_rlace_score}
m_c^{\mathrm{BCE}}(P)
=
\mathcal{L}_c^0
-
\min_{\psi}
\mathcal{L}_{\mathrm{BCE}}(c,q_{\psi}(Ph)),
\end{equation}
where \mbox{$q_{\psi}$} is a linear logistic probe and \mbox{$\mathcal{L}_c^0$} is the BCE of the optimal constant predictor, as in \autoref{eq_ctr_concept_score}. Assuming that the probe family contains this constant predictor, \mbox{$m_c^{\mathrm{BCE}}(P)\geq0$}.

\begin{proposition}[RLACE as a predictive-recoverability \acronym{} specialization]
\label{prop_ctr_rlace}
Let \mbox{$\mathcal{A}$} be a common feasible family of linear edits for \acronym{} and the adversarial erasure problem. Consider a feature-wise \acronym{} specialization with the score of \autoref{eq_ctr_rlace_score}, target \mbox{$m_c^\star=0$} and no signature contribution. For any \mbox{$q>0$}, minimizing the corresponding \acronym{} contribution over \mbox{$P\in\mathcal{A}$} is equivalent to maximizing the loss of the optimal linear logistic predictor:
\begin{equation}
\label{eq_ctr_rlace_equivalence}
\argmin_{P\in\mathcal{A}}
\left|m_c^{\mathrm{BCE}}(P)\right|^q
=
\argmax_{P\in\mathcal{A}}
\min_{\psi}
\mathcal{L}_{\mathrm{BCE}}(c,q_{\psi}(Ph)).
\end{equation}
Consequently, choosing the same relaxed rank-constrained transformation family as RLACE recovers its concept-erasure criterion at the objective level.
\end{proposition}

\begin{proof}
Since \mbox{$m_c^{\mathrm{BCE}}(P)\geq0$} and \mbox{$q>0$}, minimizing \mbox{$|m_c^{\mathrm{BCE}}(P)|^q$} is equivalent to minimizing \mbox{$m_c^{\mathrm{BCE}}(P)$}. The null loss \mbox{$\mathcal{L}_c^0$} does not depend on \mbox{$P$}, hence from \autoref{eq_ctr_rlace_score} this is equivalent to maximizing \mbox{$\min_{\psi}\mathcal{L}_{\mathrm{BCE}}(c,q_{\psi}(Ph))$}, which gives \autoref{eq_ctr_rlace_equivalence}.
\end{proof}

The connection also appears directly in the optimization procedure. At an \acronym{} iteration, the probe is fitted on the current representation, approximately minimizing \mbox{$\mathcal{L}_{\mathrm{BCE}}(c,q_{\psi}(Ph))$}. The subsequent gradient update of \mbox{$P$} toward \mbox{$m_c^\star=0$} reduces the recoverability in \autoref{eq_ctr_rlace_score}, and therefore increases the loss attained by this fitted predictor. After updating the representation, rebuilding the \acronym{} refits the probe and provides a new best response to the current transformation. Thus, the repeated probe-fitting and representation-update steps can be viewed as the two sides of the same predictor--eraser adversarial game as RLACE. The two methods differ in how this game is optimized and in the constraints imposed on the transformation.
This correspondence concerns a restricted specialization of \acronym{}. RLACE restricts the intervention to a fixed-rank orthogonal projection and considers a single adversarial predictive-erasure objective. In contrast, \acronym{} can use different concept scores and transformation families, jointly represent several concept and task features and additionally optimize their interactions through the target topology. The RLACE objective is therefore recovered when restricting \acronym{} to a single BCE-based concept feature with a zero target, no interaction constraints and the same projection family; these restrictions are not inherent to the general \acronym{} formulation.

INLP \citep{inlp} targets the same general notion of linear guardedness through a different sequential construction. At each iteration, it fits a linear concept classifier on the current representation, projects the representation onto the classifier nullspace and refits a new classifier on the projected representation. Each projection therefore exactly removes the direction used by the current classifier, while the repetition searches for additional predictive directions that remain after the previous ones have been removed. In contrast, BCE-based \acronym{} uses the recoverability score of \autoref{eq_ctr_rlace_score} as its editing criterion and replaces each hard nullspace projection by a gradient update that reduces the current concept recoverability before refitting the probe. If guardedness is evaluated with the same linear probe family and BCE criterion, both procedures target the condition \mbox{$m_c^{\mathrm{BCE}}=0$}. INLP can therefore be viewed as a hard sequential counterpart of this predictive-erasure specialization of \acronym{}, rather than as a distinct scalar objective reproduced by \acronym{}. The original INLP formulation uses classification accuracy to operationalize guardedness, so this correspondence concerns the shared linear-predictive constraint rather than 
%equality of 
their optimization procedures.

\subsection{Distributional and Nonlinear Dependence}
\label{sec_app_generalizing_nonlinear}

The previous specializations characterize concept information through covariance or linear predictability. Nonlinear erasure methods instead consider stronger distributional or statistical dependencies between the edited representation and the concept. We next show that such criteria can similarly be expressed through \acronym{} concept scores.

\myParagraph{Distribution matching: LEOPARD.}
LEOPARD \citep{leopard} seeks a rank-constrained orthogonal projection for which the class-conditional distributions of a discrete concept become indistinguishable. Let \mbox{$c\in\{1,\ldots,K\}$} and \mbox{$z=Ph$}. We define the following distributional \acronym{} concept score:
\begin{equation}
\label{eq_ctr_leopard_score}
m_c^{\mathrm{MMD}}(P)
=
\left(
\sum_{a<b}
\widehat{\operatorname{MMD}}^2
\left(
\mathbb{P}_{z\mid c=a},
\mathbb{P}_{z\mid c=b}
\right)
\right)^{1/2}.
\end{equation}
For a binary concept, this is the empirical MMD between its two class-conditional distributions.

\begin{proposition}[LEOPARD as a distribution-matching \acronym{} specialization]
\label{prop_ctr_leopard}
Consider the feature-wise specialization of \autoref{prop_ctr_feature_specialization} with the score of \autoref{eq_ctr_leopard_score}, target \mbox{$m_c^\star=0$}, no signature contribution and \mbox{$q=2$}. Then the corresponding \acronym{} contribution is exactly the LEOPARD erasure loss:
\begin{equation}
\label{eq_ctr_leopard_equivalence}
\left|m_c^{\mathrm{MMD}}(P)\right|^2
=
\sum_{a<b}
\widehat{\operatorname{MMD}}^2
\left(
\mathbb{P}_{z\mid c=a},
\mathbb{P}_{z\mid c=b}
\right).
\end{equation}
Restricting \mbox{$P=UU^\top$} to the same rank-\mbox{$r$} parameterization and adding the LEOPARD projection penalty \mbox{$\gamma\|U^\top U-I_r\|_F^2$} therefore recovers its complete training objective.
\end{proposition}

\begin{proof}
With \mbox{$m_c^\star=0$}, no signature contribution and \mbox{$q=2$}, \autoref{eq_ctr_specialization_reduction} gives the contribution \mbox{$|m_c^{\mathrm{MMD}}(P)|^2$}. Substituting \autoref{eq_ctr_leopard_score} yields \autoref{eq_ctr_leopard_equivalence}. The remaining LEOPARD term depends only on the transformation parameterization and is recovered by imposing its orthogonality penalty.
\end{proof}

Thus, LEOPARD corresponds to a distributional choice of the \acronym{} concept score: instead of measuring whether the concept is linearly predictable, it measures whether its class-conditional representation distributions can still be distinguished. With a characteristic kernel, a zero population MMD implies equality of these distributions. The adjustable projection rank and the orthogonalization used by LEOPARD concern how this criterion is optimized, not the criterion itself.

\myParagraph{Kernel dependence: Obliviator.}
Obliviator \citep{obliviator} extends the erasure criterion from linear predictability to nonlinear statistical dependence. Let \(S\) denote the unwanted concept, \(Y\) the task, \(X\) the original representation and \(X_i\) the representation provided as input to the encoder at iteration \(i\). For its output \(Z_i\), Obliviator minimizes an empirical RKHS objective that, up to the common HSIC normalization, can be written as:
\begin{equation}
\label{eq_ctr_obliviator_objective}
\operatorname{HSIC}(Z_i,S)
-\tau_x\operatorname{HSIC}(Z_i,X)
-\tau_{x_i}\operatorname{HSIC}(Z_i,X_i)
-\tau_y\operatorname{HSIC}(Z_i,Y),
\end{equation}
where the task term is omitted in the unsupervised setting. The first term directly defines a nonlinear \acronym{} concept score \(m_S(Z)=\operatorname{HSIC}(Z,S)\). Targeting \(m_S^\star=0\) therefore recovers Obliviator's erasure criterion: with characteristic kernels, a vanishing population HSIC corresponds to statistical independence between the edited representation and the unwanted concept.

The remaining terms follow the same dependence measure but encourage information about \(X\), \(X_i\) and \(Y\) to remain observable in the edited representation. They can consequently be represented in \acronym{} as additional features with scores \(m_A(Z)=\operatorname{HSIC}(Z,A)\), for \(A\in\{X,X_i,Y\}\), whose target values specify the desired level of preservation. This recovers the same erasure--utility quantities used by Obliviator, although its encoder objective maximizes the utility HSIC terms directly rather than matching them to fixed target values.

The correspondence concerns the criterion encoded by these dependence quantities rather than the complete Obliviator procedure. Obliviator alternates encoder optimization with an additional RKHS disentanglement step obtained from a constrained eigenvalue problem, which exposes utility-related information while remaining insensitive to \(S\) before training the next encoder. In contrast, \acronym{} expresses the desired decrease and preservation of these concept scores through its target representation and optimizes the chosen transformation accordingly.

\subsection{Rate-Distortion Erasure}
\label{sec_app_generalizing_rate}

\myParagraph{Kernelized rate-distortion: KRaM.}
KRaM \citep{kram} formulates concept erasure through a kernelized rate-distortion criterion.
Let \mbox{$Z=T_\theta(H)\in\mathbb{R}^{N\times d}$} denote the edited representation and let \mbox{$K\in\mathbb{R}^{N\times N}$} encode similarities between the concept labels of the samples.
KRaM considers the standard rate-distortion \mbox{$R(Z)$} and its kernelized counterpart:
\begin{equation}
\label{eq_ctr_kram_rate}
R(Z\mid K)
=
\frac{1}{2}\log_2\det\left(
I+\frac{d}{N\epsilon^2}(ZZ^\top\odot K)
\right).
\end{equation}
Maximizing \mbox{$R(Z\mid K)$} encourages samples that are similar according to the concept kernel to become dissimilar in the edited representation.
Since this can simultaneously expand the overall representation space, KRaM additionally constrains its standard coding rate around the original value \mbox{$b=R(H)$}, leading to the relaxed objective:
\begin{equation}
\label{eq_ctr_kram_objective}
\max_{\theta}\;
R(Z\mid K)
-
\lambda\left|R(Z)-b\right|.
\end{equation}

Both terms can be expressed through \acronym{} features.
Under the normalization assumed by KRaM, let \mbox{$U$} denote the constant upper bound of \mbox{$R(Z\mid K)$}.
We define a concept feature and a rate-preservation feature with scores and targets:
\begin{equation}
\label{eq_ctr_kram_scores}
m_c(Z)=U-R(Z\mid K),
\qquad
m_c^\star=0,
\qquad
m_R(Z)=\lambda R(Z),
\qquad
m_R^\star=\lambda b.
\end{equation}
%
\iffalse
The first feature measures the remaining gap to maximal kernelized rate-distortion and therefore vanishes as the concept-erasure criterion of KRaM is optimized.
The second preserves the overall coding rate of the original representation.
With \mbox{$q=1$} and no signature contribution, their joint \acronym{} contribution is:
\fi
The first feature measures the gap to maximal kernelized rate-distortion and therefore vanishes as the concept-erasure criterion of KRaM is optimized. The second preserves the coding rate of the original representation. With \mbox{$q=1$} and no signature contribution, their joint \acronym{} contribution is:
\begin{align}
\mathcal{L}_{\acronym{}}(Z)
&=
\left|m_c(Z)-m_c^\star\right|
+
\left|m_R(Z)-m_R^\star\right| \nonumber\\
&=
U-R(Z\mid K)
+
\lambda\left|R(Z)-b\right|.
\label{eq_ctr_kram_equivalence}
\end{align}
Since \mbox{$U$} is independent of the transformation, minimizing \autoref{eq_ctr_kram_equivalence} is equivalent to maximizing the KRaM objective of \autoref{eq_ctr_kram_objective}.
KRaM can therefore be expressed as a two-feature rate-distortion specialization of \acronym{}, with one feature driven toward concept erasure and the other preserving the global representation rate.
The correspondence concerns the optimized criterion: KRaM additionally normalizes the transformed representations because the rate-distortion quantities are scale-sensitive and uses a nonlinear erasure function, whereas \acronym{} is not restricted to this transformation or optimization procedure.

\section{Affine Reparameterization Invariance of \acronym{}}
\label{sec_app_invariance}

We show that the analytical MSE--overlap instantiation of \acronym{} is invariant to any invertible affine change of coordinates of the representation. This includes, for example, permutations, rotations and nonsingular rescalings of the representation dimensions.

\begin{proof}[Proof of \autoref{prop_ctr_invariance}]
Let \mbox{$Z_c$} and \mbox{$\widetilde Z_c$} denote the centered representations. Centering removes the translation \mbox{$a$}, so \mbox{$\widetilde Z_c=Z_cB^\top$}. Since \mbox{$B$} is invertible, right multiplication by \mbox{$B^\top$} does not change the column space and therefore \mbox{$\operatorname{Col}(\widetilde Z_c)=\operatorname{Col}(Z_c)$}. Let \mbox{$\Pi_Z$} denote the orthogonal projector onto this common column space.
For concept \mbox{$C_i$}, let \mbox{$\bar{\mathbf c}_i$} be its centered annotation vector. Since the annotations are centered, the optimal constant predictor is zero, with mean-squared error:

$$
    \mathrm{MSE}_{\mathrm{null}}
    =
    \frac{1}{N}
    \left\|\bar{\mathbf c}_i\right\|_2^2.
$$

The predictions of a linear probe belong to \mbox{$\operatorname{Col}(Z_c)$}. Hence, the fitted values of the optimal least-squares probe are given by the orthogonal projection \mbox{$\widehat{\mathbf c}_i=\Pi_Z\bar{\mathbf c}_i$}, and its mean-squared error is:

$$
    \mathrm{MSE}_{\mathrm{probe}}
    =
    \frac{1}{N}
    \left\|
        \bar{\mathbf c}_i-\Pi_Z\bar{\mathbf c}_i
    \right\|_2^2.
$$

By definition of an orthogonal projection, \mbox{$\Pi_Z\bar{\mathbf c}_i$} and the residual \mbox{$\bar{\mathbf c}_i-\Pi_Z\bar{\mathbf c}_i$} are orthogonal. Let \mbox{$p:=\Pi_Z\bar{\mathbf c}_i$} and \mbox{$r:=\bar{\mathbf c}_i-\Pi_Z\bar{\mathbf c}_i$}, so that \mbox{$\bar{\mathbf c}_i=p+r$}. Since \mbox{$p$} and \mbox{$r$} are orthogonal, \mbox{$p^\top r=0$}. Therefore, \mbox{$\|\bar{\mathbf c}_i\|_2^2=\|p+r\|_2^2=(p+r)^\top(p+r)=p^\top p+2p^\top r+r^\top r=\|p\|_2^2+\|r\|_2^2$}. Hence:

$$
    \left\|\bar{\mathbf c}_i\right\|_2^2
    =
    \left\|\Pi_Z\bar{\mathbf c}_i\right\|_2^2
    +
    \left\|
        \bar{\mathbf c}_i-\Pi_Z\bar{\mathbf c}_i
    \right\|_2^2.
$$

The MSE concept score is the gain of the best linear probe over the best constant predictor. Thus:

$$
    m_i
    =
    \mathrm{MSE}_{\mathrm{null}}
    -
    \mathrm{MSE}_{\mathrm{probe}} 
    =
    \frac{1}{N}
    \left\|\Pi_Z\bar{\mathbf c}_i\right\|_2^2.
$$

Since \mbox{$\Pi_Z$} is the same for \mbox{$Z$} and \mbox{$\widetilde Z$}, 
%the concept score 
\mbox{$m_i$} is identical for both representations \mbox{$Z$} and \mbox{$\widetilde Z$} for every concept.
Likewise, the predictive Gram matrix can be expressed using the same projected annotations. Using:

$$
    \Sigma_Z
    =
    \frac{1}{N}Z_c^\top Z_c,
    \qquad
    g_i
    =
    \frac{1}{N}Z_c^\top\bar{\mathbf c}_i,
$$

and the original definition \mbox{$A_{ij}=g_i^\top\Sigma_Z^+g_j$}, we obtain:

$$
\begin{aligned}
    A_{ij}
    &=
    \left(
        \frac{1}{N}Z_c^\top\bar{\mathbf c}_i
    \right)^\top
    \left(
        \frac{1}{N}Z_c^\top Z_c
    \right)^+
    \left(
        \frac{1}{N}Z_c^\top\bar{\mathbf c}_j
    \right) \\
    &=
    \frac{1}{N}
    \bar{\mathbf c}_i^\top
    Z_c
    \left(Z_c^\top Z_c\right)^+
    Z_c^\top
    \bar{\mathbf c}_j.
\end{aligned}
$$

The matrix \mbox{$\Pi_Z=Z_c(Z_c^\top Z_c)^+Z_c^\top$} is the orthogonal projector onto \mbox{$\operatorname{Col}(Z_c)$}. Hence:

$$
    A_{ij}
    =
    \frac{1}{N}
    \bar{\mathbf c}_i^\top
    \Pi_Z
    \bar{\mathbf c}_j.
$$

We have \mbox{$\bar{\mathbf c}_i^\top\Pi_Z\bar{\mathbf c}_j=\bar{\mathbf c}_i^\top\Pi_Z^\top\Pi_Z\bar{\mathbf c}_j=(\Pi_Z\bar{\mathbf c}_i)^\top(\Pi_Z\bar{\mathbf c}_j)=\left\langle\Pi_Z\bar{\mathbf c}_i,\Pi_Z\bar{\mathbf c}_j\right\rangle$} since \mbox{$\Pi_Z$} is an orthogonal projector making \mbox{$\Pi_Z^\top=\Pi_Z$} and \mbox{$\Pi_Z^2=\Pi_Z$}. Hence:

$$
    A_{ij}
    =
    \frac{1}{N}
    \left\langle
        \Pi_Z\bar{\mathbf c}_i,
        \Pi_Z\bar{\mathbf c}_j
    \right\rangle.
$$

Thus every \mbox{$A_{ij}$} is unchanged. Consequently, whenever the overlap interaction is defined:
\[
    \rho_{ij}^{\mathrm{overlap}}
    =
    \frac{A_{ij}^2}{A_{ii}A_{jj}},
\]
it is also unchanged. Since the saddle value is \mbox{$s_{ij}=\min(m_i,m_j)\rho_{ij}^{\mathrm{overlap}}$}, all concept and saddle values of the scalar field are identical for \mbox{$Z$} and \mbox{$\widetilde Z$}.

The two representations therefore induce the same scalar field on the same concept mesh and hence the same filtration, persistence pairs and persistence diagram. This proves \mbox{$\acronym{}(\widetilde Z)=\acronym{}(Z)$}. If the scores are normalized by \mbox{$\operatorname{Var}(\mathbf c_i)$}, the conclusion is unchanged because this quantity depends only on the concept annotations.
\end{proof}

\begin{remark}[Scope of the result]
The proposition concerns the MSE--overlap instantiation of \acronym{}, for which the construction is determined by least-squares predictions. It does not require \mbox{$Z$} or its covariance matrix to have full rank. It establishes invariance to invertible affine changes of representation coordinates, not to arbitrary non-invertible or nonlinear transformations.
\end{remark}

\section{Stability of \acronym{}}
\label{app_ctr_stability}

We show that, away from degenerate cases, small perturbations of a representation
induce small perturbations of its \acronym{}. We consider the normalized MSE
formulation with predictive-overlap interactions from
\autoref{sec_ctr_construction}. The result is local: this is sufficient to state
that \acronym{} varies continuously and at most linearly under sufficiently small
representation perturbations.

\subsection{Representation-to-\acronym{} Stability}

Let \mbox{$\representation\in\mathbb{R}^{\datasetSize\times d}$} be a centered
representation of \mbox{$\datasetSize$} samples, and let
\mbox{$\mathbf c_1,\ldots,\mathbf c_{n_C}\in\mathbb{R}^{\datasetSize}$} be fixed,
centered, non-constant concept annotation vectors. Set \mbox{$n:=\datasetSize-1$}.
For any centered representation \mbox{$Z\in\mathbb{R}^{\datasetSize\times d}$},
define the empirical covariance and concept cross-covariances as:
\[
\Sigma(Z)
=
\frac{1}{n}Z^\top Z,
\qquad
\mathbf g_i(Z)
=
\frac{1}{n}Z^\top\mathbf c_i.
\]
Whenever \mbox{$\Sigma(Z)$} is invertible, define the predictive Gram entries as:
\[
A_{ij}(Z)
=
\mathbf g_i(Z)^\top\Sigma(Z)^{-1}\mathbf g_j(Z).
\]
For each concept, define its variance as:
\[
\sigma_i^2
=
\frac{1}{n}\|\mathbf c_i\|_2^2
>0.
\]
The normalized MSE concept scores, predictive-overlap interactions and
corresponding saddle values are respectively defined as:
\[
m_i(Z)
=
\frac{A_{ii}(Z)}{\sigma_i^2},
\qquad
\rho_{ij}(Z)
=
\frac{A_{ij}(Z)^2}{A_{ii}(Z)A_{jj}(Z)},
\qquad
s_{ij}(Z)
=
\min(m_i(Z),m_j(Z))\rho_{ij}(Z).
\]
We denote by \mbox{$f_Z$} the resulting scalar field on the fixed concept mesh
\mbox{$\mathcal M$}, by \mbox{$\acronym{}(Z)$} its persistence diagram and by \mbox{$d_B$} the Bottleneck distance, being the \mbox{$q\to\infty$} limit of the \mbox{$L^q$-Wasserstein} distance, measuring the largest cost of a matched pair instead of accumulating all matching costs.

\begin{proof}[Proof of \autoref{thm_ctr_representation_stability}]
% The proof follows directly from the fact that the scalar values defining the concept mesh are locally Lipschitz functions of the representation, followed by stability of persistence diagrams.
The proof follows from the fact that the scalar values defining the concept mesh are locally Lipschitz functions of the representation, followed by stability of persistence diagrams.

First, the covariance and cross-covariance maps defined above are polynomial and
linear in \mbox{$Z$}, respectively, and are therefore smooth. Since
\mbox{$\Sigma(\representation)$} is positive definite, continuity implies that there
exists \mbox{$\varepsilon_1>0$} such that \mbox{$\Sigma(Z)$} remains positive definite whenever
\mbox{$\|Z-\representation\|_F<\varepsilon_1$}. The matrix inverse is smooth on the set
of invertible matrices. Hence every \mbox{$A_{ij}(Z)$} is smooth, and therefore locally
Lipschitz, around \mbox{$\representation$}.

Because each \mbox{$\sigma_i^2$} is a fixed positive constant, every maximum \mbox{$m_i(Z)$}
is also locally Lipschitz. Moreover, \mbox{$m_i(\representation)>0$} implies
\mbox{$A_{ii}(\representation)>0$}. Since there are finitely many concepts and every
\mbox{$A_{ii}(Z)$} is continuous, there exists \mbox{$\varepsilon_2>0$} such that all
\mbox{$A_{ii}(Z)$} remain strictly positive whenever
\mbox{$\|Z-\representation\|_F<\varepsilon_2$}. We may therefore choose
\mbox{$\varepsilon\le\min(\varepsilon_1,\varepsilon_2)$}, so that both the covariance
inverse and every denominator appearing in \mbox{$\rho_{ij}$} are well-defined
throughout the ball. Consequently, every \mbox{$\rho_{ij}(Z)$} is locally Lipschitz as
well.

Finally, the map \mbox{$(a,b)\mapsto\min(a,b)$} is Lipschitz, so every \mbox{$s_{ij}(Z)$} is
locally Lipschitz. Since the concept mesh contains finitely many maxima and
saddles, their local Lipschitz constants can be bounded by a single constant
\mbox{$C>0$} on a possibly smaller ball of radius \mbox{$\varepsilon$}. Hence:
\begin{equation}
\label{eq_ctr_scalar_field_stability}
\|f_{\representation}-f_{\representation'}\|_\infty
\le
C\,\|\representation-\representation'\|_F
\qquad
\text{whenever }\|\representation-\representation'\|_F<\varepsilon.
\end{equation}

The two scalar fields are defined on the same concept mesh. Although
\acronym{} uses an upper-star filtration, applying the usual persistence
stability theorem to the negated scalar fields \citep{edelsbrunner09} gives:
\[
d_B\!\left(
\acronym{}(\representation),
\acronym{}(\representation')
\right)
\le
\|f_{\representation}-f_{\representation'}\|_\infty.
\]
Combining this inequality with \autoref{eq_ctr_scalar_field_stability} proves
\autoref{thm_ctr_representation_stability}.
\end{proof}

\begin{remark}[Scope of the result]
% The radius \mbox{$\varepsilon$} is not an additional model parameter. It only formalizes ``sufficiently small perturbations'': it is chosen small enough that the covariance remains invertible and every concept recoverability remains positive throughout the ball. The assumptions deliberately exclude degenerate points at which the covariance becomes singular or a concept recoverability vanishes. 
The result is local and deliberately excludes degenerate representations. The radius $\varepsilon$ is chosen small enough such that the covariance remains invertible and all concept recoverabilities remain strictly positive whenever $\mbox{$|\representation-\representation'|_F<\varepsilon$}$. Within this neighborhood, all quantities defining the \acronym{} scalar field are smooth or Lipschitz functions of the representation.
\iffalse
A more quantitative version could replace this neighborhood argument by explicit lower bounds on the smallest eigenvalue of the covariance and on concept recoverability, together with an explicit upper bound on $\mbox{$|\representation-\representation'|_F$}$. %, and then propagate perturbation bounds through the covariance, predictive Gram matrix, maxima, and saddles. 
% A more quantitative version could replace this neighborhood argument by explicit assumptions on the representation norm, a lower bound on the smallest eigenvalue of the covariance, a lower bound on concept recoverability, and a sufficiently small bound on $\mbox{$|\representation-\representation'|_F$}$. These quantities would yield explicit constants by successively bounding the perturbations of the covariance, its inverse, the predictive Gram matrix, the maxima, and the saddles. 
Extending the result to singular covariances would additionally require controlling the Moore--Penrose pseudoinverse, for instance under fixed-rank and nonzero singular value lower bound assumptions.
% Extending the result to singular covariances would additionally require controlling the Moore--Penrose pseudoinverse, typically through a fixed-rank assumption and a positive lower bound on the smallest nonzero singular value; without such spectral control, the pseudoinverse can become unstable when the rank changes. 
The persistence pairing itself need not remain fixed, since persistence stability directly controls such changes.
\fi
Handling rank changes requires controlling the Moore--Penrose pseudoinverse and is not needed for the local stability statement above. 
A more quantitative version could replace this neighborhood by explicit lower bounds on the smallest eigenvalue and on concept recoverability, together with an explicit upper bound on $\mbox{$|\representation-\representation'|_F$}$.
Importantly, the proof does not assume that the persistence pairing remains fixed, since the persistence stability theorem directly controls such changes. % of the diagram even when the pairing changes.
\end{remark}

\subsection{Network-to-\acronym{} Stability}

The representation-level result immediately yields a network-level statement
whenever the selected representation is locally Lipschitz with respect to the
network parameters. Let \mbox{$\representation_\theta$} denote the matrix collecting
the selected representations of the same \mbox{$\datasetSize$} samples for a network
with parameters \mbox{$\theta$}.

\begin{corollary}[Network-to-\acronym{} stability]
\label{cor_ctr_network_stability}
Assume the hypotheses of \autoref{thm_ctr_representation_stability} hold at
\mbox{$\representation_\theta$}, with stability radius \mbox{$\varepsilon$}, and assume that
the representation map is locally Lipschitz in the parameters: there exist
\mbox{$L_\theta>0$} and \mbox{$\eta>0$} such that
\mbox{$\|\representation_\theta-\representation_{\theta'}\|_F
\le L_\theta\|\theta-\theta'\|$} whenever \mbox{$\|\theta-\theta'\|<\eta$}.
Then there exists \mbox{$C>0$} such that, for every \mbox{$\theta'$} satisfying:
\[
\|\theta-\theta'\|
<
\min\!\left(\eta,\frac{\varepsilon}{L_\theta}\right),
\]
we have:
\[
d_B\!\left(
\acronym{}(\representation_\theta),
\acronym{}(\representation_{\theta'})
\right)
\le
C L_\theta\|\theta-\theta'\|.
\]
\end{corollary}

\begin{proof}
The parameter restriction ensures that
\mbox{$\|\representation_\theta-\representation_{\theta'}\|_F<\varepsilon$}, so
\mbox{$\representation_{\theta'}$} lies inside the representation-level stability
ball. Applying \autoref{thm_ctr_representation_stability} and then the assumed
local Lipschitz bound for the representation map proves the result.
\end{proof}

\section{Detailed Experimental Setup}
\label{sec_app_exp_setup}

%%%%%%%%%%%%%%%%%%%%%%%%%%%%%%%%%%%%%%%%%%%%%%%%%%%%%%%%%%%%%%%%%%%%%%%%%%%%%%%
%%%%%%%%%%%%%%%%%%%%%%%%%%%%%%%%%%%%%%%%%%%%%%%%%%%%%%%%%%%%%%%%%%%%%%%%%%%%%%%
%%%%%%%%%%%%%%%%%%%%%%%%%%%%%%%%%%%%%%%%%%%%%%%%%%%%%%%%%%%%%%%%%%%%%%%%%%%%%%%

\subsection{Setup for Concept Disentanglement Before Erasure}
\label{sec_app_exp_setup_disent}

\myParagraph{Datasets and representations.}
We evaluate four binary classification datasets spanning image and text modalities. \autoref{tab:disentanglement_datasets} summarizes their task and concept labels, split sizes, fixed representations, and dimensions before and after PCA for the erasure experiments. \autoref{tab:disentanglement_groups} summarizes their group size, task-concept agreement $P(y = c)$ and label imbalance $P(y = 0)$.

\begin{table}[b]
    \caption{Datasets and fixed representations used for concept disentanglement before erasure.}
    \label{tab:disentanglement_datasets}
    \centering
    \setlength{\tabcolsep}{2pt}
    \resizebox{\textwidth}{!}{%
    \begin{tabular}{@{}lllrrr lcc@{}}
        \toprule
        \textbf{Dataset} & \textbf{Task $y$} & \textbf{Concept $c$} & \textbf{Train} & \textbf{Val} & \textbf{Test} & \textbf{Representation} & \shortstack{\textbf{Original} \\ \textbf{dim.}} & \shortstack{\textbf{PCA} \\ \textbf{dim.}} \\
        \midrule
        Waterbirds       & Bird type                  & Background & 4,795  & 1,199 & 5,794 & ResNet-50 avgpool    & 2,048 & 128 \\
        Bias in Bios SN  & Profession (surgeon/nurse) & Gender     & 21,145 & 3,256 & 8,136 & BERT \texttt{[CLS]} & 768   & 128 \\
        Dial Mention     & Mention detection           & Dialect    & 56,140 & 8,000 & 8,000 & BERT \texttt{[CLS]} & 768   & 128 \\
        Jigsaw           & Toxicity classification     & Gender     & 24,508 & 2,724 & 6,808 & BERT \texttt{[CLS]} & 768   & 128 \\
        \bottomrule
    \end{tabular}
    }
\end{table}

\begin{table}[b!]
    \caption{Group composition of each dataset split. The four group sizes are indexed by the task and concept labels $(y,c)$, with their percentage of the split shown in parentheses. Agreement is $P(y=c)$ and label imbalance is reported as $P(y=0)$.}
    \label{tab:disentanglement_groups}
    \centering
    \setlength{\tabcolsep}{5pt}
    \resizebox{\textwidth}{!}{%
    \begin{tabular}{ll *{4}{r@{\,}l} rr}
        \toprule
        \textbf{Dataset} & \textbf{Split} & \multicolumn{8}{c}{\textbf{Group size $n_{yc}$}} & \shortstack{\textbf{Agreement} \\ $P(y=c)$} & \shortstack{\textbf{Label} \\ \textbf{imbalance} \\ $P(y=0)$} \\
        \cmidrule(lr){3-10}
        & & \multicolumn{2}{c}{\textbf{$n_{00}$}} & \multicolumn{2}{c}{\textbf{$n_{01}$}} & \multicolumn{2}{c}{\textbf{$n_{10}$}} & \multicolumn{2}{c}{\textbf{$n_{11}$}} & & \\
        \midrule
        \multirow{3}{*}{Waterbirds}
            & Train & 3,498  & (72.95\%) & 184   & (3.84\%)  & 56    & (1.17\%)  & 1,057  & (22.04\%) & 94.99\% & 76.79\% \\
            & Val   & 467    & (38.95\%) & 466   & (38.87\%) & 133   & (11.09\%) & 133    & (11.09\%) & 50.04\% & 77.81\% \\
            & Test  & 2,255  & (38.92\%) & 2,255 & (38.92\%) & 642   & (11.08\%) & 642    & (11.08\%) & 50.00\% & 77.84\% \\
        \midrule
        \multirow{3}{*}{Bias in Bios SN}
            & Train & 7,521  & (35.57\%) & 1,308 & (6.19\%) & 1,127 & (5.33\%) & 11,189 & (52.92\%) & 88.48\% & 41.75\% \\
            & Val   & 1,158  & (35.57\%) & 202   & (6.20\%) & 174   & (5.34\%) & 1,722  & (52.89\%) & 88.45\% & 41.77\% \\
            & Test  & 2,894  & (35.57\%) & 504   & (6.19\%) & 434   & (5.33\%) & 4,304  & (52.90\%) & 88.47\% & 41.76\% \\
        \midrule
        \multirow{3}{*}{Dial Mention}
            & Train & 25,263 & (45.00\%) & 2,807 & (5.00\%)  & 2,807 & (5.00\%)  & 25,263 & (45.00\%) & 90.00\% & 50.00\% \\
            & Val   & 2,000  & (25.00\%) & 2,000 & (25.00\%) & 2,000 & (25.00\%) & 2,000  & (25.00\%) & 50.00\% & 50.00\% \\
            & Test  & 2,000  & (25.00\%) & 2,000 & (25.00\%) & 2,000 & (25.00\%) & 2,000  & (25.00\%) & 50.00\% & 50.00\% \\
        \midrule
        \multirow{3}{*}{Jigsaw}
            & Train & 11,029 & (45.00\%) & 1,225 & (5.00\%)  & 1,225 & (5.00\%)  & 11,029 & (45.00\%) & 90.00\% & 50.00\% \\
            & Val   & 681    & (25.00\%) & 681   & (25.00\%) & 681   & (25.00\%) & 681    & (25.00\%) & 50.00\% & 50.00\% \\
            & Test  & 1,702  & (25.00\%) & 1,702 & (25.00\%) & 1,702 & (25.00\%) & 1,702  & (25.00\%) & 50.00\% & 50.00\% \\
        \bottomrule
    \end{tabular}
    }
\end{table}

Waterbirds \citep{waterbirds} is constructed by segmenting bird foregrounds from the CUB dataset and compositing them onto land or water scenes from Places. The task label distinguishes landbirds from waterbirds, while the concept records the background type. In the training set, 95\% of each bird type appears on its associated background; validation and test instead allocate each bird type equally across both backgrounds. The resulting shift makes background strongly predictive during training while retaining rare landbird-on-water and waterbird-on-land groups, which expose reliance on the spurious background cue during evaluation.

Bias in Bios SN \citep{bias_in_bios} is obtained by retaining the surgeon and nurse examples from Bias in Bios, a corpus of 397,340 English online biographies collected from Common Crawl and labeled with one of 28 occupations and a binary gender inferred from gendered pronouns. We preserve the original train, validation, and test splits after filtering. The task is to predict which of the two professions is associated with a biography, while the concept is gender. This subset preserves the pronounced profession--gender imbalance in the source corpus and therefore tests whether gender information can be erased without removing profession-relevant content.

Dial Mention \citep{dial_blodgett,dial_elazar} is built from Twitter messages in the DIAL corpus. Its task is to determine whether a tweet originally contained an \texttt{@mention}, thereby distinguishing conversational from non-conversational tweets; all mention tokens are removed before the tweet is encoded, so they cannot directly reveal the label. The concept distinguishes the African-American English (AAE) and Standard American English (SAE) dialect categories, which the original benchmark uses as proxies for non-Hispanic Black and non-Hispanic White authors, respectively. The source construction retains only dialect assignments whose estimated confidence exceeds 80\%, selecting tweets for which the dialect proxy is comparatively unambiguous while preserving dialect-linked lexical variation in the encoded text.

Jigsaw \citep{jigsaw} contains public online comments from the Civil Comments platform, crowd-annotated for both toxicity and mentioned identities. We assign the positive task label when the aggregated toxicity score is greater than $0.5$. For the concept, we retain comments for which at least one selected identity annotation is available and aggregate eight source attributes: female, male, transgender, other gender, heterosexual, homosexual/gay/lesbian, bisexual, and other sexual orientation. The concept is positive when any of these attributes has a nonzero annotation and negative otherwise; we use the compact name ``Gender'' in tables and results.

We extract one fixed representation per example: the average-pooling output of an ImageNet-pretrained ResNet-50 for Waterbirds, and the final-layer \texttt{[CLS]} representation of BERT-base for the text datasets. Before applying either \acronym{} or an erasure baseline, we fit PCA on the training representations and use the same transformation for validation and test data. Each dataset is therefore represented in its own 128-dimensional PCA space. The feature extractors remain frozen throughout, so methods differ only in how they transform these representations post hoc.

\myParagraph{Evaluation protocol.}
After each representation intervention, we independently fit linear logistic-regression probes
for the task $y$ and concept $c$ on the transformed training representations. Inputs are
standardized using training statistics and the $\ell_2$ regularization parameter is selected
from $C\in\{0.01,0.1,1,10,100\}$ by validation accuracy. The selected probes are kept fixed
when evaluating the test set. Task utility is measured by accuracy and worst-group accuracy
(WGA), where the four groups correspond to the joint task--concept labels $(y,c)$ and WGA
is the minimum task accuracy over these groups.
Our primary concept-leakage metric is the adjusted balanced accuracy of the linear concept
probe, defined as $\mathrm{Leak}=2\,\mathrm{BalAcc}-1$, where
$\mathrm{BalAcc}$ averages the concept-class recalls. Thus, $0$ corresponds to chance-level
balanced prediction and $1$ to perfect concept prediction, independently of concept-class
imbalance. The leakage probe itself is unweighted; ``balanced'' refers to the evaluation
metric. Unless stated otherwise, the strong-erasure setting requires
$\mathrm{Leak}<0.1$. Additional oracle-conditioned and non-linear leakage diagnostics are
described in \subappref{sec_app_exp_disent_leak}.

\myParagraph{Model selection.}
For methods with tunable erasure hyperparameters, candidates under a leakage constraint
$t$ are considered feasible only when both their training and validation adjusted balanced
linear leakage are at most $t$. Among feasible candidates, we maximize the mean of validation
task accuracy and validation WGA. Test metrics are used only after model selection.
% For multi-seed parameter searches, the five best candidates are first evaluated over five seeds; the best mean configuration is then evaluated over five additional seeds, yielding the ten-seed averages reported in the main experiments.

\myParagraph{\acronym{} configuration.}
For both \acronym{}-DIS and \acronym{}-ERA, we use BCE recoverability with
differentiable linear logistic probes, with $C=0.01$, Newton updates and
convergence-based stopping. The representation transformation is a two-layer ReLU MLP
mapping the 128-dimensional PCA representation through a hidden dimension of 256 back to
128 dimensions. \acronym{}-DIS additionally uses interference saddles between the concept
and task nodes and is optimized for 200 iterations. Its learning rate is selected on validation
after applying LEACE to the transformed representation, after which the same learned
disentangling transformation is evaluated with all considered erasure methods.
In contrast, \acronym{}-ERA directly suppresses the concept score while preserving the task
score and does not use a saddle interaction. Its learning rate, task-preservation weight and
number of optimization iterations are selected jointly using the same leakage-constrained
validation criterion used for model selection, considering both SGD and Adam candidates.

%%%%%%%%%%%%%%%%%%%%%%%%%%%%%%%%%%%%%%%%%%%%%%%%%%%%%%%%%%%%%%%%%%%%%%%%%%%%%%%
%%%%%%%%%%%%%%%%%%%%%%%%%%%%%%%%%%%%%%%%%%%%%%%%%%%%%%%%%%%%%%%%%%%%%%%%%%%%%%%
%%%%%%%%%%%%%%%%%%%%%%%%%%%%%%%%%%%%%%%%%%%%%%%%%%%%%%%%%%%%%%%%%%%%%%%%%%%%%%%

\subsection{Setup for Concept Transfer}
\label{sec_app_exp_setup_disti}

\myParagraph{Dataset and architectures.}
We use the standard CIFAR-100 \citep{cifar100} split, containing 50,000 training and 10,000
test images from 100 classes. For hyperparameter selection, we hold out a deterministic
class-stratified 10\% of the training set, resulting in 45,000 training and 5,000 validation
images. After selecting the hyperparameters, we discard this validation split, retrain on all
50,000 training images, and report performance on the 10,000 test images. Following the
protocol of CKD \citet{contrastiveKD}, we use their pretrained ResNet32$\times$4 teacher
and consider ResNet8$\times$4 and ShuffleNetV2 students. The teacher is frozen throughout
distillation, while the students are trained from scratch. Their architectures are summarized in \autoref{tab:distillation_architectures}.

\begin{table}[b!]
\caption{Teacher and student architectures used for concept transfer.}
\label{tab:distillation_architectures}
\centering
\setlength{\tabcolsep}{8pt}
\begin{tabular}{@{}llcrl@{}}
\toprule
\textbf{Role} & \textbf{Architecture} &
\shortstack{\textbf{Penultimate}\\ \textbf{dim.}} & \textbf{Parameters} &
\textbf{Training} \\
\midrule
Teacher & ResNet32$\times$4 & 256   & 7,433,860 & CKD pretrained checkpoint, frozen \\
Student & ResNet8$\times$4  & 256   & 1,233,540 & From scratch \\
Student & ShuffleNetV2       & 1,024 & 1,355,528 & From scratch \\
\bottomrule
\end{tabular}
\end{table}

\myParagraph{Distillation training protocol and losses.}
We follow the CIFAR-100 training protocol of CKD \citep{contrastiveKD}: all students
are trained for 240 epochs with a batch size of 64 using SGD with momentum $0.9$ and weight
decay $5 \times 10^{-4}$. The initial learning rate is $0.05$ for ResNet8$\times$4 and $0.01$
for ShuffleNetV2 and is multiplied by $0.1$ after epochs 150, 180 and 210. We use random
cropping and horizontal flipping during training and the standard CIFAR-100 normalization.
For logit distillation, the temperature is fixed to $T=4$. The weights used by each method are
summarized in \autoref{tab:distillation_losses}, where the total objective is written as
$\gamma\mathcal{L}_{\mathrm{CE}}+\alpha\mathcal{L}_{\mathrm{KD}}
+\beta\mathcal{L}_{\mathrm{aux}}$. FitNet uses the CKD hint layer 2 with a learned
convolutional regressor and a mean-squared hint loss, while SPKD matches pairwise similarities
of the last convolutional feature maps.

\begin{table}[t!]
\begin{minipage}[t]{0.56\textwidth}
\caption{Weights used in the concept transfer experiments. A dash indicates that no
auxiliary loss is used.}
\label{tab:distillation_losses}
\centering
\setlength{\tabcolsep}{7pt}
\resizebox{\linewidth}{!}{%
\begin{tabular}{@{}lcccl@{}}
\toprule
\textbf{Method} & $\boldsymbol{\gamma}$ & $\boldsymbol{\alpha}$ &
$\boldsymbol{\beta}$ & \textbf{Auxiliary loss} \\
\midrule
Student & 1.0 & 0   & --   & -- \\
KD      & 0.1 & 0.9 & --   & -- \\
FitNet  & 1.0 & 0   & 100  & Hint regression \\
SPKD    & 1.0 & 0   & 3000 & Similarity preservation \\
TopKD   & 0.1 & 0.9 & 4.0  & Topology distillation \\
\acronym{} $k_{10}$ & 0.1 & 0.9 & 1.0 & \acronym{} matching \\
\acronym{} $k_{20}$ & 0.1 & 0.9 & 0.8 & \acronym{} matching \\
\bottomrule
\end{tabular}
}
\end{minipage}%
\hfill
\begin{minipage}[t]{0.41\textwidth}

\caption{Validation-selected auxiliary-loss weights $\beta$ for \acronym{}.}
\label{table_disti_aux}
\centering
\resizebox{\linewidth}{!}{%
\begin{tabular}{lcccc}
\toprule
& $k_2$ & $k_5$ & $k_{10}$ & $k_{20}$ \\
\midrule
Full image         & 30.0 & 20.0 & 10.0 & 1.0 \\
4 crops   & 5.0  & 2.0  & 0.2  & 0.9 \\
9 crops   & 6.0  & 1.3  & 1.0  & 0.8 \\
16 crops   & 0.9  & 1.0  & 2.0  & 2.0 \\
\bottomrule
\end{tabular}
}
\end{minipage}
\end{table}

For \acronym{}, the auxiliary-loss weight is selected using the held-out CIFAR-100 validation
split described above; the selected configuration is then retrained from scratch using the full
training set before test evaluation. We apply the same procedure to TopKD \citep{topkd}.
TopKD is not part of the CKD benchmark and its loss coefficients are tuned in the
original work for a different teacher--student configuration. We therefore retain the same
$\gamma=0.1$ and $\alpha=0.9$ CE/KD weighting as our KD baseline and select only the weight
of its topology loss on validation data. This ensures that improvements are compared under
the same underlying CE/KD objective rather than being confounded by different loss weighting.
%For TopKD, the resulting auxiliary weight is $\beta=4.0$. 
Final results are averaged over five
random seeds.

Concept Distillation \citep{conceptDistillation} leverages teacher-defined concepts to
(de)sensitize a student to selected concepts rather than transferring their organization between
representations and is therefore not directly applicable to our setting.

\myParagraph{\acronym{} construction.}
For concept transfer, we use the analytical normalized-MSE formulation of
\autoref{sec_app_mse} with least-squares probes and the continuous Dictionary Learning
coefficients as concept targets. 
%The reported transfer experiments match concept scores without pairwise saddle interactions. 
During training, the student \acronym{} is estimated
from a rolling bank of 2,048 augmented training representations, whose least-squares
sufficient statistics are updated incrementally. The first epoch initializes this bank before
\acronym{} matching is applied. Concept identities and coefficients are obtained once from
the deterministic frozen-teacher concept discovery, and the resulting teacher \acronym{}
remains a fixed reference throughout student training.

\myParagraph{Concept discovery and recoverability.}
We discover concepts from the frozen teacher penultimate representation using Dictionary
Learning with $K\in\{2,5,10,20\}$ components, non-negative sparse codes and a fixed discovery
seed. Each value of $K$ defines an independently fitted dictionary: the $K=2,5,10$ and $20$
concept banks are therefore not nested subsets of a common dictionary. This distinction is
important when evaluating generalization across concept banks in \autoref{sec_exp_transfer}.
The same frozen teacher-defined dictionaries are used when evaluating all distillation methods,
so concept recoverability always refers to an identical target concept representation.

In addition to concepts discovered from full images, we consider concepts discovered at
progressively finer spatial scales. Following the crop-based concept discovery strategy of
CRAFT \citep{craft}, neighboring crops use a stride approximately equal to $80\%$ of their
width. We construct deterministic $2\times2$, $3\times3$ and $4\times4$ grids, corresponding
respectively to 4, 9 and 16 crops per image. For the $32\times32$ CIFAR-100 images, the
resulting crop sizes are $18\times18$, $12\times12$ and $9\times9$ pixels, with strides 14, 9
and 7 pixels, respectively. Each crop is resized to $32\times32$ before extracting its teacher
representation. The dictionary is fitted on these crop representations and subsequently used
to encode the full-image teacher representations, providing concept coefficients at increasingly
local spatial scales. Together with the full-image setting, this yields four concept-discovery
scales. % for each $K\in\{2,5,10,20\}$.
The main transfer comparison in \autoref{table_distillation} uses the $3\times3$
(9-crop) concept discovery, while the other scales are used to evaluate the effect
of concept locality and generalization across independently discovered concept banks.

\iffalse
We quantify concept recoverability (CR) by measuring how accurately the continuous teacher
concept coefficients can be linearly recovered from a representation. For every concept $j$,
we fit a least-squares linear probe on the CIFAR-100 training representations and corresponding
teacher coefficients, without using test observations during fitting. Its held-out score is:

$$
    R_j^2
    =
    1-
    \frac{\sum_{i\in\mathrm{test}}(c_{ij}-\widehat{c}_{ij})^2}
         {\sum_{i\in\mathrm{test}}(c_{ij}-\overline{c}_{j,\mathrm{train}})^2},
$$

where $c_{ij}$ is the teacher coefficient and $\overline{c}_{j,\mathrm{train}}$ is its training-set
mean. We report CR $k_K$ as the mean held-out $R^2$ across the $K$ concepts, expressed
as a percentage. Thus, CR measures how much of the teacher-defined concept representation
remains linearly accessible from the student representation, independently of its task accuracy.
\fi

We quantify concept recoverability (CR) by measuring how accurately the teacher-derived concept coefficients can be linearly decoded from a model representation on unseen images. For a dictionary containing $K$ concepts, let $C\in\mathbb{R}^{N\times K}$ denote the corresponding teacher concept coefficients and $H\in\mathbb{R}^{N\times d}$ the representation whose recoverability we evaluate. Using the CIFAR-100 training set, we fit a linear least-squares decoder from $H_{\mathrm{train}}$ to $C_{\mathrm{train}}$: $\widehat C_{\mathrm{train}} = H_{\mathrm{train}} W + b$, where $W$ and $b$ are estimated exclusively from the training samples. The same frozen decoder is then applied to the test representation, $\widehat C_{\mathrm{test}} = H_{\mathrm{test}} W + b$.

For a test image $i$ and concept $j$, $\widehat c_{ij}$ therefore denotes the coefficient predicted by this linear decoder, while $c_{ij}$ denotes the corresponding coefficient obtained by encoding the teacher representation with the fixed concept dictionary.
For each concept $j$, we compute the test coefficient of determination:

$$
    R_j^2
    =
    1-
    \frac{\sum_{i\in\mathrm{test}}(c_{ij}-\widehat c_{ij})^2}
         {\sum_{i\in\mathrm{test}}(c_{ij}-\overline c_{j,\mathrm{train}})^2},
$$

where $\overline c_{j,\mathrm{train}}$ is the mean teacher coefficient of concept $j$ over the training set and therefore defines a constant predictor that uses no test information. An $R_j^2$ of $1$ indicates perfect linear recovery of the teacher coefficient, while $0$ corresponds to the performance of this training-mean predictor; negative values indicate worse generalization than the constant predictor. We finally define CR $k_K$ as the mean test $R^2$ over the $K$ concepts and report it as a percentage. 
Thus, CR measures how much of the teacher-defined concept representation remains linearly accessible from the evaluated representation on unseen data: a high CR indicates that a linear decoder fitted to predict the teacher's concept coefficients on the training set continues to recover those same concepts accurately on previously unseen images, independently of the model's task accuracy.

%%%%%%%%%%%%%%%%%%%%%%%%%%%%%%%%%%%%%%%%%%%%%%%%%%%%%%%%%%%%%%%%%%%%%%%%%%%%%%%
%%%%%%%%%%%%%%%%%%%%%%%%%%%%%%%%%%%%%%%%%%%%%%%%%%%%%%%%%%%%%%%%%%%%%%%%%%%%%%%
%%%%%%%%%%%%%%%%%%%%%%%%%%%%%%%%%%%%%%%%%%%%%%%%%%%%%%%%%%%%%%%%%%%%%%%%%%%%%%%
%%%%%%%%%%%%%%%%%%%%%%%%%%%%%%%%%%%%%%%%%%%%%%%%%%%%%%%%%%%%%%%%%%%%%%%%%%%%%%%
%%%%%%%%%%%%%%%%%%%%%%%%%%%%%%%%%%%%%%%%%%%%%%%%%%%%%%%%%%%%%%%%%%%%%%%%%%%%%%%
%%%%%%%%%%%%%%%%%%%%%%%%%%%%%%%%%%%%%%%%%%%%%%%%%%%%%%%%%%%%%%%%%%%%%%%%%%%%%%%

\section{Additional Experiments \& Results}
\label{sec_app_exp}

\subsection{Concept Disentanglement Before Erasure}
\label{sec_app_exp_disent}

% Generated with: /home/mathieu/my-venv/bin/python3 /media/mathieu/Data/Etude/Postdoc/ourPapers/conceptTopologicalRepresentation/figures/table_wga/make_table.py --siunitx --no-math --scriptsize --underline-second-best --select-by-validation-wga --aspect-ratio 1.3 --text-scale 2.6 --style-scale 2.9 --label-spacing 1.08 1.18 --label-point-force 1.35 1.70 --label-pull 0.028 0.040 --label-iterations 2500
% Requires \usepackage{siunitx}
% Optional resizing requires \usepackage{graphicx}
\newcommand{\tableTimingTabular}{%
\begin{tabular}{l *{4}{S[table-format=4.2{ {\scriptsize \textbf{\textpm 00.00}}}, detect-weight=true, table-align-text-after=true, mode=text]}}
\toprule
\textbf{Method}             & {\textbf{Waterbirds (s)}}                                                                     & {\textbf{Bias in Bios SN (s)}}                                                                & {\textbf{Dial Mention (s)}}                                                                   & {\textbf{Jigsaw (s)}}                                                                         \\
\midrule
INLP                        & 1.35 {\scriptsize \textpm 0.41}                                                               & 5.57 {\scriptsize \textpm 1.01}                                                               & 28.05 {\scriptsize \textpm 1.90}                                                              & 3.07 {\scriptsize \textpm 0.65}                                                               \\
LEACE                       & 0.03 {\scriptsize \textpm 0.06}{\llap{\underline{\phantom{0.03 {\scriptsize \textpm 0.06}}}}} & \bfseries 0.03 {\scriptsize \textbf{\textpm 0.06}}                                            & \bfseries 0.04 {\scriptsize \textbf{\textpm 0.07}}                                            & \bfseries 0.02 {\scriptsize \textbf{\textpm 0.05}}                                            \\
RLACE                       & 38.91 {\scriptsize \textpm 1.13}                                                              & 65.28 {\scriptsize \textpm 2.56}                                                              & 208.60 {\scriptsize \textpm 3.70}                                                             & 96.77 {\scriptsize \textpm 6.67}                                                              \\
SPLINCE                     & \bfseries 0.03 {\scriptsize \textbf{\textpm 0.02}}                                            & 0.15 {\scriptsize \textpm 0.20}{\llap{\underline{\phantom{0.15 {\scriptsize \textpm 0.20}}}}} & 0.30 {\scriptsize \textpm 0.33}{\llap{\underline{\phantom{0.30 {\scriptsize \textpm 0.33}}}}} & 0.10 {\scriptsize \textpm 0.03}{\llap{\underline{\phantom{0.10 {\scriptsize \textpm 0.03}}}}} \\
LEOPARD                     & 10.65 {\scriptsize \textpm 0.61}                                                              & 47.28 {\scriptsize \textpm 0.36}                                                              & 132.32 {\scriptsize \textpm 0.56}                                                             & 53.18 {\scriptsize \textpm 0.21}                                                              \\
KRaM                        & 5.39 {\scriptsize \textpm 1.30}                                                               & 36.58 {\scriptsize \textpm 4.44}                                                              & 59.58 {\scriptsize \textpm 28.59}                                                             & 45.26 {\scriptsize \textpm 5.67}                                                              \\
Obliviator                  & 1.67 {\scriptsize \textpm 0.14}                                                               & 93.14 {\scriptsize \textpm 0.29}                                                              & 8.62 {\scriptsize \textpm 0.27}                                                               & 11.95 {\scriptsize \textpm 0.32}                                                              \\
\midrule
\acronym{}-DIS + INLP       & 48.18 {\scriptsize \textpm 4.60}                                                              & 42.36 {\scriptsize \textpm 1.95}                                                              & 277.02 {\scriptsize \textpm 34.47}                                                            & 35.52 {\scriptsize \textpm 1.55}                                                              \\
\acronym{}-DIS + LEACE      & 47.33 {\scriptsize \textpm 4.36}                                                              & 41.06 {\scriptsize \textpm 2.05}                                                              & 88.46 {\scriptsize \textpm 16.05}                                                             & 32.74 {\scriptsize \textpm 1.54}                                                              \\
\acronym{}-DIS + RLACE      & 84.51 {\scriptsize \textpm 5.65}                                                              & 188.65 {\scriptsize \textpm 7.93}                                                             & 1001.34 {\scriptsize \textpm 679.06}                                                          & 113.47 {\scriptsize \textpm 2.43}                                                             \\
\acronym{}-DIS + SPLINCE    & 47.47 {\scriptsize \textpm 4.56}                                                              & 41.12 {\scriptsize \textpm 2.05}                                                              & 88.82 {\scriptsize \textpm 16.22}                                                             & 32.81 {\scriptsize \textpm 1.53}                                                              \\
\acronym{}-DIS + LEOPARD    & 57.39 {\scriptsize \textpm 4.37}                                                              & 88.70 {\scriptsize \textpm 2.02}                                                              & 226.28 {\scriptsize \textpm 20.29}                                                            & 86.53 {\scriptsize \textpm 1.52}                                                              \\
\acronym{}-DIS + KRaM       & 51.92 {\scriptsize \textpm 4.24}                                                              & 56.00 {\scriptsize \textpm 1.99}                                                              & 127.32 {\scriptsize \textpm 16.13}                                                            & 52.25 {\scriptsize \textpm 4.41}                                                              \\
\acronym{}-DIS + Obliviator & 49.01 {\scriptsize \textpm 4.49}                                                              & 53.01 {\scriptsize \textpm 1.95}                                                              & 118.78 {\scriptsize \textpm 15.91}                                                            & 54.06 {\scriptsize \textpm 1.43}                                                              \\
\acronym{}-ERA              & 4.05 {\scriptsize \textpm 0.47}                                                               & 10.42 {\scriptsize \textpm 1.14}                                                              & 142.63 {\scriptsize \textpm 1.55}                                                             & 8.24 {\scriptsize \textpm 0.48}                                                               \\
\bottomrule
\end{tabular}%
}
\NewDocumentCommand{\tableTiming}{m m m o}{%
\begin{table}[#3]
\caption{#1}
\label{#2}
\centering
\IfNoValueTF{#4}{%
\tableTimingTabular
}{%
\resizebox{#4\textwidth}{!}{\tableTimingTabular}%
}%
\end{table}%
}

\tableTiming{Execution time (seconds) for standalone erasure methods, \acronym{}-DIS followed by erasure, and direct \acronym{}-ERA.}{tab_timing}{b!}[1.0]

\myParagraph{Runtime.}
We additionally report the execution time of each intervention in \autoref{tab_timing}. As expected, using \acronym{}-DIS as a preprocessing step introduces an additional optimization cost before applying the corresponding erasure method. For lightweight methods such as LEACE and SPLINCE, the resulting runtime nevertheless remains below $90$ seconds on all datasets, while combining \acronym{}-DIS with more computationally intensive methods such as RLACE can be substantially more expensive, reaching $1001.34$ seconds on DIAL-Mention. Direct \acronym{}-ERA avoids this additional erasure stage and requires between $4.05$ and $142.63$ seconds across the four datasets. In particular, it is faster than RLACE on every dataset and faster than LEOPARD on three out of four datasets, showing that direct erasure with \acronym{} can remain computationally competitive with iterative erasure baselines.

\myParagraph{Test accuracy.} In addition of the test WGA shown in \autoref{table_wga},
% of \autoref{sec_exp_dis}, 
test accuracy is shown in \autoref{table_acc}. 
The test task accuracy provides a complementary view to WGA. On Waterbirds and Bias in Bios, the gains from \acronym{}-DIS also translate into substantially higher overall accuracy, whereas on DIAL-Mention and Jigsaw the accuracy remains comparable or can slightly decrease despite consistent improvements in WGA. This difference highlights that aggregate accuracy can mask group-specific degradation, while WGA directly measures utility on the most affected group. Notably, \acronym{}-ERA achieves higher task accuracy than every standalone erasure baseline on all four datasets and obtains the highest accuracy among all erasure methods on Waterbirds, DIAL-Mention and Jigsaw.

%%%%%%%%%%%%%%%%%%%%%%%%%%%%%%%%%%%%%%%
%%%%%%%%%%%%%%%%%%%%%%%%%%%%%%%%%%%%%%%

\subsubsection{Beyond-linear Concept Leakage}
\label{sec_app_exp_disent_leak}

\begin{samepage}
%\afterpage{
% Generated with: /home/mathieu/my-venv/bin/python3 /media/mathieu/Data/Etude/Postdoc/ourPapers/conceptTopologicalRepresentation/figures/table_wga/make_table.py --siunitx --no-math --scriptsize --underline-second-best --select-by-validation-wga --aspect-ratio 1.3 --text-scale 2.6 --style-scale 2.9 --label-spacing 1.08 1.18 --label-point-force 1.35 1.70 --label-pull 0.028 0.040 --label-iterations 2500
% Requires \usepackage{siunitx}
% Optional resizing requires \usepackage{graphicx}
\newcommand{\tableACCTabular}{%
\begin{tabular}{l *{4}{S[table-format=2.2{ {\scriptsize \textbf{\textpm 00.00}}}, detect-weight=true, table-align-text-after=true, mode=text]}}
\toprule
\textbf{Method}             & {\textbf{Waterbirds}}                                                                           & {\textbf{Bias in Bios SN}}                                                                      & {\textbf{Dial Mention}}                                                                         & {\textbf{Jigsaw}}                                                                               \\
\midrule
\textit{Raw}                & {\textit{88.95}}                                                                                & {\textit{96.76}}                                                                                & {\textit{64.91}}                                                                                & {\textit{71.56}}                                                                                \\
\midrule
INLP                        & 85.17 {\scriptsize \textpm 0.00}                                                                & 71.66 {\scriptsize \textpm 0.00}                                                                & 63.49 {\scriptsize \textpm 0.00}                                                                & 65.82 {\scriptsize \textpm 0.00}                                                                \\
LEACE                       & 85.67 {\scriptsize \textpm 0.00}                                                                & 76.55 {\scriptsize \textpm 0.00}                                                                & 65.06 {\scriptsize \textpm 0.00}{\llap{\underline{\phantom{65.06 {\scriptsize \textpm 0.00}}}}} & 67.58 {\scriptsize \textpm 0.00}{\llap{\underline{\phantom{67.58 {\scriptsize \textpm 0.00}}}}} \\
RLACE                       & 85.64 {\scriptsize \textpm 0.25}                                                                & 76.62 {\scriptsize \textpm 0.33}                                                                & 64.78 {\scriptsize \textpm 0.21}                                                                & 67.30 {\scriptsize \textpm 0.17}                                                                \\
SPLINCE                     & 85.66 {\scriptsize \textpm 0.00}                                                                & 76.50 {\scriptsize \textpm 0.00}                                                                & 65.06 {\scriptsize \textpm 0.00}{\llap{\underline{\phantom{65.06 {\scriptsize \textpm 0.00}}}}} & 67.51 {\scriptsize \textpm 0.00}                                                                \\
LEOPARD                     & 85.66 {\scriptsize \textpm 0.00}                                                                & 76.55 {\scriptsize \textpm 0.00}                                                                & 64.85 {\scriptsize \textpm 0.02}                                                                & 67.58 {\scriptsize \textpm 0.00}{\llap{\underline{\phantom{67.58 {\scriptsize \textpm 0.00}}}}} \\
KRaM                        & 78.44 {\scriptsize \textpm 0.77}                                                                & 68.23 {\scriptsize \textpm 0.93}                                                                & 59.00 {\scriptsize \textpm 0.40}                                                                & 62.20 {\scriptsize \textpm 0.79}                                                                \\
Obliviator                  & 84.70 {\scriptsize \textpm 0.46}                                                                & 75.08 {\scriptsize \textpm 0.63}                                                                & 62.87 {\scriptsize \textpm 0.61}                                                                & 64.39 {\scriptsize \textpm 0.63}                                                                \\
\midrule
\acronym{}-DIS + INLP       & 88.57 {\scriptsize \textpm 1.32}                                                                & 87.18 {\scriptsize \textpm 2.07}                                                                & 62.72 {\scriptsize \textpm 0.81}                                                                & 59.31 {\scriptsize \textpm 1.92}                                                                \\
\acronym{}-DIS + LEACE      & 88.47 {\scriptsize \textpm 0.55}                                                                & \bfseries 94.68 {\scriptsize \textbf{\textpm 0.21}}                                             & 64.99 {\scriptsize \textpm 0.28}                                                                & 64.98 {\scriptsize \textpm 0.85}                                                                \\
\acronym{}-DIS + RLACE      & 88.98 {\scriptsize \textpm 0.56}{\llap{\underline{\phantom{88.98 {\scriptsize \textpm 0.56}}}}} & 94.59 {\scriptsize \textpm 0.29}                                                                & 63.88 {\scriptsize \textpm 0.61}                                                                & 65.19 {\scriptsize \textpm 0.69}                                                                \\
\acronym{}-DIS + SPLINCE    & 88.64 {\scriptsize \textpm 0.98}                                                                & 94.60 {\scriptsize \textpm 0.27}{\llap{\underline{\phantom{94.60 {\scriptsize \textpm 0.27}}}}} & 64.99 {\scriptsize \textpm 0.31}                                                                & 64.89 {\scriptsize \textpm 0.84}                                                                \\
\acronym{}-DIS + LEOPARD    & 88.94 {\scriptsize \textpm 0.53}                                                                & 94.60 {\scriptsize \textpm 0.31}                                                                & 64.98 {\scriptsize \textpm 0.36}                                                                & 65.01 {\scriptsize \textpm 0.79}                                                                \\
\acronym{}-DIS + KRaM       & 85.14 {\scriptsize \textpm 1.82}                                                                & 74.41 {\scriptsize \textpm 5.91}                                                                & 61.19 {\scriptsize \textpm 0.85}                                                                & 62.42 {\scriptsize \textpm 1.21}                                                                \\
\acronym{}-DIS + Obliviator & 88.13 {\scriptsize \textpm 1.50}                                                                & 89.06 {\scriptsize \textpm 1.39}                                                                & 64.90 {\scriptsize \textpm 0.56}                                                                & 64.78 {\scriptsize \textpm 0.66}                                                                \\
\acronym{}-ERA              & \bfseries 89.69 {\scriptsize \textbf{\textpm 0.38}}                                             & 92.69 {\scriptsize \textpm 0.30}                                                                & \bfseries 66.44 {\scriptsize \textbf{\textpm 0.31}}                                             & \bfseries 67.66 {\scriptsize \textbf{\textpm 0.35}}                                             \\
\bottomrule
\end{tabular}%
}
\NewDocumentCommand{\tableACC}{m m m o}{%
\begin{table}[#3]
\caption{#1}
\label{#2}
\centering
\IfNoValueTF{#4}{%
\tableACCTabular
}{%
\resizebox{#4\textwidth}{!}{\tableACCTabular}%
}%
\end{table}%
}

\tableACC{Test task accuracy with an adjusted balanced linear concept leakage $< 0.1$.}{table_acc}{p}[0.85]

% Generated with: /home/mathieu/my-venv/bin/python3 /media/mathieu/Data/Etude/Postdoc/ourPapers/conceptTopologicalRepresentation/figures/table_wga/make_table.py --siunitx --no-math --scriptsize --underline-second-best --select-by-validation-wga --aspect-ratio 1.3 --text-scale 2.6 --style-scale 2.9 --label-spacing 1.08 1.18 --label-point-force 1.35 1.70 --label-pull 0.028 0.040 --label-iterations 2500
% Requires \usepackage{siunitx}
% Optional resizing requires \usepackage{graphicx}
\newcommand{\tableOCLTabular}{%
\begin{tabular}{l *{4}{S[table-format=2.2{ {\scriptsize \textbf{\textpm 00.00}}}, detect-weight=true, table-align-text-after=true, mode=text]}}
\toprule
\textbf{Method}             & {\textbf{Waterbirds}}                                                                           & {\textbf{Bias in Bios SN}}                                                                      & {\textbf{Dial Mention}}                                                                         & {\textbf{Jigsaw}}                                                                             \\
\midrule
\textit{Raw}                & {\textit{85.71}}                                                                                & {\textit{98.62}}                                                                                & {\textit{61.78}}                                                                                & {\textit{51.88}}                                                                              \\
\midrule
INLP                        & 60.58 {\scriptsize \textpm 0.00}                                                                & 77.59 {\scriptsize \textpm 0.00}                                                                & 34.47 {\scriptsize \textpm 0.00}                                                                & 29.99 {\scriptsize \textpm 0.00}                                                              \\
LEACE                       & 72.35 {\scriptsize \textpm 0.00}                                                                & 95.09 {\scriptsize \textpm 0.00}                                                                & 38.13 {\scriptsize \textpm 0.00}                                                                & 33.37 {\scriptsize \textpm 0.00}                                                              \\
RLACE                       & 72.37 {\scriptsize \textpm 0.26}                                                                & 94.99 {\scriptsize \textpm 0.10}                                                                & 38.30 {\scriptsize \textpm 0.20}                                                                & 33.55 {\scriptsize \textpm 0.49}                                                              \\
SPLINCE                     & 72.14 {\scriptsize \textpm 0.00}                                                                & 94.56 {\scriptsize \textpm 0.00}                                                                & 38.28 {\scriptsize \textpm 0.00}                                                                & 33.31 {\scriptsize \textpm 0.00}                                                              \\
LEOPARD                     & 72.59 {\scriptsize \textpm 0.00}                                                                & 94.99 {\scriptsize \textpm 0.00}                                                                & 35.84 {\scriptsize \textpm 0.26}                                                                & 33.34 {\scriptsize \textpm 0.00}                                                              \\
KRaM                        & \bfseries 44.58 {\scriptsize \textbf{\textpm 2.70}}                                             & \bfseries 58.73 {\scriptsize \textbf{\textpm 2.74}}                                             & 23.60 {\scriptsize \textpm 1.27}{\llap{\underline{\phantom{23.60 {\scriptsize \textpm 1.27}}}}} & 22.85 {\scriptsize \textpm 1.56}                                                              \\
Obliviator                  & 63.26 {\scriptsize \textpm 0.96}                                                                & 86.43 {\scriptsize \textpm 1.78}                                                                & 30.83 {\scriptsize \textpm 1.10}                                                                & 24.29 {\scriptsize \textpm 1.54}                                                              \\
\midrule
\acronym{}-DIS + INLP       & 56.41 {\scriptsize \textpm 1.84}                                                                & 89.23 {\scriptsize \textpm 1.32}                                                                & 28.31 {\scriptsize \textpm 1.58}                                                                & \bfseries 5.49 {\scriptsize \textbf{\textpm 2.74}}                                            \\
\acronym{}-DIS + LEACE      & 60.76 {\scriptsize \textpm 3.15}                                                                & 95.77 {\scriptsize \textpm 0.65}                                                                & 30.44 {\scriptsize \textpm 1.17}                                                                & 10.24 {\scriptsize \textpm 1.25}                                                              \\
\acronym{}-DIS + RLACE      & 59.56 {\scriptsize \textpm 1.39}                                                                & 95.79 {\scriptsize \textpm 0.67}                                                                & 28.41 {\scriptsize \textpm 1.80}                                                                & 10.34 {\scriptsize \textpm 1.37}                                                              \\
\acronym{}-DIS + SPLINCE    & 61.01 {\scriptsize \textpm 1.81}                                                                & 95.97 {\scriptsize \textpm 0.47}                                                                & 30.18 {\scriptsize \textpm 1.41}                                                                & 9.96 {\scriptsize \textpm 1.23}{\llap{\underline{\phantom{9.96 {\scriptsize \textpm 1.23}}}}} \\
\acronym{}-DIS + LEOPARD    & 59.03 {\scriptsize \textpm 1.36}                                                                & 95.74 {\scriptsize \textpm 0.65}                                                                & 29.40 {\scriptsize \textpm 1.24}                                                                & 10.41 {\scriptsize \textpm 1.37}                                                              \\
\acronym{}-DIS + KRaM       & 53.00 {\scriptsize \textpm 3.82}                                                                & 65.52 {\scriptsize \textpm 5.20}{\llap{\underline{\phantom{65.52 {\scriptsize \textpm 5.20}}}}} & \bfseries 21.98 {\scriptsize \textbf{\textpm 1.76}}                                             & 11.35 {\scriptsize \textpm 2.34}                                                              \\
\acronym{}-DIS + Obliviator & 59.62 {\scriptsize \textpm 1.26}                                                                & 96.53 {\scriptsize \textpm 0.37}                                                                & 28.85 {\scriptsize \textpm 1.23}                                                                & 12.87 {\scriptsize \textpm 1.78}                                                              \\
\acronym{}-ERA              & 51.93 {\scriptsize \textpm 1.78}{\llap{\underline{\phantom{51.93 {\scriptsize \textpm 1.78}}}}} & 91.82 {\scriptsize \textpm 0.89}                                                                & 34.10 {\scriptsize \textpm 1.65}                                                                & 24.82 {\scriptsize \textpm 1.20}                                                              \\
\bottomrule
\end{tabular}%
}
\NewDocumentCommand{\tableOCL}{m m m o}{%
\begin{table}[#3]
\caption{#1}
\label{#2}
\centering
\IfNoValueTF{#4}{%
\tableOCLTabular
}{%
\resizebox{#4\textwidth}{!}{\tableOCLTabular}%
}%
\end{table}%
}

\tableOCL{Adjusted balanced oracle-conditioned test concept leakage.}{table_ocl}{p}[0.85]

% Generated with: /home/mathieu/my-venv/bin/python3 /media/mathieu/Data/Etude/Postdoc/ourPapers/conceptTopologicalRepresentation/figures/table_wga/make_table.py --siunitx --no-math --scriptsize --underline-second-best --select-by-validation-wga --aspect-ratio 1.3 --text-scale 2.6 --style-scale 2.9 --label-spacing 1.08 1.18 --label-point-force 1.35 1.70 --label-pull 0.028 0.040 --label-iterations 2500
% Requires \usepackage{siunitx}
% Optional resizing requires \usepackage{graphicx}
\newcommand{\tableNLLTabular}{%
\begin{tabular}{l *{4}{S[table-format=2.2{ {\scriptsize \textbf{\textpm 00.00}}}, detect-weight=true, table-align-text-after=true, mode=text]}}
\toprule
\textbf{Method}             & {\textbf{Waterbirds}}                                                                           & {\textbf{Bias in Bios SN}}                                                                      & {\textbf{Dial Mention}}                                                                         & {\textbf{Jigsaw}}                                                                               \\
\midrule
\textit{Raw}                & {\textit{79.83}}                                                                                & {\textit{98.89}}                                                                                & {\textit{55.99}}                                                                                & {\textit{42.77}}                                                                                \\
\midrule
INLP                        & 71.01 {\scriptsize \textpm 1.60}                                                                & 93.64 {\scriptsize \textpm 0.21}                                                                & 52.99 {\scriptsize \textpm 0.65}                                                                & 31.37 {\scriptsize \textpm 0.95}                                                                \\
LEACE                       & 74.60 {\scriptsize \textpm 0.92}                                                                & 98.05 {\scriptsize \textpm 0.21}                                                                & 55.14 {\scriptsize \textpm 0.48}                                                                & 35.95 {\scriptsize \textpm 0.91}                                                                \\
RLACE                       & 74.54 {\scriptsize \textpm 1.79}                                                                & 98.01 {\scriptsize \textpm 0.27}                                                                & 54.96 {\scriptsize \textpm 0.58}                                                                & 36.44 {\scriptsize \textpm 1.11}                                                                \\
SPLINCE                     & 75.14 {\scriptsize \textpm 1.42}                                                                & 97.45 {\scriptsize \textpm 0.26}                                                                & 56.18 {\scriptsize \textpm 0.45}                                                                & 36.91 {\scriptsize \textpm 1.03}                                                                \\
LEOPARD                     & 74.07 {\scriptsize \textpm 1.06}                                                                & 98.02 {\scriptsize \textpm 0.18}                                                                & 53.45 {\scriptsize \textpm 0.47}                                                                & 35.44 {\scriptsize \textpm 1.24}                                                                \\
KRaM                        & \bfseries 40.71 {\scriptsize \textbf{\textpm 4.02}}                                             & 76.64 {\scriptsize \textpm 2.24}{\llap{\underline{\phantom{76.64 {\scriptsize \textpm 2.24}}}}} & \bfseries 36.37 {\scriptsize \textbf{\textpm 1.30}}                                             & \bfseries 11.37 {\scriptsize \textbf{\textpm 2.96}}                                             \\
Obliviator                  & 68.36 {\scriptsize \textpm 0.96}                                                                & 93.38 {\scriptsize \textpm 0.52}                                                                & 50.00 {\scriptsize \textpm 1.29}                                                                & 21.47 {\scriptsize \textpm 1.35}                                                                \\
\midrule
\acronym{}-DIS + INLP       & 76.84 {\scriptsize \textpm 1.18}                                                                & 98.30 {\scriptsize \textpm 0.14}                                                                & 55.07 {\scriptsize \textpm 0.74}                                                                & 32.62 {\scriptsize \textpm 1.67}                                                                \\
\acronym{}-DIS + LEACE      & 78.63 {\scriptsize \textpm 1.86}                                                                & 98.59 {\scriptsize \textpm 0.12}                                                                & 56.49 {\scriptsize \textpm 0.61}                                                                & 37.83 {\scriptsize \textpm 1.90}                                                                \\
\acronym{}-DIS + RLACE      & 78.35 {\scriptsize \textpm 1.46}                                                                & 98.55 {\scriptsize \textpm 0.16}                                                                & 55.64 {\scriptsize \textpm 0.73}                                                                & 38.20 {\scriptsize \textpm 1.11}                                                                \\
\acronym{}-DIS + SPLINCE    & 74.28 {\scriptsize \textpm 3.48}                                                                & 98.60 {\scriptsize \textpm 0.17}                                                                & 57.29 {\scriptsize \textpm 0.63}                                                                & 37.45 {\scriptsize \textpm 1.53}                                                                \\
\acronym{}-DIS + LEOPARD    & 78.64 {\scriptsize \textpm 1.46}                                                                & 98.62 {\scriptsize \textpm 0.16}                                                                & 56.15 {\scriptsize \textpm 0.80}                                                                & 38.15 {\scriptsize \textpm 1.23}                                                                \\
\acronym{}-DIS + KRaM       & 63.46 {\scriptsize \textpm 2.05}{\llap{\underline{\phantom{63.46 {\scriptsize \textpm 2.05}}}}} & \bfseries 76.44 {\scriptsize \textbf{\textpm 3.92}}                                             & 37.14 {\scriptsize \textpm 1.22}{\llap{\underline{\phantom{37.14 {\scriptsize \textpm 1.22}}}}} & 15.71 {\scriptsize \textpm 2.20}{\llap{\underline{\phantom{15.71 {\scriptsize \textpm 2.20}}}}} \\
\acronym{}-DIS + Obliviator & 75.02 {\scriptsize \textpm 2.08}                                                                & 98.48 {\scriptsize \textpm 0.24}                                                                & 55.06 {\scriptsize \textpm 0.53}                                                                & 33.59 {\scriptsize \textpm 1.66}                                                                \\
\acronym{}-ERA              & 65.75 {\scriptsize \textpm 2.14}                                                                & 97.75 {\scriptsize \textpm 0.22}                                                                & 53.02 {\scriptsize \textpm 0.78}                                                                & 28.73 {\scriptsize \textpm 1.24}                                                                \\
\bottomrule
\end{tabular}%
}
\NewDocumentCommand{\tableNLL}{m m m o}{%
\begin{table}[#3]
\caption{#1}
\label{#2}
\centering
\IfNoValueTF{#4}{%
\tableNLLTabular
}{%
\resizebox{#4\textwidth}{!}{\tableNLLTabular}%
}%
\end{table}%
}

\tableNLL{Adjusted balanced non-linear test concept leakage.}{table_nll}{p}[0.85]
%}
\end{samepage}

The experiments of \autoref{sec_exp_dis} use an adjusted balanced linear leakage threshold measured through a linear probe trained with a convex loss.
% This criterion follows the notion of linear guardedness used by prior concept-erasure methods. 
% In particular, LEACE establishes that equality of class centroids, equivalently zero feature--label covariance, is equivalent to the optimality of the constant predictor for linear predictors under losses convex in the prediction \citep{leace}. 
This follows the linear-guardedness perspective of prior concept-erasure work, for instance, LEACE guarantees that, after erasure, no linear predictor can improve over the optimal constant predictor under convex losses \citep{leace}. 
This criterion is therefore broader than a particular choice of logistic regression, but does not imply that the edited representation is statistically independent of the concept or guarded against arbitrary non-linear predictors.
Indeed, direct linear-guardedness need not imply that a downstream model cannot infer the concept indirectly through other preserved information \citep{loglinear_guardedness}.
This distinction is particularly relevant when task and concept labels are correlated, since preserving task-predictive information may necessarily retain some predictability of the concept. 

To assess concept leakage beyond the criterion used for model selection, we therefore report two stronger diagnostics: \emph{oracle-conditioned leakage}, which fits separate balanced linear concept probes within each
task class, and \emph{non-linear leakage}, measured with a non-linear concept probe. 

For the \emph{oracle-conditioned leakage}, we fit two separate linear concept probes, one on samples with \(y=0\) and one on samples with \(y=1\). Each probe is an \(\ell_2\)-regularized logistic regression trained with balanced concept-class weights within the corresponding task train set; its regularization strength is selected on the matching validation set by balanced accuracy. At evaluation time, each sample is routed to the probe corresponding to its ground-truth task label \(y\), and leakage is reported as adjusted balanced accuracy, \(2\,\mathrm{BalAcc}-1\). This asks whether \(c\) remains linearly predictable once \(y\) is held fixed, rather than allowing predictability of \(y\) itself to explain concept leakage.

For \emph{non-linear leakage}, we fit an unweighted ReLU MLP concept probe on the edited representation. Inputs are standardized using training statistics, and the probe is optimized with cross-entropy and Adam. We select the architecture among one, two, or three hidden layers of width \(128\), together with the learning rate and training epoch, according to validation accuracy,
%; the learning-rate search starts from \(10^{-3}\) and explores neighboring values on our standard positive-float lattice. 
the learning rate is searched around \(10^{-3}\) on the \(\{1,2,5\}\times 10^k\) lattice.
We train for at most \(200\) epochs with batch size \(256\). The selected probe is then evaluated on the test set using the same adjusted balanced leakage metric.

The oracle-conditioned results (\autoref{table_ocl}) show that low global linear leakage does not generally imply low leakage after conditioning on the task. 
Nevertheless, \acronym{}-DIS often reduces this stronger conditional leakage relative to the corresponding erasure baseline, with particularly large reductions on Jigsaw and consistent improvements on Waterbirds and Dial Mention. Bias in Bios remains difficult, most methods retain high conditional leakage despite satisfying the linear leakage constraint.

Non-linear leakage (\autoref{table_nll}) remains substantial for many methods and is often not reduced by \mbox{\acronym{}-DIS}. 
This confirms that our main results should be interpreted as improvements in task preservation at matched \emph{linear} concept leakage, rather than as guarantees of complete concept removal. 
KRaM achieve lower non-linear leakage on several datasets, but at a different utility-erasure trade-off, indeed it often has lower wga \autoref{table_wga} and test accuracy \autoref{table_acc} than the other methods.

%%%%%%%%%%%%%%%%%%%%%%%%%%%%%%%%%%%%%%%
%%%%%%%%%%%%%%%%%%%%%%%%%%%%%%%%%%%%%%%

\subsubsection{Utility Preservation Across Erasure Levels and Data Regimes}
\label{sec_app_exp_disent_tradeoff}

\begin{figure}[b!]
    \centering
    \setlength{\tabcolsep}{3pt}
    \begin{tabular}{@{}ccc@{}}

        % Row 1
        \stackinset{l}{16pt}{t}{0pt}{\scriptsize (a)}{%
            \includegraphics[width=0.39\textwidth]{figures/erasure_utility_tradeoff/erasure_utility_tradeoff_ranges_waterbirds.png}%
        }
        &
        \stackinset{l}{16pt}{t}{0pt}{\scriptsize (b)}{%
            \includegraphics[width=0.39\textwidth]{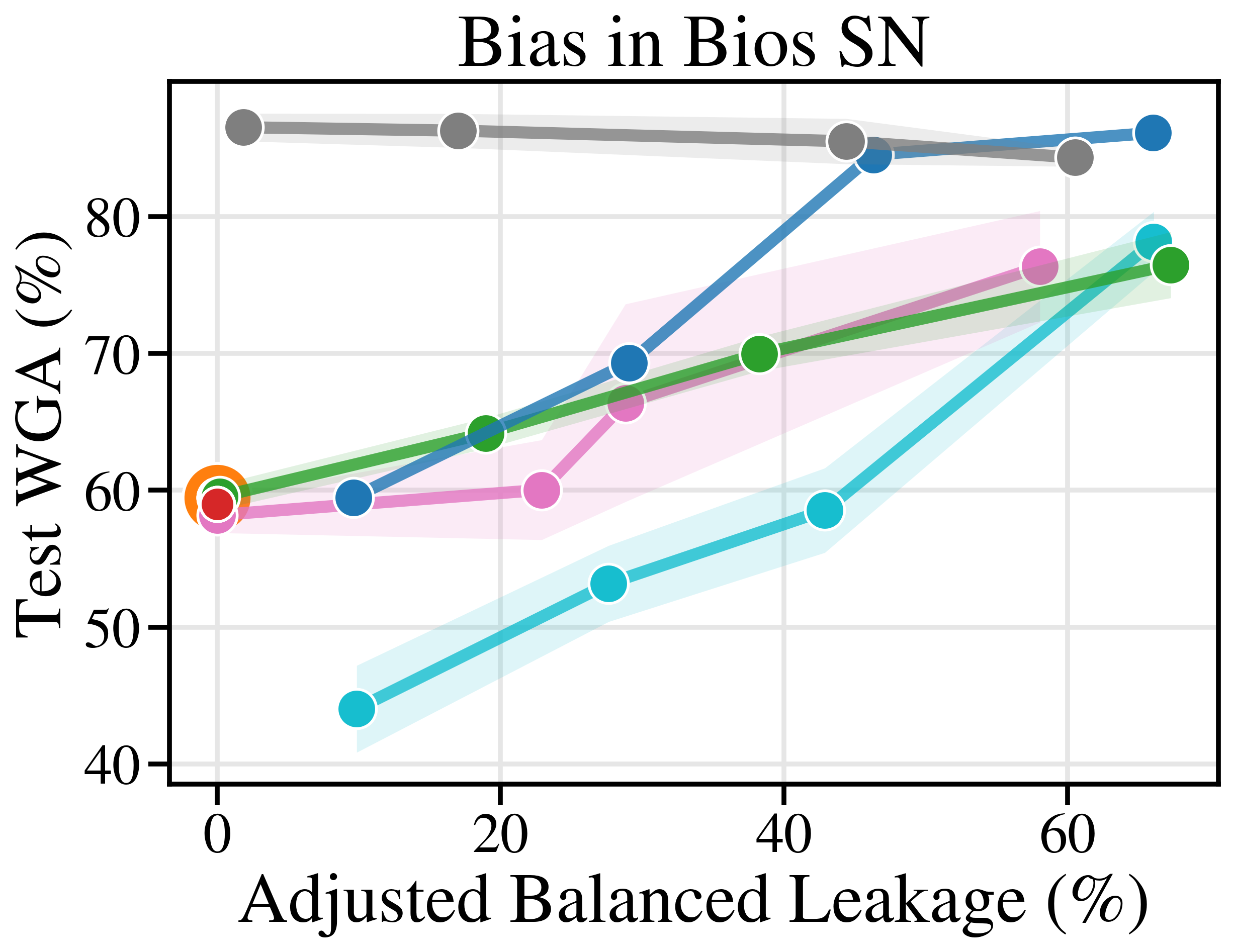}%
        }
        &
        \makebox[0.12\textwidth]{} \\[2pt]

        % Row 2
        \stackinset{l}{16pt}{t}{0pt}{\scriptsize (c)}{%
            \includegraphics[width=0.39\textwidth]{figures/spurious_and_imbalance/spurious_and_imbalance_dial_mention_li.png}%
        }
        &
        \stackinset{l}{16pt}{t}{0pt}{\scriptsize (d)}{%
            \includegraphics[width=0.39\textwidth]{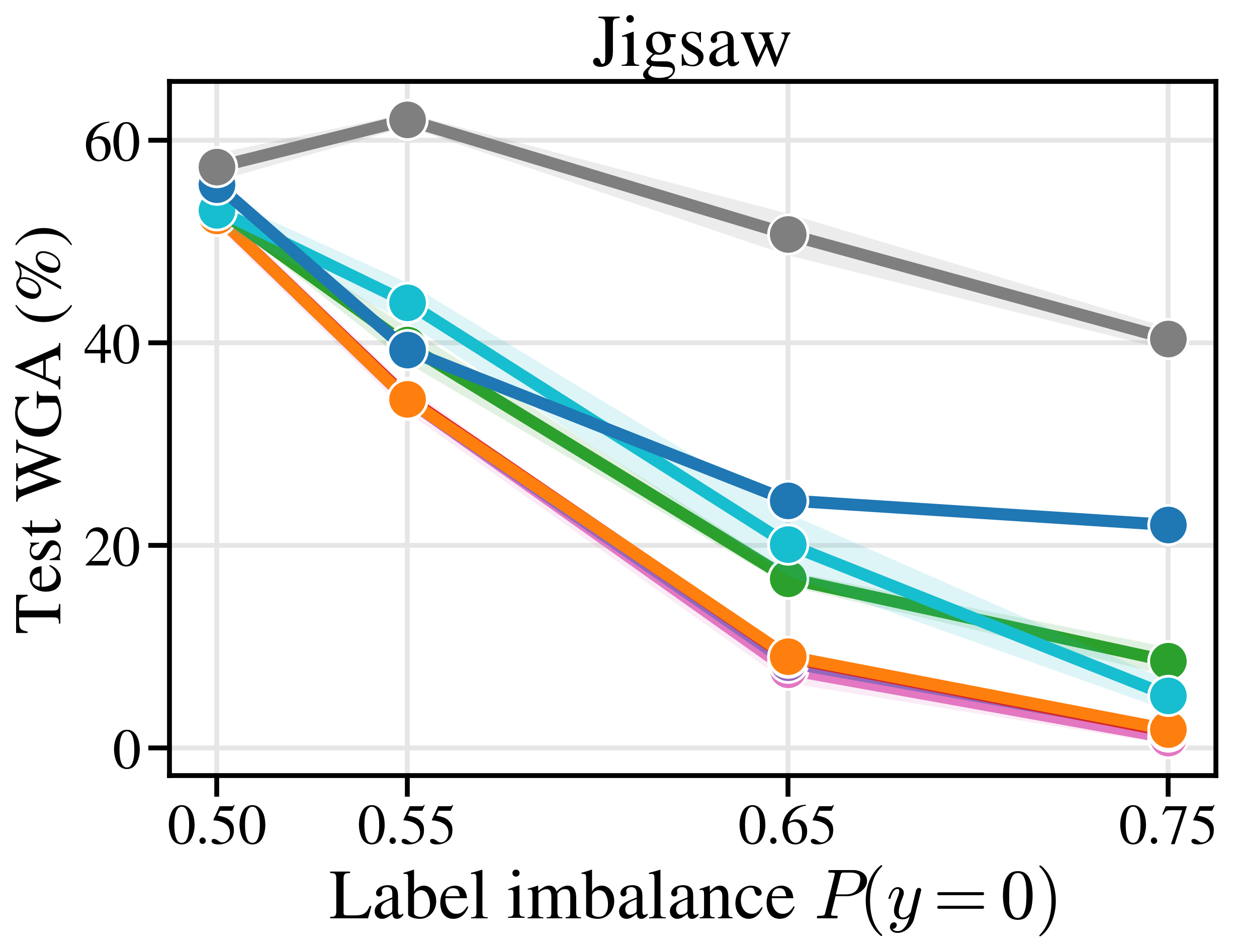}%
        }
        &
        \raisebox{25pt}[0pt][0pt]{%
            \includegraphics[width=0.12\textwidth]{figures/spurious_and_imbalance/spurious_and_imbalance_legend.png}%
        } \\[2pt]

        % Row 3
        \stackinset{l}{16pt}{t}{0pt}{\scriptsize (e)}{%
            \includegraphics[width=0.39\textwidth]{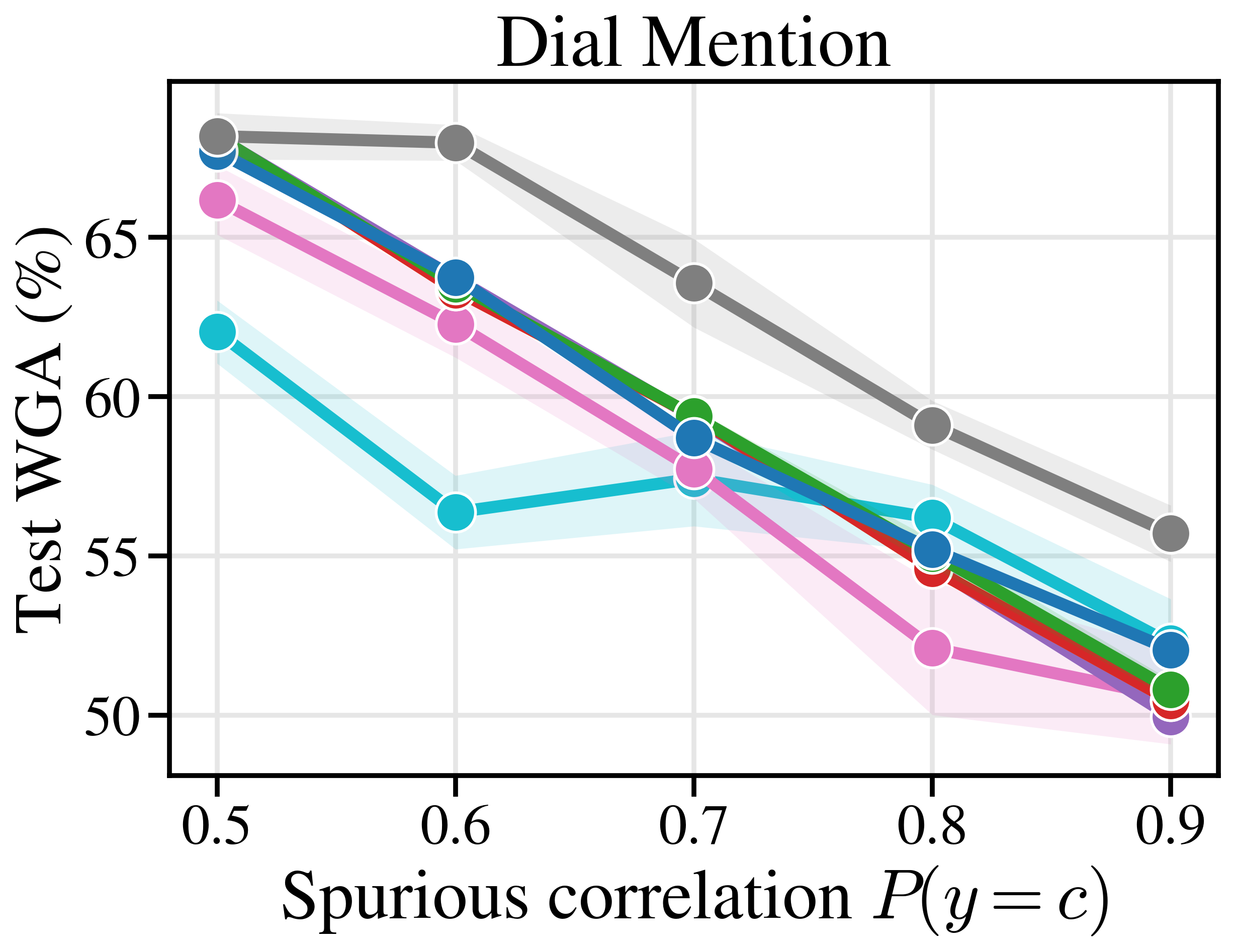}%
        }
        &
        \stackinset{l}{16pt}{t}{0pt}{\scriptsize (f)}{%
            \includegraphics[width=0.39\textwidth]{figures/spurious_and_imbalance/spurious_and_imbalance_jigsaw_sc.png}%
        }
        &
        \makebox[0.12\textwidth]{}

    \end{tabular}

    \caption{
Extended erasure-utility and data-regime analysis.
Erasure-utility trade-off, showing test WGA against adjusted balanced linear leakage,
on Waterbirds (a) and Bias in Bios (b); impact of label imbalance on DIAL-Mention (c)
and Jigsaw (d); and of spurious correlation on DIAL-Mention (e) and Jigsaw (f).
Panels (a), (c) and (f) correspond to the main-text results for direct comparison.
}
    \label{fig_app_erasure_utility_tradeoff}
\end{figure}

\autoref{fig_app_erasure_utility_tradeoff} complements the analyses of
\autoref{fig_tradeoff} with the corresponding experiments on other
dataset and data-regime combinations. Panels (a), (c) and (f) reproduce the
main-text results for direct comparison; we discuss the complementary results below.

On Bias in Bios (\autoref{fig_app_erasure_utility_tradeoff}b), \acronym{} exhibits a
particularly favorable trade-off under strong erasure, reaching $86.52$ WGA at only
$1.88\%$ adjusted balanced linear leakage, compared with $59.53$ for the best baseline
under the strong-erasure constraint. Its WGA then remains nearly constant as the
constraint is relaxed, reaching $84.31$ at $60.52\%$ leakage. In contrast, several
baselines recover task utility primarily by retaining substantially more concept
information; for instance, INLP approaches \acronym{} only at approximately
$46\%$ leakage ($84.52$ versus $85.50$ WGA).

The Jigsaw label-imbalance experiment (\autoref{fig_app_erasure_utility_tradeoff}d)
shows a similar robustness advantage. \acronym{} achieves the highest WGA at every
imbalance level, with its margin over the best baseline increasing from $1.73$ points
at $P(y{=}0)=0.50$ to $18.04$ at $0.55$ and $26.29$ at $0.65$. Even at
$P(y{=}0)=0.75$, \acronym{} retains $40.42$ WGA compared with $22.03$ for the
strongest baseline, indicating substantially better utility preservation as the
task-label distribution becomes increasingly imbalanced.

Finally, on DIAL-Mention (\autoref{fig_app_erasure_utility_tradeoff}e), \acronym{}
is essentially tied with the strongest baseline when task and concept are independent
at $P(y{=}c)=0.50$ ($68.18$ versus $68.17$ WGA), but obtains the highest WGA at
every stronger spurious correlation. Its advantage ranges from $2.92$ to $4.23$
points for $P(y{=}c)\in[0.6,0.9]$, complementing the analogous trend observed on
Jigsaw in the main text.

Together, these additional experiments support the same behavior observed in the
main analysis: \acronym{} remains consistently competitive with the strongest
baselines across the explored settings, while its utility-preservation advantage
becomes most pronounced when strong erasure, label imbalance or spurious dependence
make task and concept information harder to separate.

%%%%%%%%%%%%%%%%%%%%%%%%%%%%%%%%%%%%%%%
%%%%%%%%%%%%%%%%%%%%%%%%%%%%%%%%%%%%%%%

\subsubsection{Representation Similarity After Intervention}
\label{sec_app_exp_disent_cka}

\begin{figure}[b!]
    \centering
    \setlength{\tabcolsep}{3pt}
    \begin{tabular}{@{}cc@{}}

        % Row 1
        \stackinset{l}{16pt}{t}{0pt}{\scriptsize (a)}{%
            \includegraphics[width=0.49\textwidth]{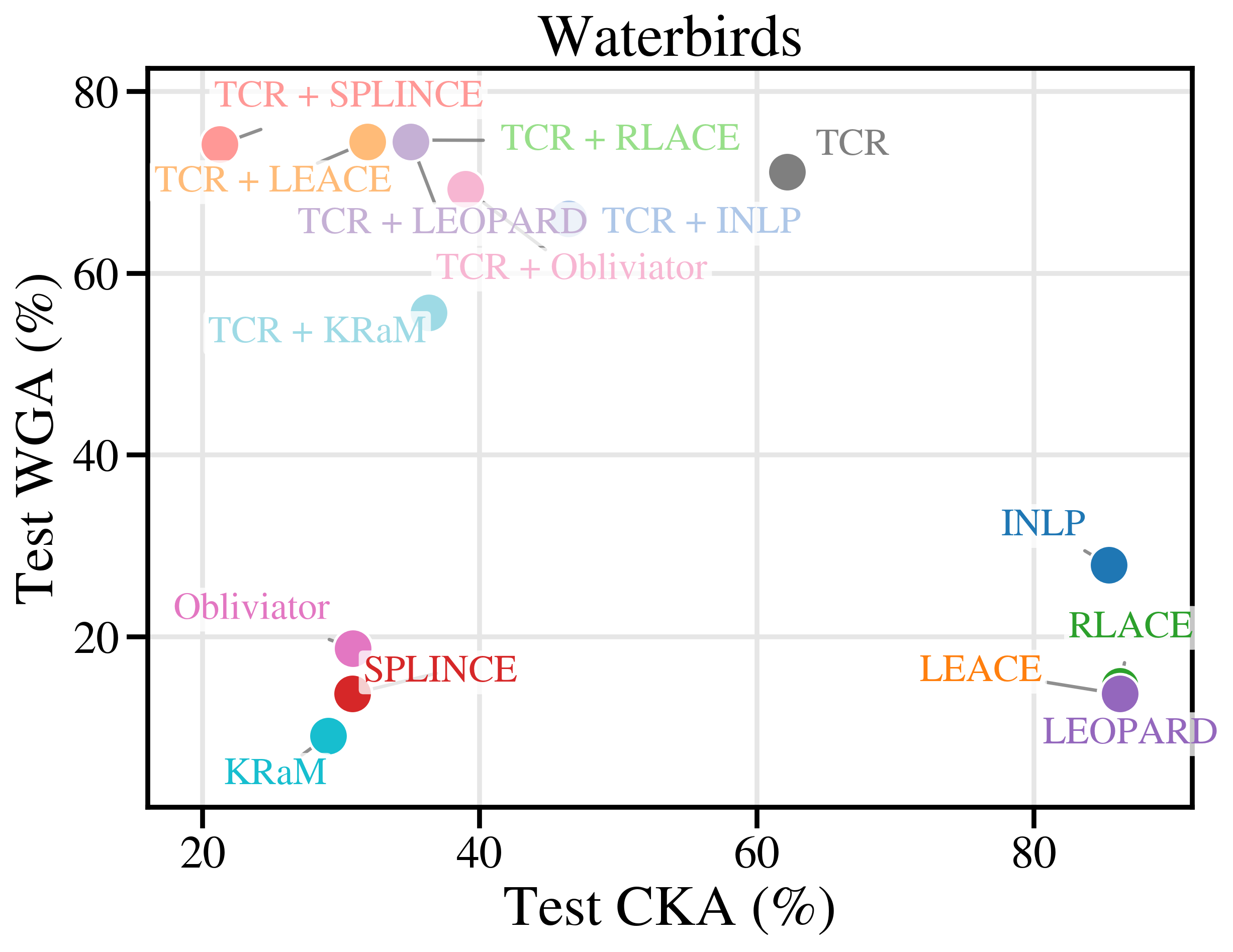}%
        }
        &
        \stackinset{l}{16pt}{t}{0pt}{\scriptsize (b)}{%
            \includegraphics[width=0.49\textwidth]{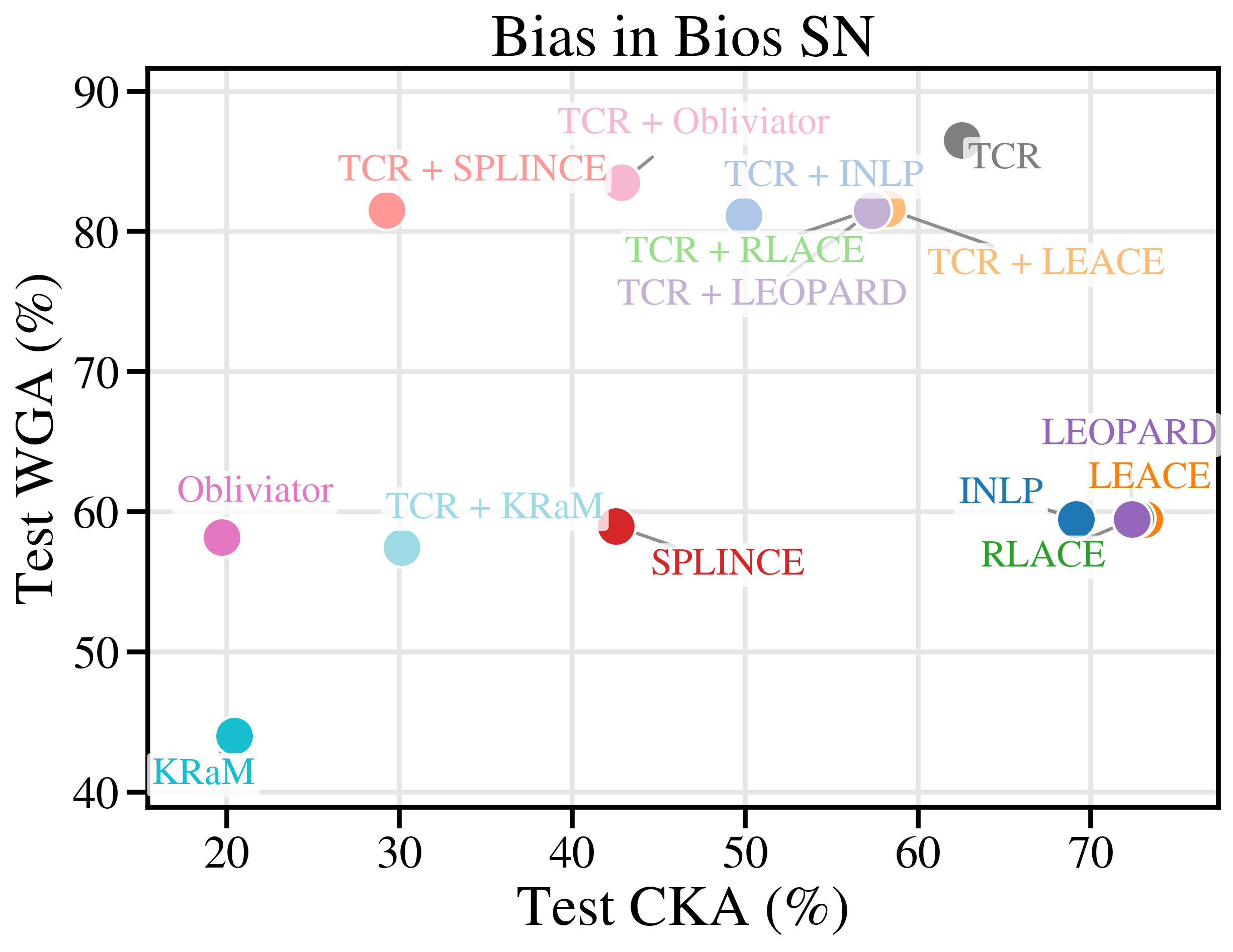}%
        } \\[2pt]

        % Row 2
        \stackinset{l}{16pt}{t}{0pt}{\scriptsize (c)}{%
            \includegraphics[width=0.49\textwidth]{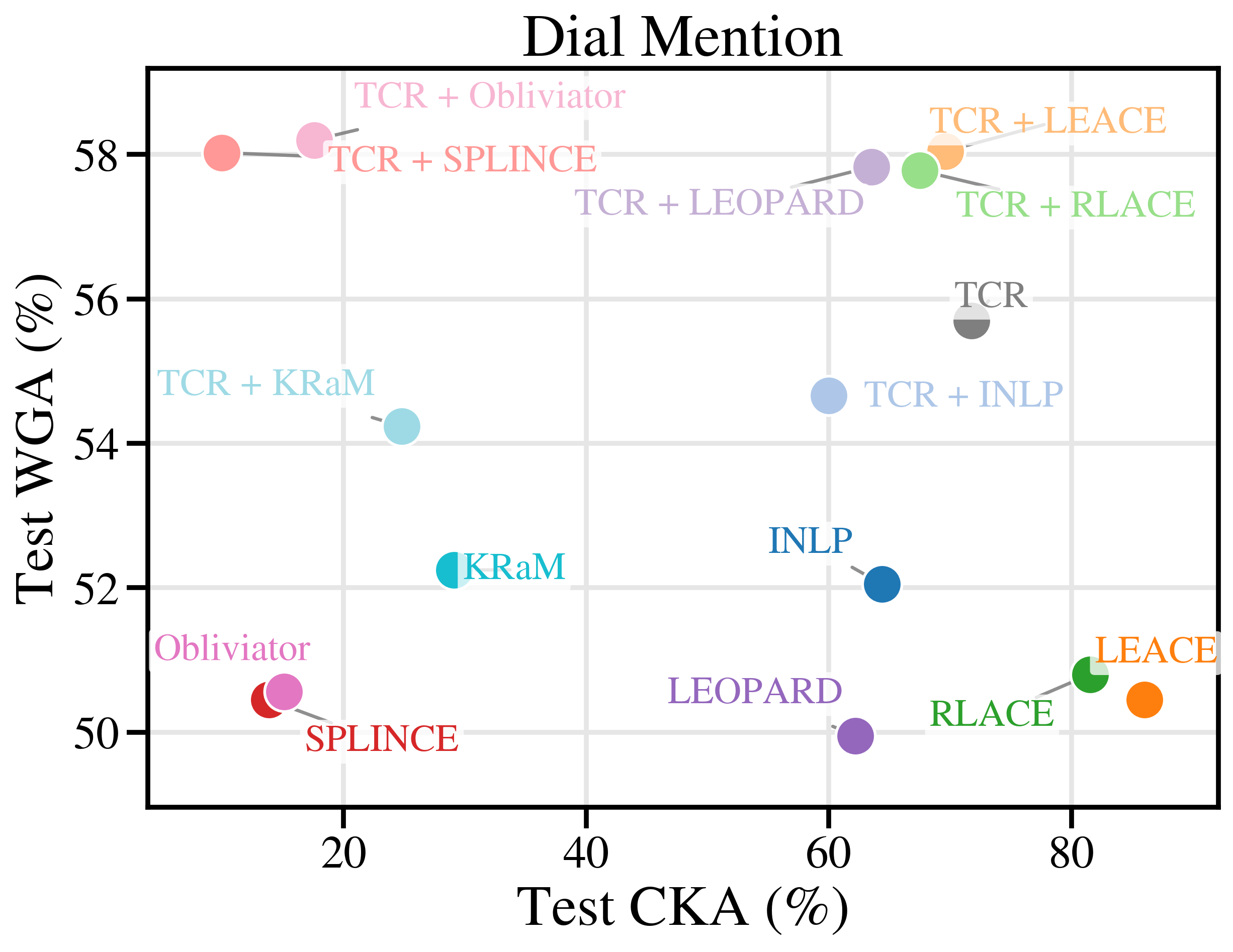}%
        }
        &
        \stackinset{l}{16pt}{t}{0pt}{\scriptsize (d)}{%
            \includegraphics[width=0.49\textwidth]{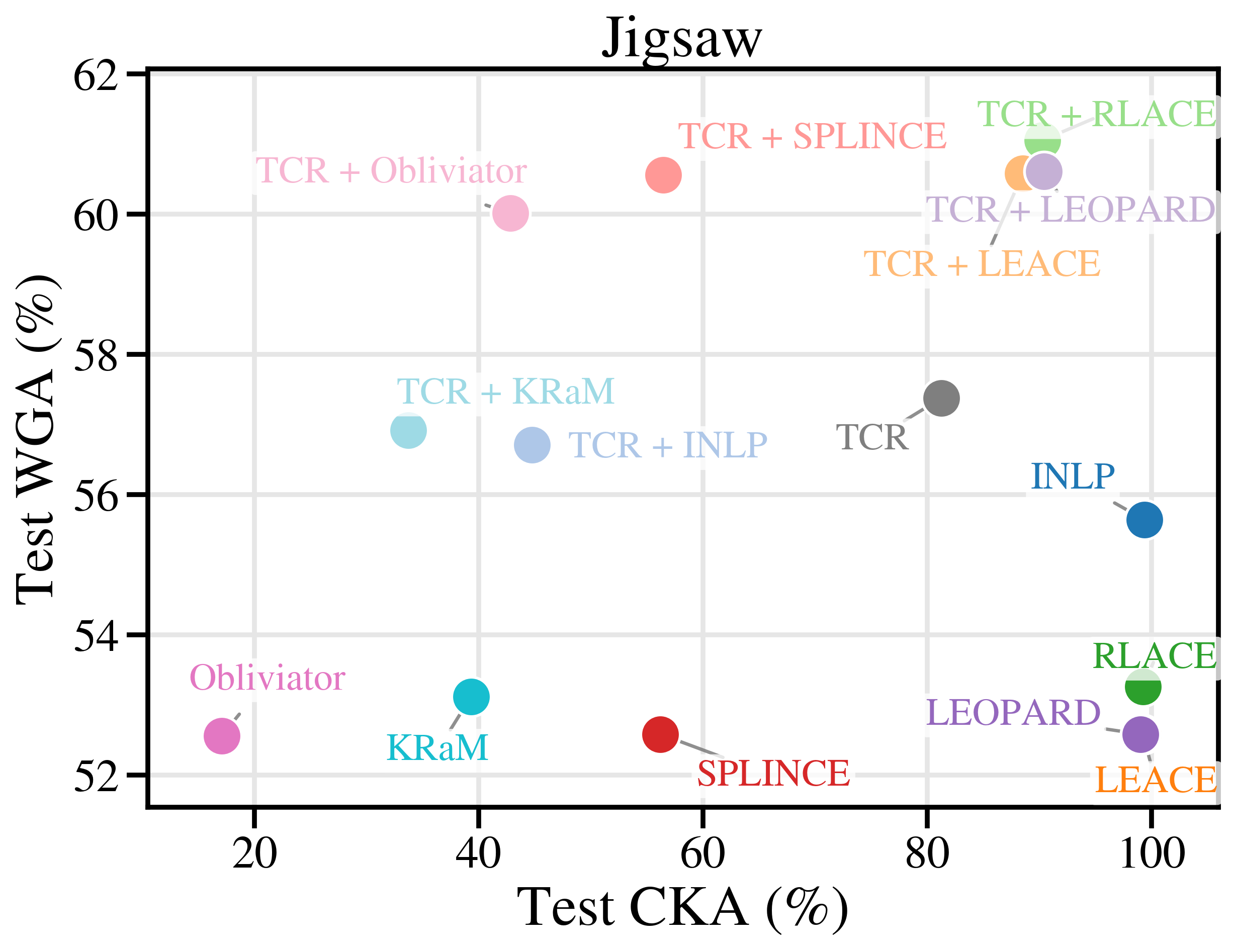}%
        }

    \end{tabular}

    \caption{
Task utility against representational similarity after concept erasure, showing test WGA
against linear CKA between the original and intervened test representations on Waterbirds
(a), Bias in Bios (b), DIAL-Mention (c) and Jigsaw (d). Higher CKA indicates greater
similarity. % to the original representation.
}
    \label{fig_cka}
\end{figure}

Beyond task utility and concept leakage, we evaluate how strongly each intervention modifies
the representation globally. This is important because improved task preservation after erasure
could otherwise result from extensively reorganizing the representation around the specific
task--concept pair. We quantify this effect using linear centered kernel alignment (CKA)
\citep{cka}, a standard measure for comparing neural representations that assesses their
similarity. % of their centered Gram structure.

For a test representation $H\in\mathbb{R}^{N\times d}$ before intervention and the
corresponding representation $Z\in\mathbb{R}^{N\times d'}$ after intervention, let
$H'$ and $Z'$ denote the matrices obtained by centering each feature over the same test
samples. We compute linear CKA as:
\[
    \operatorname{CKA}(H,Z)
    =
    \frac{\|Z'^{\top}H'\|_F^2}
    {\|H'^{\top}H'\|_F\,\|Z'^{\top}Z'\|_F}.
\]
Linear CKA is invariant to isotropic scaling and orthogonal transformations of either
representation, and therefore does not penalize such changes of coordinates, while remaining
sensitive to more general changes in representational geometry \citep{cka}. Values closer to
$1$ indicate greater global similarity between the original and edited representations. Unlike
a learned reconstruction or alignment score, CKA requires no fitted mapping; we therefore
compute it directly between the original and final test representations. For composed
\acronym{}-DIS methods, $Z$ denotes the representation after both disentanglement and the
subsequent erasure method.
% while $\representation$ denotes the representation before disentanglement. 
We compute CKA independently for each run and report its mean
across seeds as a percentage.

\autoref{fig_cka} reports test WGA against the linear CKA between the original and
intervened representations. The results first show that preserving global representation
geometry and preserving task utility are distinct objectives: high CKA does not necessarily
translate into high WGA. On Waterbirds, INLP, LEACE, RLACE and LEOPARD retain between
$85.40\%$ and $86.22\%$ CKA while reaching only $13.71$--$27.88$ WGA. Similarly, on
Jigsaw, these methods retain approximately $99\%$ CKA but obtain only $52.59$--$55.64$
WGA. Thus, concept erasure can leave most of the global representational structure unchanged
while substantially degrading the task-relevant information required by the worst-performing
groups.

Conversely, \acronym{}-DIS often induces a larger representational reorganization while
preserving substantially more task utility. This is particularly pronounced on Waterbirds,
where \acronym{}-DIS variants reach $66.01$--$74.66$ WGA despite lower CKA, and on Bias
in Bios, where the variants based on INLP, LEACE, RLACE, SPLINCE and LEOPARD reach
approximately $81$--$82$ WGA. The same pattern is milder on Jigsaw: combining
\acronym{}-DIS with LEACE, RLACE or LEOPARD reduces CKA from approximately $99\%$ to
$89$--$90\%$, while increasing WGA from $52.59$--$53.27$ to $60.58$--$61.05$.
Hence, the utility gains of disentanglement do not arise from keeping the edited
representation as close as possible to its original geometry; modifying this geometry can
instead facilitate the separation of task-relevant information from the concept being erased.

The amount of representational change also depends strongly on the form of the intervention.
Across datasets, the non-linear baselines KRaM and Obliviator tend to exhibit among the
lowest CKA values, whereas the linear projection-based INLP, LEACE and RLACE generally
remain substantially closer to the original representation. On Waterbirds and Jigsaw, for
instance, the latter retain approximately $85$--$99\%$ CKA despite substantially lower WGA.
At the same time, \acronym{} does not systematically decrease CKA: \acronym{}-DIS
increases both CKA and WGA over Obliviator on Bias in Bios ($19.71\%$ to $42.83\%$ CKA
and $58.18$ to $83.50$ WGA) and Jigsaw ($17.06\%$ to $42.79\%$ CKA and $52.56$ to
$60.01$ WGA), with similar behavior for KRaM on Waterbirds. Direct \acronym{}-ERA also
retains substantial similarity to the original representation, with CKA ranging from
$62.18\%$ on Waterbirds to $81.22\%$ on Jigsaw, while achieving strong WGA across the four
datasets. Overall, the degree of geometric reorganization appears to depend on both the
intervention and the subsequent erasure operator rather than being systematically larger for
\acronym{}.

Finally, CKA measures global representational similarity rather than preservation of all
information in the original representation. A lower CKA therefore indicates a
change in representation geometry, but does not by itself imply that unrelated information
has been discarded. We consequently interpret these results as showing that improvements
from \acronym{} can involve meaningful representational reorganization, while task utility
can remain preserved despite lower global similarity.

%%%%%%%%%%%%%%%%%%%%%%%%%%%%%%%%%%%%%%%%%%%%%%%%%%%%%%%%%%%%%%%%%%%%%%%%%%%%%%%
%%%%%%%%%%%%%%%%%%%%%%%%%%%%%%%%%%%%%%%%%%%%%%%%%%%%%%%%%%%%%%%%%%%%%%%%%%%%%%%
%%%%%%%%%%%%%%%%%%%%%%%%%%%%%%%%%%%%%%%%%%%%%%%%%%%%%%%%%%%%%%%%%%%%%%%%%%%%%%%

%\subsection{Concept Transfer}
%\label{sec_app_exp_transfer}

\iffalse
\section{Limitations}
\acronym{} depends on the chosen concepts, probes and interaction criterion, so different
choices may induce different concept organizations. In particular, our interference score
is operational: it measures the effect of an erasure intervention and should not be
interpreted as causal dependence. The persistence reduction is intentionally lossy and
retains the multiscale connectivity induced by pairwise interactions rather than the full
concept mesh; other abstractions may preserve complementary structure. Our construction
also models pairwise interactions and $0$-dimensional topology, leaving genuinely
higher-order relations for future work. Although persistence retains only $n_C-1$
hierarchy-defining saddles, constructing the complete concept mesh still requires
quadratically many pairwise interaction evaluations, which may limit scalability to very
large concept sets. Finally, our experiments focus primarily on post-hoc edits of selected
representations and mostly linear concept probes; nonlinear information may remain after
linear editing, and extending \acronym{} across layers and larger models remains future work.
\fi

% \iffalse
\section{Limitations}
\label{sec_limitations}
\acronym{} depends on the concepts used to construct it and on the probes and interaction
criterion used to characterize them. Different choices may therefore induce different
concept organizations, and concepts absent from the chosen set are not represented.
In particular, our interference criterion is operational: it measures how erasing one
concept affects the recoverability of another and should not be interpreted as establishing
a causal dependence between them.
The persistence reduction is intentionally lossy. It retains the multiscale connectivity
induced by pairwise interactions while discarding variations of the concept mesh that do
not alter this hierarchy. We do not claim that topology is the unique useful abstraction
of concept organization; alternative reductions or the full concept mesh may preserve
complementary information. Moreover, our current construction relies on pairwise
interactions and $0$-dimensional topology and therefore does not explicitly represent
%genuinely 
higher-order relationships between concepts.
Although the persistence diagram contains only $n_C-1$ hierarchy-defining saddles,
constructing the complete concept mesh requires quadratically many pairwise interaction
evaluations, which may become costly for very large concept sets. Finally, our experiments
focus primarily on 
%selected intermediate 
penultimate representations and mostly linear concept probes.
Nonlinear concept information may remain after linear interventions, and the estimated
concept organization may also change under distribution shift. Extending \acronym{} to
larger concept collections, multiple network layers and richer interaction models provides
%several 
directions for future work.
%\fi

\end{document}